\def\year{2022}\relax
\documentclass[letterpaper]{article} 
\usepackage{aaai22}  
\usepackage{times}  
\usepackage{helvet}  
\usepackage{courier}  
\usepackage[hyphens]{url}  
\usepackage{graphicx} 
\usepackage{natbib}  
\usepackage{caption} 
\DeclareCaptionStyle{ruled}{labelfont=normalfont,labelsep=colon,strut=off} 
\usepackage{algorithm}
\usepackage{algorithmic}

\usepackage{newfloat}
\usepackage{listings}
\floatstyle{ruled}
\newfloat{listing}{tb}{lst}{}
\floatname{listing}{Listing}
\newcommand{\citesupp}[1]{Appendix~\ref{#1}}

\usepackage{booktabs}       
\usepackage{multirow}
\usepackage{amsfonts}       
\usepackage{amsmath}
\usepackage{amsthm}
\usepackage{amssymb}
\usepackage{cleveref}
\usepackage{caption}
\usepackage{subcaption}
\usepackage{changepage}
\usepackage{mathtools}
\usepackage{nicefrac}
\usepackage{graphicx}
\usepackage{algorithm}
\usepackage{algorithmic}
\usepackage{optidef}
\graphicspath{ {./images/} }
\crefformat{footnote}{#2\footnotemark[#1]#3}
\usepackage{caption}
\usepackage{subcaption}
\newcommand{\squishlisttwo}{
 \begin{list}{$\bullet$}
  { \setlength{\itemsep}{1pt}
     \setlength{\parsep}{0pt}
    \setlength{\topsep}{0pt}
    \setlength{\partopsep}{0pt}
    \setlength{\leftmargin}{1em}
    \setlength{\labelwidth}{1.5em}
    \setlength{\labelsep}{0.5em} } }

\newcommand{\squishend}{
  \end{list}  }

\newtheorem{theorem}{Theorem}
\newtheorem{corollary}{Corollary}
\newtheorem{lemma}{Lemma}
\newtheorem{proposition}{Proposition}

\newtheorem{definition}{Definition}

\newtheorem{manualtheoreminner}{Proposition}
\newenvironment{manualproposition}[1]{%
  \renewcommand\themanualtheoreminner{#1}%
  \manualtheoreminner
}{\endmanualtheoreminner}

\newtheorem{manualcoroinner}{Corollary}
\newenvironment{manualcorollary}[1]{%
  \renewcommand\themanualcoroinner{#1}%
  \manualcoroinner
}{\endmanualcoroinner}

\newcommand{\expect}{\mathbb{E}}

\newcommand{\brac}[1]{\left( #1 \right)}
\newcommand{\logicaland}{\mathbin{\wedge}}
\newcommand{\mbf}[1]{\mathbf{#1}}
\newcommand{\inn}[2]{\left\langle #1, #2 \right\rangle}
\newcommand{\norm}[1]{\left\lVert#1\right\rVert}

\title{Incentivizing Collaboration in Machine Learning via Synthetic Data Rewards}
\author{
    Sebastian Shenghong Tay,\textsuperscript{\rm 1}\textsuperscript{\rm 2}
    Xinyi Xu,\textsuperscript{\rm 1}\textsuperscript{\rm 2}
    Chuan Sheng Foo,\textsuperscript{\rm 2}
    Bryan Kian Hsiang Low\textsuperscript{\rm 1}
}
\affiliations{
    \textsuperscript{\rm 1}Department of Computer Science, National University of Singapore, Singapore \\
    \textsuperscript{\rm 2}Institute for Infocomm
Research, A*STAR, Singapore\\
    sebastian.tay@u.nus.edu, xuxinyi@comp.nus.edu.sg, foo\textunderscore chuan\textunderscore sheng@i2r.a-star.edu.sg, lowkh@comp.nus.edu.sg
}
\begin{document}
\maketitle
\begin{abstract}
This paper presents a novel \emph{collaborative generative modeling} (CGM) framework that incentivizes collaboration among self-interested parties to contribute data to a pool for training a generative model (e.g., GAN), from which synthetic data are drawn and distributed to the parties as rewards commensurate to their contributions. Distributing synthetic data as rewards (instead of trained models or money) offers task- and model-agnostic benefits for downstream learning tasks and is less likely to violate data privacy regulation. To realize the framework, we firstly propose a data valuation function using \emph{maximum mean discrepancy} (MMD) that values data based on its quantity and quality in terms of its closeness to the true data distribution and provide theoretical results guiding the kernel choice in our MMD-based data valuation function. Then, we formulate the reward scheme as a linear optimization problem that when solved, guarantees certain incentives such as fairness in the CGM framework. We devise a weighted sampling algorithm for generating synthetic data to be distributed to each party as reward such that the value of its data and the synthetic data combined matches its assigned reward value by the reward scheme. We empirically show using simulated and real-world datasets that the parties' synthetic data rewards are commensurate to their contributions.
\end{abstract}

\section{Introduction}
For the state-of-the-art deep learning models, training with a large quantity of data is important to prevent overfitting and achieve good generalization. 
So, when there are multiple parties with each owning a dataset sampled from the same distribution, 
pooling their datasets and training on the pooled dataset would yield an improved \emph{machine learning} (ML) model for every participating party. For example, banks that use ML models to predict their customers' credit ratings~\cite{TSAI2010374} would benefit from pooling their datasets as every bank can now train its ML model on a much larger dataset with more unique customers. This benefit would be even more pronounced in applications where data is difficult/costly to obtain and every party has limited data, such as in medical imaging~\cite{sandfort2019data}. 

However, data sharing/pooling is challenging in practice due to issues of data privacy~\cite{devereaux2016} and possibly inequitable benefits from such a form of collaboration~\cite{lo2016}. To elaborate on the latter, parties would be more willing to participate in the collaboration if \emph{fairness} is guaranteed. For example, if only one party contributes  high-quality data 
but all parties receive equal rewards, then clearly the other parties benefit disproportionately from the collaboration and that contributing party has no incentive to participate, especially when all parties are self-interested. One may define fairness as every party receiving a reward commensurate to its contribution (however contribution is measured), which incentivizes not only participation but also non-trivial contributions from all parties.

To resolve the above issues, the notion of \emph{collaborative ML} (also referred to as \emph{multi-party ML}~\cite{chen2020mechanism}) allows multiple self-interested parties to mutually benefit from collaboration 
in data sharing/pooling
by incentivizing non-trivial contributions from them while accounting for fairness and data privacy. 
A prior work of collaborative ML~\cite{simcollaborative}  has focused on the supervised learning setting where every party contributes training data and receives a model as reward with predictive performance commensurate to its contribution, while another work~\cite{ohrimenko2019collaborative} has developed a marketplace where parties pay money for better performing ML models on their specific learning tasks and receive money when their contributed data improve the ML models of other parties. 
A key limitation of these works is that trained ML models are distributed to the parties as rewards, which limits each party's flexibility to experiment with different model architectures and hyperparameters. If more competitive architectures emerge in the future, the parties cannot take advantage of these new architectures without reinitiating the collaboration. Another limitation of distributing trained ML models as rewards is that it precludes the possibility of performing a different learning task on the same dataset as the ML model is tied to a specific task.

One way of overcoming the above limitations is to distribute synthetic data to the parties as rewards (in short, \emph{synthetic data rewards}) instead of trained models. It has been demonstrated that  augmenting real data with synthetic data can improve model performance: For example, some works~\cite{bowles2018gan,frid2018gan,sandfort2019data} have used \emph{generative adversarial networks} (GANs) for data augmentation to improve classification performance on various medical imaging applications such as liver lesions or brain scan segmentations. 
Distributing synthetic data as rewards is less likely to violate data privacy regulations, unlike sharing real data directly.
Also, there is no assumption on whether all parties share a common downstream learning task, the task of interest to each party  (e.g., supervised or unsupervised, classification or regression), or the type of ML model used by each party. In particular, with the synthetic data reward, each party can now optimize over model architectures and hyperparameters, train new model architectures emerging in the future, and train separate ML  models for different learning tasks. 

As a departure from the restriction to supervised learning, this paper presents a novel \emph{collaborative generative modeling} (CGM) framework that incentivizes collaboration among self-interested parties to contribute data to a pool for training an unsupervised generative model (e.g., a GAN), from which 
synthetic data are drawn and distributed to the parties as rewards (i.e., commensurate to their contributions)
instead of 
sharing real data directly. Like previous works on collaborative ML~\cite{ohrimenko2019collaborative, simcollaborative}, our CGM framework only requires a trusted mediator to train the generative model on the pooled dataset but differs in offering the above-mentioned task- and model-agnostic benefits of synthetic data rewards.
Our framework does not consider monetary payment and hence enables participation from parties such as startups or non-profit organizations with data but limited/no funds. 
Our work here provides the following specific novel contributions:
\squishlisttwo
    \item We propose a task- and model-agnostic data valuation function 
    using \emph{maximum mean discrepancy} (MMD)
    that values (real and/or synthetic) data based on its quantity and quality in terms of its closeness to the true data distribution
    (Sec.~\ref{junkyard}), and provide theoretical results guiding the choice of the kernel in our MMD-based data valuation function (Sec.~\ref{sect:kernelselection});
    \item We formulate the reward scheme as a linear optimization problem that when solved, guarantees certain incentives such as fairness in the CGM framework
    (Sec.~\ref{lalaland});
    \item We devise a weighted sampling algorithm for generating synthetic data to be distributed to each party as reward such that the value of its data and the synthetic data combined matches its assigned reward value by the reward scheme (Sec.~\ref{force}),   
    and empirically show using simulated and real-world datasets that 
    the parties' synthetic data rewards are commensurate to their contributions 
    (Sec.~\ref{expt}).\vspace{0.5mm}
\squishend

\noindent
\textbf{Related Work.} Collaborative ML is a rich and novel field which uses solution concepts from cooperative game theory and mechanism design. 
The Shapley value is a commonly adopted solution concept to formalize a notion of fairness in quantifying the contributions of self-interested parties~(e.g., via their shared data)~\cite{ohrimenko2019collaborative, simcollaborative}.
This line of research inspires several data valuation methods using the Shapley value~\cite{ghorbani2020distributional, Ghorbani2019, Jia2020, Wang2020-SV-in-FL}, the core~\cite{yan2021core}, influence functions~\cite{richardson2019rewarding}, and volume~\cite{xu2021validation}.
Previous works have used concepts from mechanism design to elicit truthful reporting~\cite{chen2020mechanism, Richardson2020-FL-incentive} and to incentivize sharing data and/or model parameters in \emph{federated learning}~\cite{Cong2020-FL-incentive, kang2019incentive2, kang2019incentive, lyu2020-CFFL,Yu-et-al:2020AIES, zhan2020learning, xu2021gradient}.
Other works have addressed data privacy~\cite{ding2021differentially,Hu2019-FDML-collaborative-on-features}, adversarial robustness~\cite{ hayes2018contamination, so2020scalable}, communication efficiency~\cite{ding2021differentially}, and fairness in Bayesian optimization~\cite{sim2021collaborative}.
Compared to existing works which have mainly focused on supervised learning, our work investigates a novel task- and model-agnostic setting through the CGM framework that distributes synthetic data as rewards, which to the best of our knowledge has not been considered in the literature.

\section{Problem Statement and Notations}
\label{yoda}
The CGM framework comprises a set of honest, non-malicious parties $N := \{1, \ldots, n\}$ and their corresponding datasets $D_1,\ldots,D_n$. 
Let $\mathcal D$ be the true data distribution s.t.~each party $i$ may only be able to sample its dataset $D_i$ from a restricted subset of the support of $\mathcal D$. 
Every party $i$ sends $D_i$ to a trusted mediator who trains a generative model (e.g., GAN, variational autoencoder, or flow-based model) on the pooled dataset $\bigcup_{i\in N}D_i$ to produce a distribution $\mathcal G$ from which the mediator is able to draw samples. Informally, $\mathcal G$ represents an approximation of $\mathcal D$. The mediator then generates a large 
\emph{synthetic dataset} $G$ s.t.~each synthetic data point in $G$ is drawn i.i.d.~from $\mathcal G$. The reward to each party $i$ will be a subset $G_i$ (of synthetic data points) of $G$ and is thus said to be \emph{freely replicable}.\footnote{\label{mojo}Like digital goods, model or data reward can be replicated at no marginal cost and given to more parties~\cite{simcollaborative}.}
In this paper, we use the following definitions from cooperative game theory \cite{chalkiadakis2011computational}: A \emph{coalition} $C$ is a subset of parties (i.e., $C \subseteq N$). The \emph{grand coalition} is the set $N$ of all parties. A \emph{coalition structure} $CS$ is a partition of the parties into disjoint coalitions s.t.~$\bigcup_{C \in CS} C = N$, $C \cap C' = \emptyset$ for all $C, C' \in CS$ and $C \neq C'$,
and each party cooperates only with parties in the same coalition. 
A \emph{characteristic function} $v_c: 2^N \rightarrow \mathbb{R}$ maps each coalition to a (real) value of the coalition.
Finally, the \emph{reward vector} $(r_1, \ldots,r_n) \in \mathbb R^{n}$ denotes the final reward values assigned to parties $1,\ldots,n$.

Our work here considers the problem of CGM defined as follows: 
Given the parties' datasets $D_1,\ldots,D_n$ and an appropriate \emph{data valuation} function $v$ (quantitatively capturing the practical assumptions A, B, and C in Sec.~\ref{junkyard} on the desired qualities of a dataset), determine the reward vector $(r_1, \ldots,r_n)$
that guarantees certain incentives (Sec.~\ref{sect:rewardscheme}),
and then distribute subsets of synthetic data points $G_1,\ldots,G_{n} \subseteq G$ to the respective parties $1,\ldots,n$ as rewards s.t.~$v(D_i \cup G_i) = r_i$ (Sec.~\ref{force}).

\section{Data Valuation with Maximum Mean Discrepancy (MMD)}
\label{junkyard}
Existing metrics for evaluating the approximation quality of generative models do so by measuring some form of distance between the generated and the true distributions~\cite{borji2019pros}. One such distance measure is the \emph{maximum mean discrepancy} (MMD) which is a statistic to test whether two distributions $\mathcal D'$ and $\mathcal D$ are different by measuring the difference of their expected function values based on samples drawn from these distributions~\cite{gretton2012kernel}:
\begin{equation*}
\textstyle    \text{MMD}(\mathcal F, \mathcal D', \mathcal D)\! :=\! \sup_{f \in \mathcal F}\left(\expect_{x \sim \mathcal D}[f(x)]\! - \!\expect_{x' \sim \mathcal D'}[f(x')]\right)  
\end{equation*}
where $\mathcal F$ is the 
class of functions $f$ 
in the unit ball of 
the reproducing kernel Hilbert space associated with a kernel function $k$. We defer the discussion on kernels appropriate for use with MMD to \citesupp{app:kernel}, and will discuss the choice of kernel function $k$ in  Sec.~\ref{sect:kernelselection}. Note that $\text{MMD}(\mathcal F, \mathcal D', \mathcal D) = 0$ iff $\mathcal D' = \mathcal D$ \cite{gretton2012kernel}. 
Let the \emph{reference dataset} $T := D_1 \cup\ldots \cup D_n \cup G$ denote a union of the pooled dataset with the synthetic dataset and hence represents all available data in our problem setting. 
Let $t := |T|$ and $S$ be any arbitrary subset of $T$ where $s := |S|$.
The unbiased estimate ${\text{MMD}}^2_u(\mathcal F, S, T)$ and biased estimate ${\text{MMD}}^2_b(\mathcal F, S, T)$ of the squared MMD can be obtained in the form of matrix Frobenius inner products, as shown in~\cite{gretton2012kernel}:  
\begin{equation}
\begin{array}{l}
    \displaystyle{\text{MMD}}^2_u(\mathcal F, S, T)= \displaystyle\langle (s(s-1))^{-1} \boldsymbol1_{[x, x' \in S, x \neq x']}\ - \vspace{1mm}\\ 
    \quad 2(st)^{-1} \boldsymbol1_{[x \in S, x'\in T]} + (t(t-1))^{-1}\boldsymbol1_{[x, x'\in T, x \neq x']}, \mbf K \rangle \vspace{2mm}\\
    \displaystyle{\text{MMD}}^2_b(\mathcal F, S, T)= \displaystyle\langle s^{-2} \boldsymbol1_{[x, x' \in S]} \ - \vspace{1mm}\\
   \quad 2(st)^{-1} \boldsymbol1_{[x \in S, x'\in T]} + t^{-2}\boldsymbol1_{[x, x'\in T]}, \mbf K \rangle\vspace{-4mm}
\end{array}
\label{seb}\vspace{2mm}
\end{equation}
where $\boldsymbol 1_{A}$ is a matrix with components $1(x,x')$ for all $x,x'\in T$ such that $1(x,x')$ is an indicator function of value $1$ if condition $A$ holds and $0$ otherwise,
and $\mbf K$
is a matrix with components $k(x,x')$ for all $x,x'\in T$.

Our \emph{data valuation} function exploits the negative  ${\text{MMD}}^2_b(\mathcal F, S, T)$~\eqref{seb} w.r.t.~reference dataset $T$:\footnote{A similar form to~\eqref{eq:valuefunc} is  considered in another work with a different focus on interpretable ML~\cite{kim2016examples}.}
\begin{equation}
\begin{array}{rl}
    v(S) :=&\hspace{-2.4mm} \displaystyle\left\langle t^{-2}\boldsymbol1_{[x, x'\in T]}, \mbf K \right\rangle - {\text{MMD}}^2_b(\mathcal F, S, T) \vspace{0.5mm}\\
    =&  \hspace{-2.4mm}\displaystyle\left\langle {2}(st)^{-1} \boldsymbol1_{[x \in S, x'\in T]} - s^{-2} \boldsymbol1_{[x, x' \in S]}, \mbf K \right\rangle 
\end{array}    
    \label{eq:valuefunc}
\end{equation}
which is a reasonable choice for our problem setting under the following practical assumptions:\vspace{0.5mm}

\noindent
(A) Every party benefits from having 
data drawn from $\mathcal D$ 
besides having just its dataset $D_i$ since $D_i$ may only be sampled from a restricted subset of the support of $\mathcal D$ (Sec.~\ref{yoda}). We discuss its validity in \citesupp{app:assumptionA}. \vspace{0.5mm}

\noindent
(B) The empirical distribution associated with the reference dataset $T$ (i.e., the pooled dataset and synthetic dataset) approximates the true data distribution $\mathcal D$ well.
This principle of approximating the ground truth with an aggregate has precedence in multi-party ML~\cite{blanchard2017machine}.\vspace{0.5mm} 

\noindent
(C) Having more data is at least never worse off, which is generally true for ML problems (precluding cases such as excessively noisy data or adversarial data)
and investigated in computational learning theory in the form of sample complexity~\cite{bousquet2003introduction}.\vspace{0.5mm}

We will now show that under such practical assumptions, $v(S)$~\eqref{eq:valuefunc} w.r.t.~reference dataset $T$ is a reasonable choice for data valuation: 
\begin{proposition}
\label{prop:valuefunc}
Let $k^*$ be the value of every diagonal component of $\mathbf K$ s.t.~$k^* := k(x, x)\geq k(x, x')$ for all $x,x'\in T$, and 
$\sigma_S := \left \langle s^{-2} \boldsymbol1_{[x, x' \in S]}, \mbf K \right\rangle$. Then,  $v(S)$~\eqref{eq:valuefunc} can be re-expressed as
\begin{equation}
    v(S) = (s-1)^{-1}(\sigma_S-k^*)-{\textup{MMD}}^2_u(\mathcal F, S, T) + c
\label{seb2}    
\end{equation}
where $c$ is a constant (i.e., independent of $S$).
\end{proposition}
Since $\sigma_S$ is an average of kernel components $k(x,x')$ for all $x,x'\in S$, 
$\sigma_S \leq k^*$. It follows that the value $v(S)$~\eqref{seb2} of dataset $S$ appears to weakly increase as $s$ increases (hence satisfying assumption C) and/or as ${\text{MMD}}^2_u(\mathcal F, S, T)$ decreases (thus satisfying assumptions A and B, since ${\text{MMD}}^2_u(\mathcal F, S, T)$ is an unbiased estimate of the squared
MMD between the distributions associated with $S$ and $T$). But, this interpretation is not entirely correct as
the value of $\sigma_S$ may fluctuate with an increasing $s$, which depends on what data points are added to $S$. The result below gives a more precise interpretation if the value of every off-diagonal component of $\mathbf K$ can be bounded:

\begin{corollary}
\label{coro:valuefunc}
Suppose that there exist some constants $\gamma$ and $\eta$ s.t.~$\gamma\leq k(x,x')\leq \eta\leq k^*$ for all $x,x'\in T$ and $x\neq x'$.
\begin{multline}
    s^{-1}(\gamma - k^*)-{\textup{MMD}}^2_u(\mathcal F, S, T) + c \leq v(S) \\ \leq  s^{-1}(\eta - k^*)-{\textup{MMD}}^2_u(\mathcal F, S, T) + c\ .
\label{rara}    
\end{multline}
\end{corollary}
Since $\gamma \leq \eta \leq k^*$, as $s$ increases and/or ${\text{MMD}}^2_u(\mathcal F, S, T)$ decreases, the upper and lower bounds of $v(S)$ in~\eqref{rara} both weakly increase. 
So,
given that the above practical assumptions 
hold, $v(S)$ is a reasonable choice for data valuation as it accounts for both the dataset quantity $s$ and quality in terms of 
closeness to the empirical distribution associated with  reference dataset $T$ via ${\text{MMD}}^2_u(\mathcal F, S, T)$.
Also, $v(S)$ is downstream \emph{task-agnostic} (i.e., no assumption on how each party uses its synthetic data reward) and \emph{model-agnostic} (i.e., no restriction to the type of ML model adopted by each party) which are desirable properties as they afford flexibility
to the parties.
We will discuss in Sec.~\ref{sect:kernelselection} how $\gamma$ and $\eta$ can be set to guarantee a non-negative and monotone $v(S)$.

Finally, our \emph{characteristic function} for data valuation is defined as $v_c(C):= v(\bigcup_{i\in C} D_i)$ which will be used to determine the expected marginal contributions of parties $1,\ldots,n$ to CGM via the Shapley value and in turn their reward values $(r_1,\ldots,r_n)$ (Sec.~\ref{yoda}), as detailed next.

\section{Reward Scheme for Guaranteeing Incentives in CGM Framework}
\label{sect:rewardscheme}
To incentivize collaboration among all parties in the grand coalition, their assigned rewards have to satisfy certain incentive conditions established in cooperative game theory. However, classical cooperative game theory cannot be directly applied to our problem setting involving freely replicable synthetic data reward\cref{mojo}. 
Inspired by the reward scheme of~\citet{simcollaborative} for Bayesian supervised learning that is designed to guarantee certain incentives under freely replicable model reward\cref{mojo},
we will propose here a novel reward scheme that meets \emph{appropriately modified} incentive conditions to suit our CGM framework. 

We begin by considering the \emph{Shapley value} $\phi_i$ of party $i$, which quantifies its expected \emph{marginal contribution} when it joins the other parties preceding it in any permutation:
\begin{equation}
\textstyle
    \phi_i := ({1}/{n!}) \sum_{\pi \in \Pi_N} 
    \left[v_c(C_{\pi,i} \cup \{i\}) - v_c(C_{\pi,i})\right]
    \label{shapley}
\end{equation}
where the characteristic function $v_c$ for data valuation is previously defined in Sec.~\ref{junkyard}, $\Pi_N$ is the set of all possible permutations of $N$, and
$C_{\pi,i}$ is the coalition of parties preceding $i$ in permutation $\pi$~\cite{chalkiadakis2011computational}. 
The notion of marginal contribution (and hence Shapley value) plays a significant role in the properties of ({\bf F3}) strict desirability and ({\bf F4}) strict monotonicity 
that define the ({\bf R5}) fairness incentive in~\cite{simcollaborative}:\footnote{The other two properties: ({\bf F1}) uselessness and ({\bf F2}) symmetry defining R5 in~\cite{simcollaborative} are standard axioms of Shapley value~\cite{shapley1953value} and commonly used in works on data valuation~\cite{Ghorbani2019,Jia2020,ohrimenko2019collaborative}. 
Due to lack of space, we have reproduced the formal definitions of properties F1 to F4 in \citesupp{app:rewardincentives}.} 
In our work, the implication of F3 is that if the marginal contributions of parties $i$ and $j$ only differ for coalition $C$  (i.e., $v_c(C\cup \{i\}) > v_c(C\cup \{j\})$), then it is only fair for party $i$ to be assigned a larger reward value $r_i$; its effect on our modified F4 will be discussed later in Sec.~\ref{lalaland}.

Besides R5, the reward scheme of~\citet{simcollaborative} has considered other desirable incentives when forming the grand coalition $N$:
({\bf R1}) Non-negativity: $\forall i \in N\ \ r_i \geq 0$;
({\bf R2}) Feasibility: ${\forall i \in N}\ \ {r_i \leq v_{c}(N)}$;
({\bf R3}) Weak efficiency: ${\exists i \in N}\ \ {r_i = v_{c}(N)}$; 
({\bf R4}) Individual rationality: $\forall i \in N\ \ r_i \geq v_{c}(\{i\})$;
({\bf R6}) Stability: $\forall C \subseteq N \ \  \forall i \in C \ \ (\phi_i = \max_{j \in C} \phi_j) \Rightarrow\ v_c(C) \leq r_i$;
and ({\bf R7}) Group welfare involves maximizing $\sum_{i \in N} r_i$. 
Intuitively, R4 says that the reward value assigned to each party $i$ should be at least the value of its dataset $D_i$, which makes it prefer collaboration in $N$ than working alone. R6 states that the grand coalition is stable if for every coalition $C\subseteq N$, the reward value assigned to the party with largest Shapley value is at least the value  of datasets $\bigcup_{i\in C} D_i$, which prevents all parties in coalition $C$ from simultaneously breaking away and obtaining larger reward values.
We will describe the  intuition underlying our modified R2 and R3 in Sec.~\ref{lalaland}.

Given that $v_c$ is non-negative and monotonically increasing  (see Sec.~\ref{sect:kernelselection} for sufficient conditions that guarantee these properties), the reward scheme of~\citet{simcollaborative} exploits the notion of \emph{$\rho$-Shapley fair reward values} $r_i := \brac{{\phi_i}/{\phi^*}}^\rho \times v_c(N)$ for each party $i\in N$ with an adjustable parameter $\rho$ to trade off between satisfying the above incentive conditions. For your convenience, we've reproduced their main result and full definitions of incentive conditions in \citesupp{app:rewardincentives} and
consolidated our discussion of the key differences with the work of~\citet{simcollaborative} in \citesupp{app:comparison}.

\subsection{A Modified Reward Scheme with Rectified $\rho$-Shapley Fair Reward Values}
\label{lalaland}
Under the CGM framework, each party $i$ initially has dataset $D_i$ (Sec.~\ref{yoda}) and would thus be assigned at least a reward value of $r_i = v_c(\{i\}) = v(D_i)$, i.e., when $G_i = \emptyset$. 
This is a subtle yet important difference with the reward scheme of~\citet{simcollaborative}, the latter of which allows a party to be assigned a reward value of $0$. So, we introduce a \emph{rectified} form of the above $\rho$-Shapley fair reward values:
\begin{equation}
    r_i := \max \left\{v_c(\{i\}), \brac{{\phi_i}/{\phi^*}}^\rho \times v^*\right\}
\label{eq:rewarddef}    
\end{equation}
for each party $i\in N$ where $v^*$ is the maximum reward value (i.e., $v^*\geq r_i$ for any party $i\in N$), 
as discussed below (notice from Theorem~1 that $v^* = v_c(N)$ in~\cite{simcollaborative}). When the grand coalition $N$ forms, R4 is trivially satisfied since each party $i$ has at least its dataset $D_i$, hence distinguishing our modified reward scheme from that of~\citet{simcollaborative} whose R4 may be violated.
So, for our reward scheme, no party will be worse off by participating in the collaboration.
However, other non-trivial issues ensue: 

\begin{proposition}\label{prop:etalesser}
If $v^* = v_c(N)$ and $\rho$ satisfies  $\brac{{\phi_i}/{\phi^*}}^\rho \times v^* < v_c(\{i\})$ for some party $i\in N$, then $(r_1,\ldots,r_n)$~\eqref{eq:rewarddef} may not satisfy R5 due to possibly violating F3.
\end{proposition}

Furthermore, recall from Sec.~\ref{yoda} that under the CGM framework, the mediator generates a synthetic dataset $G$ from which subsets of synthetic data points are sampled to distribute to the parties as rewards. This leads to a few important implications. Firstly, since every party can at most be rewarded the entire synthetic dataset $G$,
the largest possible reward value $v(D_i \cup G)$ may differ across parties $i=1,\ldots,n$.
In contrast, for the reward scheme of~\citet{simcollaborative}, the largest possible reward value $v_c(N)$ is the same across all parties.
Note that in our work, $v(D_i \cup G) > v_c(N)$ is possible. All these motivate the need to consider a generalized notion of the maximum reward value $v^*$ 
(i.e., $v^*\geq r_i$ for any party $i\in N$)
in our modified reward scheme; we will discuss below
how $v^*$ can be optimized via a linear program.
As a result, R2 and R3 have to be redefined to reflect the possibility of $v(D_i \cup G) > v_c(N)$ and ensure at least one party being assigned the maximum reward value $v^*$ instead of the possibly smaller $v_c(N)$, respectively:

\begin{definition}[\bf R2: CGM Feasibility]
\label{def:cgmfeasibility}
\emph{No party in the grand coalition should be assigned a reward value larger than that of its dataset and the synthetic dataset combined:\vspace{-1mm} $$\forall i \in N\ \ r_i \leq v(D_i \cup G)\ .$$
}
\end{definition}

\begin{definition}[\bf R3: CGM Weak Efficiency]
\label{def:cgmweakefficiency}
\emph{At least a party in the grand coalition should be assigned the maximum reward value: $\exists i \in N\ \ r_i = v^*\ .$}
\end{definition}

We need to redefine property F4 defining R5 to account for the notion of 
maximum reward value $v^*$:
\begin{definition}[\bf F4: CGM Strict Monotonicity]
\label{def:cgmstrictmonotonicity}
\emph{Let $v_c$ and $v'_c$ denote any two characteristic functions for data valuation with the same domain $2^N$, $r_i$ and $r'_i$ be the corresponding reward values assigned to party $i$, and $v'^*$ be the maximum reward value under $v'_c$.
If the marginal contribution of party $i$ is larger under $v'_c$ than $v_c$ (e.g., by including a larger dataset) for at least a coalition, \emph{ceteris paribus}, then party $i$ should be assigned a larger reward value under $v'_c$ than $v_c$:\vspace{1mm}\\
$\forall i\in N\ \   
    [\exists C \subseteq N \setminus \{i\}\ \  v'_c(C\cup\{i\}) > v_c(C\cup\{i\})]
    \vspace{0.5mm}\\ \logicaland\ 
    [\forall B \subseteq N \setminus \{i\}\ \   v'_c(B\cup\{i\}) \geq v_c(B\cup\{i\})] \vspace{0.5mm} \\
\logicaland
[\forall A \subseteq N \setminus \{i\}\ \ v'_c(A) = v_c(A)] \logicaland (v'^* > r_i) 
 \Rightarrow r'_i > r_i\ .$
} 
\end{definition}

The following result verifies that the rectified $\rho$-Shapley fair reward values~\eqref{eq:rewarddef} in our modified reward scheme satisfy the above redefined incentive conditions R2, R3, R5 and  previously defined ones by selecting appropriate $\rho$ and $v^*$:

\begin{proposition}
\label{prop:cgmincentives}
Let $0 \leq \rho \leq 1$.
Using the new definitions of R2, R3, and F4 in Definitions~\ref{def:cgmfeasibility}, \ref{def:cgmweakefficiency}, and~\ref{def:cgmstrictmonotonicity}, 
the rectified $\rho$-Shapley fair reward values $(r_1,\ldots,r_n)$~\eqref{eq:rewarddef}
satisfy\vspace{1mm}\\
(a) R1 to R4 if $\rho$ and $v^*$ are set to satisfy\vspace{0.5mm}\\ $\forall i \in N\ \ (v_c(\{i\})\leq v^*) \wedge  (\brac{{\phi_i}/{\phi^*}}^\rho \times v^* \leq v(D_i \cup G))\ ,$\vspace{1mm}\\ 
(b) R1 to R5 if $\rho > 0$ and $v^*$ are set to satisfy\vspace{0.5mm}\\
$\forall i \in N\ \ v_c(\{i\}) \leq \brac{{\phi_i}/{\phi^*}}^\rho \times v^* \leq v(D_i \cup G)\ ,$ and\vspace{1mm}\\
(c) R1 to R6 if $\rho > 0$ and $v^*$ are set to satisfy\vspace{0.5mm}\\
$\forall i \in N\ \ v_c(C_i) \leq \brac{{\phi_i}/{\phi^*}}^\rho \times v^* \leq v(D_i \cup G)\ .$
\end{proposition}
On the other hand, R7 (i.e., group welfare) may not be achieved since  $\sum_{i\in N}r _i$ is maximized by $r_i = v(D_i\cup G)$ for each party $i \in N$ which may not be satisfied by any pair of feasible values of $\rho$ and $v^*$ given some  synthetic dataset $G \neq \emptyset$.
We will instead do our best to increase $\sum_{i\in N}r _i$ while giving precedence to satisfying the other incentive conditions in Proposition~\ref{prop:cgmincentives}, as detailed next.\vspace{0.5mm} 

\label{sect:linprog}

\noindent
{\bf Optimizing $\rho$ and $v^*$ via a Linear Program.} 
After computing the Shapley value $\phi_i$ of each party $i$~\eqref{shapley},
we have to optimize the values of 
$\rho$ and $v^*$ before assigning the resulting rectified $\rho$-Shapley fair reward values $(r_1,\ldots,r_n)$~\eqref{eq:rewarddef} to parties $1,\ldots,n$.
Let $\alpha_i := {\phi_i}/{\phi^*}$ denote the normalized Shapley value of party $i$, $v^{\text{min}}_i := v_c(\{i\})$, and $v^{\text{max}}_i := v(D_i \cup G)$.
We desire $v^*$ to be as large as possible to increase $\sum_{i\in N}r _i$ (group welfare).
Also, if we like $(r_1,\ldots,r_n)$~\eqref{eq:rewarddef} to be closer in proportion to  $(\alpha_1,\ldots,\alpha_n)$ (i.e., expected marginal contributions of parties $1,\ldots,n$) or purely Shapley fair (i.e., $\rho=1$),
then $\rho$ should be as close to $1$ as possible.\footnote{Alternatively, one may consider decreasing $\rho$ to increase $\sum_{i\in N}r _i$ (i.e., group welfare).}
Together with Proposition~\ref{prop:cgmincentives}b, it follows that the optimization problem can be framed as $\max_{v^*, \rho} (\log v^* + \epsilon \rho)$ subject to the constraints of $\forall i \in N\ \ v^{\text{min}}_i \leq v^* \alpha^\rho_i \leq v^{\text{max}}_i$ and $0 \leq \rho \leq 1$ where $\epsilon$ is a weight controlling the relative importance of $\rho$.\footnote{We consider $\log v^*$ instead of $v^*$ and the constraint of $\rho >0$ in  Proposition~\ref{prop:cgmincentives}b is relaxed to $\rho \geq 0$ in our optimization problem to facilitate its reformulation as a linear program. In our experiments, we have never observed $\rho=0$ since this term in the objective function is to be maximized.}
To additionally satisfy R6 (i.e., Proposition~\ref{prop:cgmincentives}c), we can set $v^{\text{min}}_i := v_c(C_i)$ instead.
Such a problem can be formulated as a  \emph{linear program} (LP) in inequality form that can be solved using standard LP solvers: $\min_{\mbf x} {\mbf c}^\top \mbf x$ subject to the constraint of $\mbf A\mbf x \preceq \mbf b$ where
$\mbf x := (\log v^*, \rho)^{\top}$, $\mbf c := (-1, -\epsilon)^{\top}$, $\mbf b := (\log v^{\text{max}}_1, \ldots, \log v^{\text{max}}_n, -\log v^{\text{min}}_1,\ldots, -\log v^{\text{min}}_n, 1, 0)^{\top}$, and $\mbf A$ is a matrix of size $2n+2$ by $2$ with the first column $(1,\ldots,1,-1,\ldots,-1, 0, 0)^{\top}$ and the second column $(\log\alpha_1,\ldots,\log\alpha_n,-\log\alpha_1,\ldots,-\log\alpha_n,1,-1)^{\top}$.
This formulation also informs us of a suitable choice of the synthetic dataset $G$: A sufficient but not necessary condition for the feasible set of the LP to be non-empty is $\min_{i\in N} v^{\text{max}}_i \geq \max_{i\in N} v^{\text{min}}_i$. When generating the synthetic dataset $G$, we may thus increase the size of $G$ until this condition is satisfied; we provide an intuition for why this works in \citesupp{app:increaseG}. 

\subsection{Distributing Synthetic Data Rewards to Parties via Weighted Sampling}
\label{force}
After assigning the rectified $\rho$-Shapley reward value $r_i$ to each party $i\in N$ (Sec.~\ref{lalaland}), we 
greedily sample
synthetic data points from $G$ to be distributed to each party $i$ as reward until the resulting $v(D_i \cup G_i)$ reaches the reward value $r_i$ (Sec.~\ref{yoda}).\footnote{Though $v(D_i \cup G_i)$ may slightly exceed the assigned reward value $r_i$ when sampling terminates due to  discreteness of synthetic data points,
such a margin diminishes when sufficiently large $|G_i|$ and $|G|$ are considered, as observed in our experiments (Sec.~\ref{expt}).}  
Specifically, let $\Delta_x := v(D_i \cup G_i \cup \{x\}) - v(D_i \cup G_i)$ denote the marginal increase in the value $v(D_i \cup G_i)$ of its dataset $D_i$ combined with its current synthetic dataset $G_i$ by sampling the synthetic data point $x$. 
In each iteration of our weighted sampling algorithm for distributing synthetic data reward to party $i$ (Algo.~\ref{alg:rewardrealization} in \citesupp{app:alg}),
we firstly perform min-max normalization to rescale $\Delta_x$ to $\bar \Delta_x$ for all synthetic data points $x\in G\setminus G_i$
to lie within the  $[0,1]$ interval. 
We compute the probability of each synthetic data point $x$ being sampled using the softmax function: $p(x) = {\exp{({\beta \bar \Delta_x}})}/{\sum_{x'\in G\setminus G_i} \exp{({\beta \bar \Delta_{x'}})}}$ where $\beta \in [0, \infty)$ is the inverse temperature hyperparameter.
Finally, we sample $x$ based on $p(x)$ and add it to $G_i$.
We repeat this process until $v(D_i \cup G_i)$ reaches  $r_i$.

As $\beta \rightarrow \infty$, the synthetic data points $x$ sampled by our algorithm tend to have larger $\Delta_x$. This leads to fewer sampled synthetic points $G_i$ as reward and thus a smaller $|D_i \cup G_i|$ when the resulting $v(D_i \cup G_i)$ reaches the assigned reward value $r_i$ and the sampling ends. 
This in turn results in a smaller   ${\text{MMD}}^2_u(\mathcal F, D_i \cup G_i, T)$,
by Proposition~\ref{prop:valuefunc}.
As $\beta \rightarrow 0$, the sampled synthetic points tend to have smaller $\Delta_x$; at $\beta = 0$, our algorithm performs random sampling since all synthetic points are weighted equally. By the same reasoning, this leads to a larger $|D_i \cup G_i|$ and thus a larger ${\text{MMD}}^2_u(\mathcal F, D_i \cup G_i, T)$.
So, $\beta$ implicitly controls the trade-off between the no.~of sampled synthetic points $G_i$ vs.~closeness to the distribution associated with reference dataset $T$.

Computing $v$ using~\eqref{eq:valuefunc} incurs $\mathcal{O}(s(s+t))$
time. Instead of naively recomputing $v$ for every synthetic data point $x$, the time needed to compute $\Delta_{x}$ can be reduced by performing a sequential update of $v$. By storing the values of $\inn{\boldsymbol 1_{[x \in S, x'\in T]}}{\mbf K}$ and $\inn{\boldsymbol 1_{[x, x'\in S]}}{\mbf K}$ at every iteration where $S=D_i \cup G_i$ (i.e., $s=|D_i \cup G_i|$),  $\Delta_{x}$ can be recomputed for each $x$ in $\mathcal{O}(s+t)$ time. The weighted sampling algorithm overall incurs $\mathcal O(n|G|^2(s+t)d)$ time. For computational details, refer to \citesupp{app:computation}.

\section{Kernel Selection}
\label{sect:kernelselection}

Recall from Sec.~\ref{junkyard} that our data valuation function~\eqref{eq:valuefunc} depends on the choice of kernel function $k$ which we will discuss here.
The $\log$ on $v(S)$ for different subsets $S\subseteq T$ being used in the LP (Sec.~\ref{sect:linprog}) requires $v(S)$ to be non-negative for all such subsets $S$.
The result below gives a sufficient condition on $k$ to guarantee the non-negativity of $v(S)$:

\begin{proposition}[\bf Lower bound of $k$ for non-negative $v(S)$]
\label{prop:nonneg}
Suppose that there exist some constants $\gamma$ and $\eta$ s.t.~$\gamma\leq k(x,x')\leq \eta\leq k^*$ for all $x,x'\in T$ and $x\neq x'$. Then,
\begin{multline}
\forall S\subseteq T\ \ [\gamma = ({t-2s})(k^* + (s-1)\eta)/(2s(t-s))] \Rightarrow \\ v(S)\geq 0 \ .
\label{esso}
\end{multline}
\end{proposition}

Ideally, we also want $v(S)$ to be monotonically increasing as the addition of a data point to a dataset should not decrease its value, as discussed in assumption C (Sec.~\ref{junkyard}).
The work of~\citet{kim2016examples} provides a sufficient condition on $k$ for $v$ to be a monotonic function:
\begin{theorem}[\bf Upper bound of $k$ for monotone $v(S)$~\cite{kim2016examples}]
\label{thm:monotone}
Suppose that there exists some constant $\eta$ s.t.~$k(x,x')\leq \eta\leq k^*$ for all $x,x'\in T$ and $x\neq x'$. Then,
\begin{multline}
\label{brat}
\forall S\subseteq T\ \
[\eta =  {tk^*}/({(s+1)(s(t-2)+t)})]\Rightarrow \\ [\forall x\in T\setminus S\ \ v(S\cup\{x\})\geq v(S)]\ .
\end{multline}
\end{theorem}

We can thus set an upper bound $\eta$~\eqref{brat} and a lower bound $\gamma$~\eqref{esso} of every off-diagonal component of $\mathbf K$ to guarantee the monotonicity and non-negativity of $v(S)$,  respectively. Unfortunately, no kernel exists to satisfy both sufficient conditions in  Theorem~\ref{thm:monotone} and Proposition~\ref{prop:nonneg} at the same time if the size of $S$ is less than half of that of the reference dataset $T$:
\begin{proposition}
\label{prop:impossible}
Let $\gamma$ and $\eta$ be set according to~\eqref{esso} and~\eqref{brat}. If $s < (t/2-1)$, then $\gamma >\eta\ .$
\end{proposition}
We prefer to guarantee the non-negativity of $v(S)$ (over monotonicity) for implementing the LP
and hence only satisfy the lower bound of $k$ (Proposition~\ref{prop:nonneg}). Trivially setting all components of $\mathbf K$ to $k^*$ satisfies this lower bound but is not useful as 
it values all datasets $S$ of the same size $s$ to be the same. 
Also, when the off-diagonal components of $\mathbf K$ are large, a non-monotonic behavior of $v(S)$ has been empirically observed, which agrees with our intuition formalized in Theorem~\ref{thm:monotone} that a monotone $v(S)$ is guaranteed by an upper bound $\eta$~\eqref{brat} of every off-diagonal component of $\mathbf K$. 
To strike a middle ground, we use a simple binary search algorithm to find the min.~length-scale of a kernel s.t.~$v(D_1),\ldots,v(D_n)$ are non-negative. We have observed in our experiments that this results in an approximately monotone $v$ and roughly $76\%$ of all synthetic data points added causing an increase in $v$. We have also empirically observed that the synthetic data points are more likely to result in a decrease in $v$ as more data points are added and $s$ increases, which aligns with our intuition given by Theorem~\ref{thm:monotone} that the upper bound $\eta$~\eqref{brat} to guarantee a monotone $v(S)$ decreases with a growing $s$ and thus becomes harder to satisfy.

\section{Experiments and Discussion}
\label{expt}
\begin{figure}
\begin{tabular}{cc}
\hspace{-2mm}\includegraphics[width=0.22\textwidth]{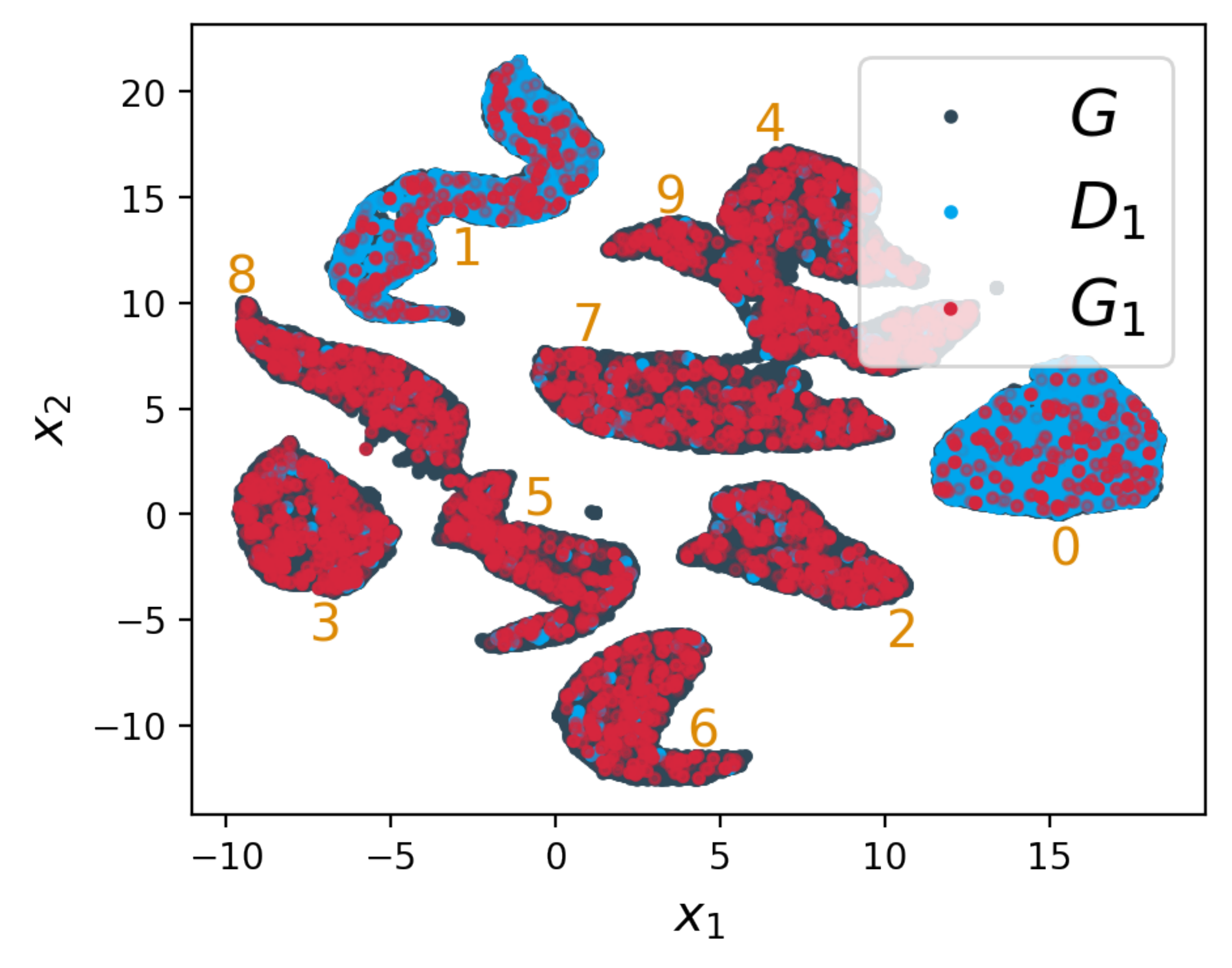} & \hspace{-3mm}\includegraphics[width=0.22\textwidth]{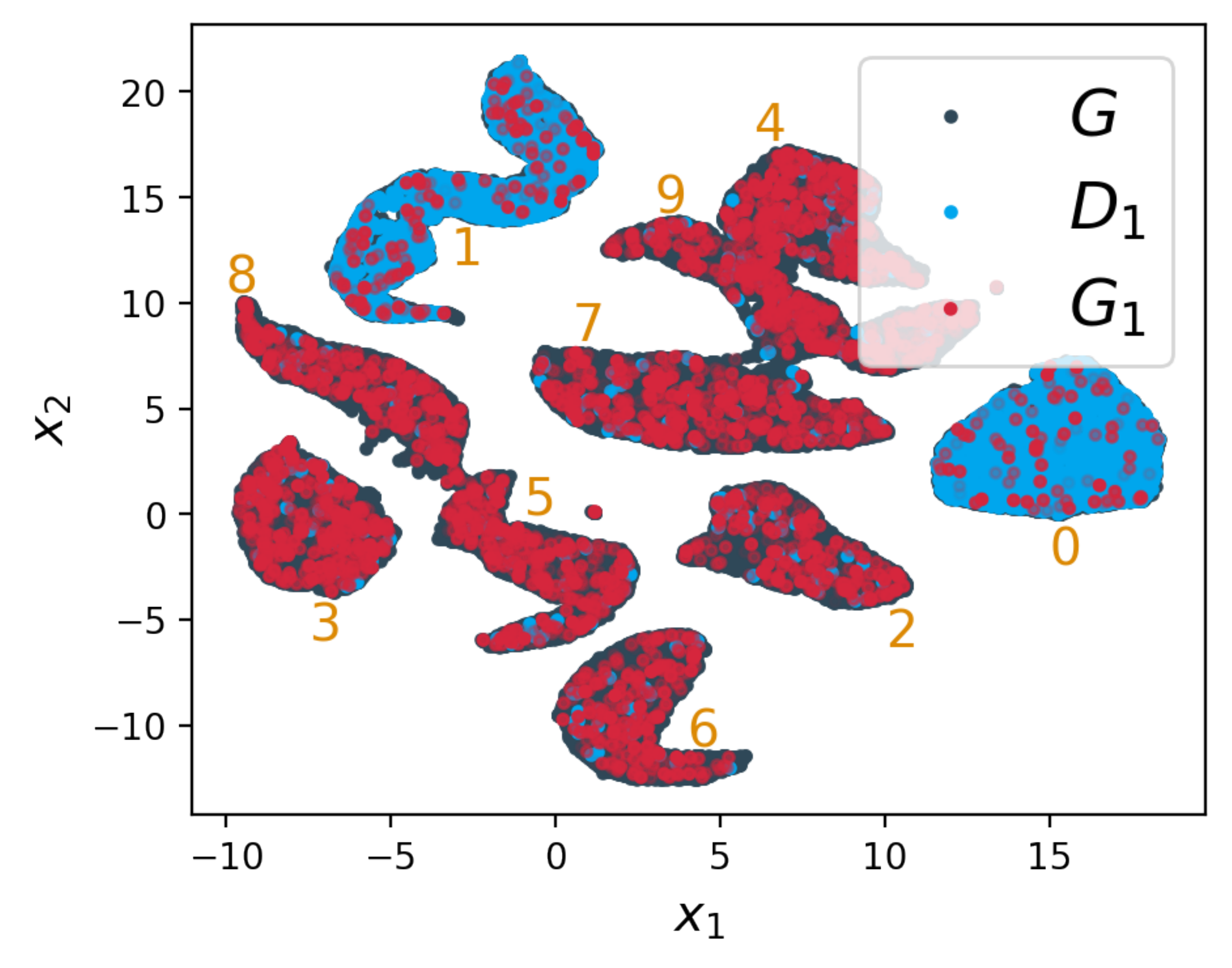}\vspace{-1.4mm}\\
\hspace{3mm}\small{$\beta=1$} & \hspace{2mm}\small{$\beta=2$}\\
\hspace{-2mm}\includegraphics[width=0.22\textwidth]{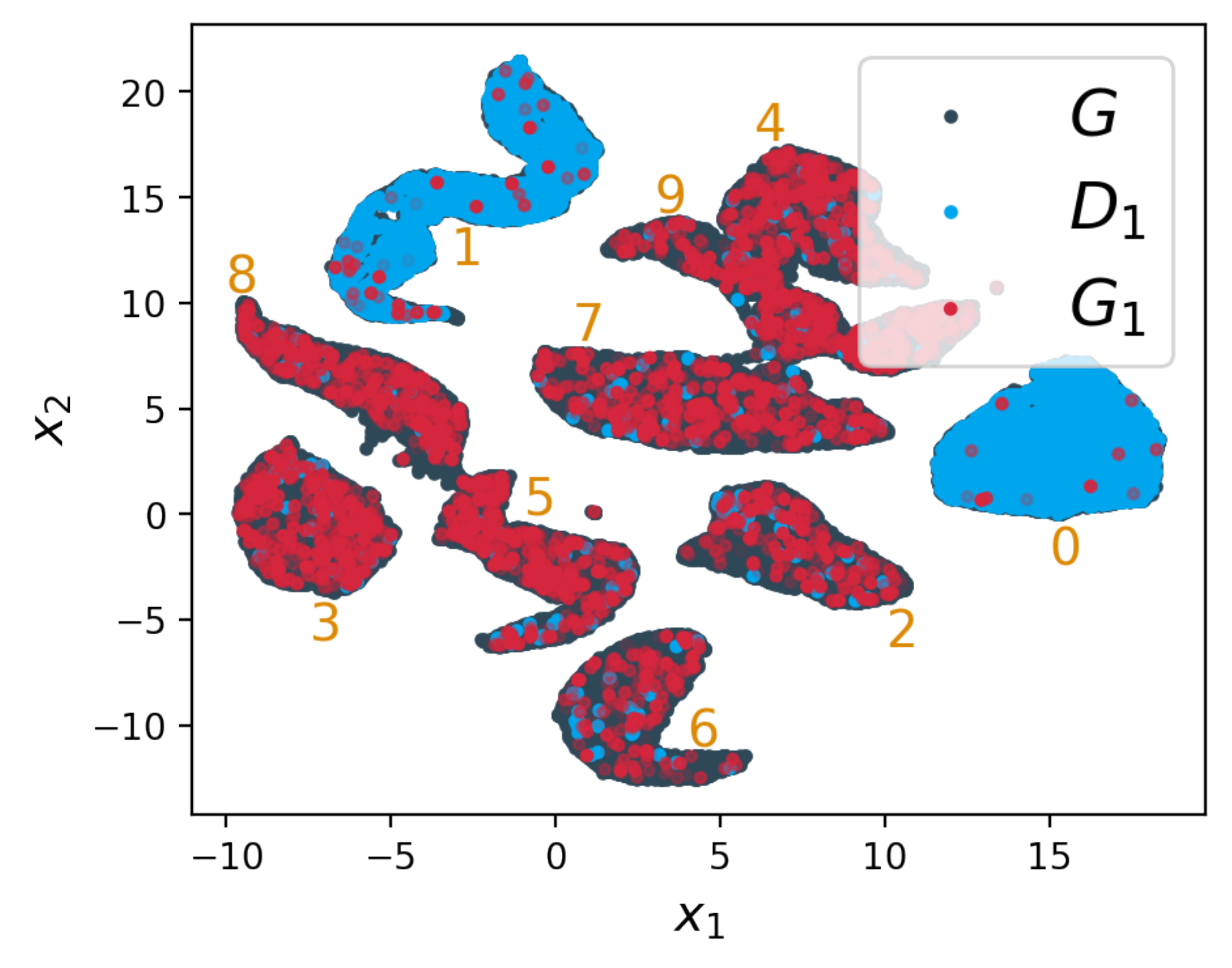} & \hspace{-3mm}\includegraphics[width=0.22\textwidth]{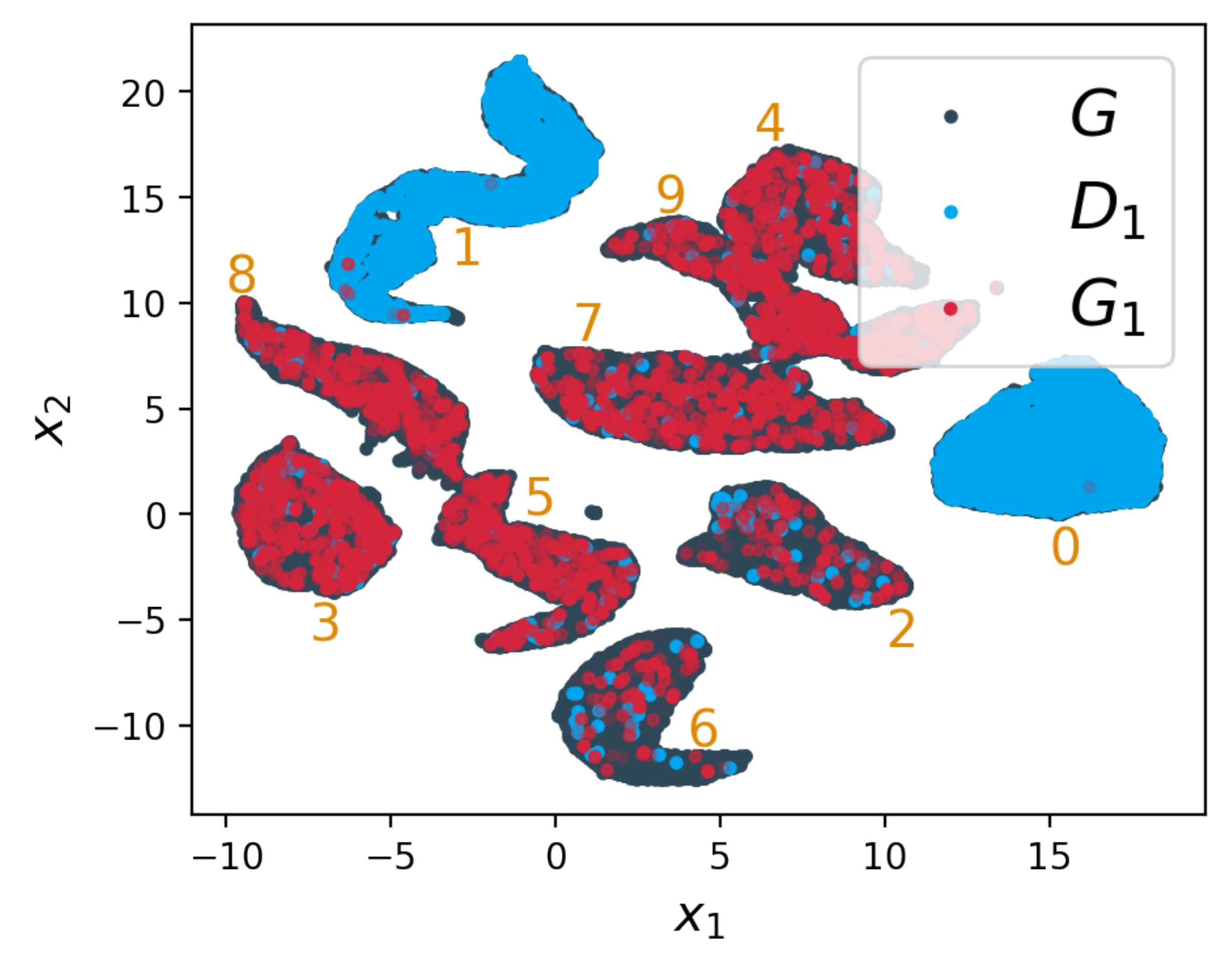}\vspace{-1.4mm}\\
\hspace{3mm}\small{$\beta=4$} & \hspace{2mm}\small{$\beta=8$}\vspace{-3mm}
\end{tabular}
     \caption{Synthetic data points $G_1$ (visualized in $2$-D embedding using UMAP~\cite{mcinnes2018umap}) as reward to party $1$ with varying $\beta$ in equal disjoint split. Each cluster has majority of the MNIST digit in yellow.}
     \label{fig:mnistp1}\vspace{-3mm}
\end{figure}
This section empirically evaluates the performance of our CGM framework using simulated and real-world datasets:\vspace{1mm}

\noindent
(a) \textbf{Simulated credit ratings.} We simulate a scenario where banks collaborate and share customer's \emph{credit ratings} (CR) indirectly to improve their predictions on the likelihood of default~\cite{TSAI2010374}. The banks serve different regions and hence own different subsets of the overall data distribution, but would like to predict well on the entire population for future expansion. Credit ratings are simulated using a $2$-D Gaussian mixture model dataset with $5$ clusters (classes) where the first dimension is the credit score and the second dimension is a measure of the likelihood of default.\vspace{0.5mm}

\noindent
(b) \textbf{Credit card fraud dataset.} We use the real-world \emph{credit card} (CC) fraud dataset \cite{dal2015calibrating} containing European credit card transactions such that most variables are transformed using PCA to yield $28$ principal components as features and an `Amount' variable denoting the amount transacted. We select the first $4$ principal components to create a $4$-D dataset, and separate the dataset into $5$ classes according to Amount percentiles so as to simulate collaborating banks serving different populations that tend to make transactions within certain ranges of amounts. Synthetic data are obtained by sampling from a distribution fit to the CC dataset with kernel density estimation.\vspace{0.5mm}

\noindent
(c) \textbf{Simulated medical imaging.} 
Synthetic image data is commonly used to improve performance on downstream ML tasks such as in medical imaging~\cite{bowles2018gan, frid2018gan, sandfort2019data}. 
We simulate a scenario where hospitals serving different populations share patients' data indirectly to improve predictions on medical imaging classification tasks on the whole population using the real-world MNIST \cite{lecun1998gradient} and CIFAR-10 \cite{krizhevsky2009learning} image datasets as surrogates. Synthetic data are obtained by sampling from pre-trained MMD GANs~\cite{binkowski2018demystifying}. We perform dimensionality reduction on the surrogate MNIST and CIFAR-10 image datasets to create $8$-D datasets, detailed in \citesupp{app:expdetails}.

CR and CC have $5$ classes, while MNIST and CIFAR-10 have $10$ classes. For all datasets, we simulate $5$ parties, and split the data among the $5$ parties in $2$ ways to simulate different settings of data sharing. The first split, which we refer to as `equal disjoint', is when each party has a large majority of data in $1$ class for CR and CC ($2$ for MNIST and CIFAR-10) and a small quantity of data in the other classes, and these majority classes are non-overlapping to simulate real-world settings where every party contributes data from a different restricted subset of the support of the data distribution. 
The second split, which we refer to as `unequal', is when the first $2$ parties have a uniform distribution of data over all classes while the remaining $3$ parties have a large majority of data in $3$ classes ($6$ for MNIST and CIFAR-10) and a small quantity of data in the rest of the classes to simulate real-world settings where some parties have `higher-quality' data than the other parties in terms of the coverage of the support of the data distribution. However, our CGM framework is \emph{not given these class labels} to simulate real-world scenarios where the class differences among parties are unknown. We use the squared exponential kernel with its length-scale computed using the binary search algorithm in Sec.~\ref{sect:kernelselection}. Our full CGM framework, which includes computing the normalized Shapley values $\alpha_1,\ldots,\alpha_n$ (i.e., expected marginal contributions) of parties $1,\ldots,n$, solving the LP to obtain their assigned rectified $\rho$-Shapley fair reward values $(r_1,\ldots,r_n)$, and running the weighted sampling algorithm for generating synthetic data points $G_1,\ldots,G_n$ to be distributed to them as rewards (Sec.~\ref{sect:rewardscheme}), 
is applied across all datasets and splits. \citesupp{app:expdetails} provides full details of the experimental settings, additional results,
and visualizations of the synthetic data rewards. As none of the prior work has previously considered synthetic data rewards, our results below set the baseline for future work.\vspace{0.5mm}

\begin{figure}
\begin{tabular}{cc}       \hspace{-3.5mm} \includegraphics[width=0.23\textwidth]{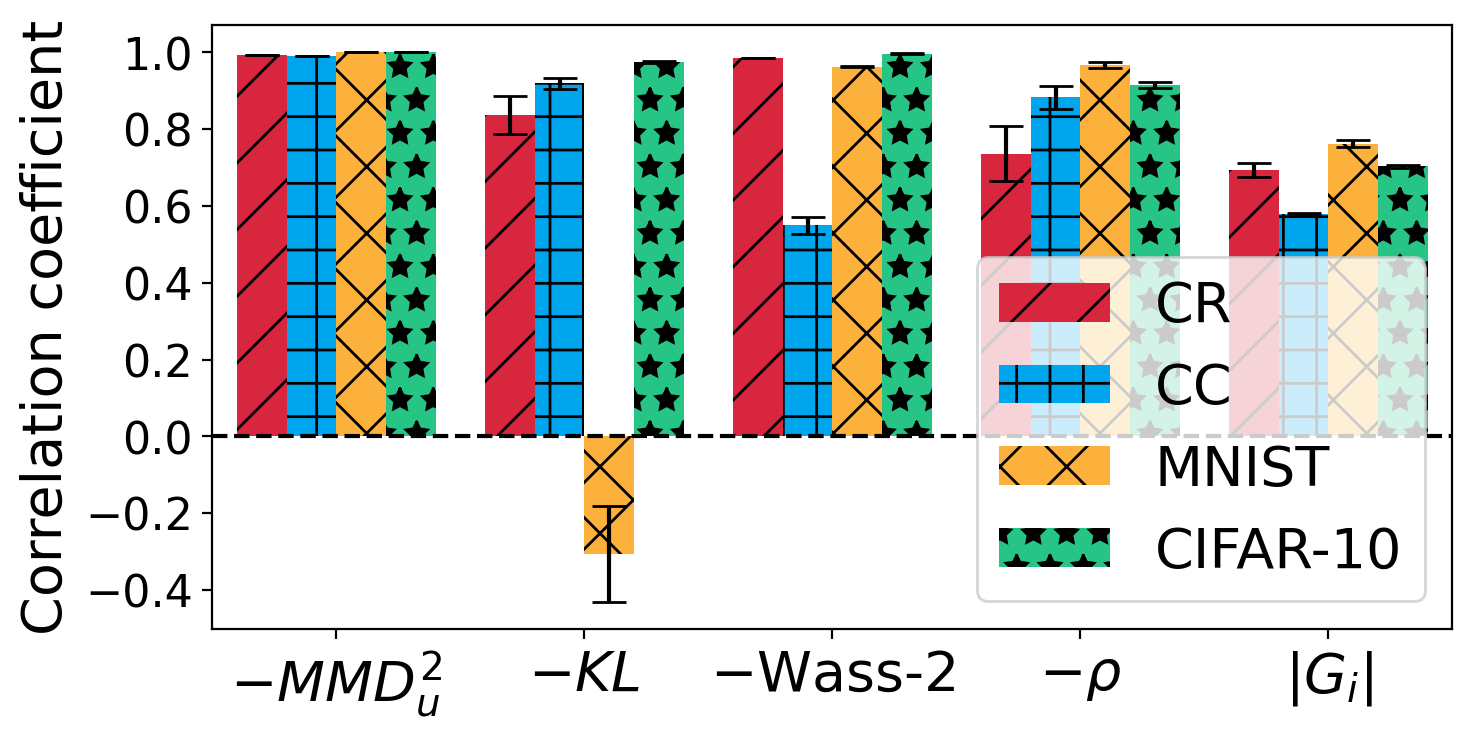} & \hspace{-4.5mm} \includegraphics[width=0.23\textwidth]{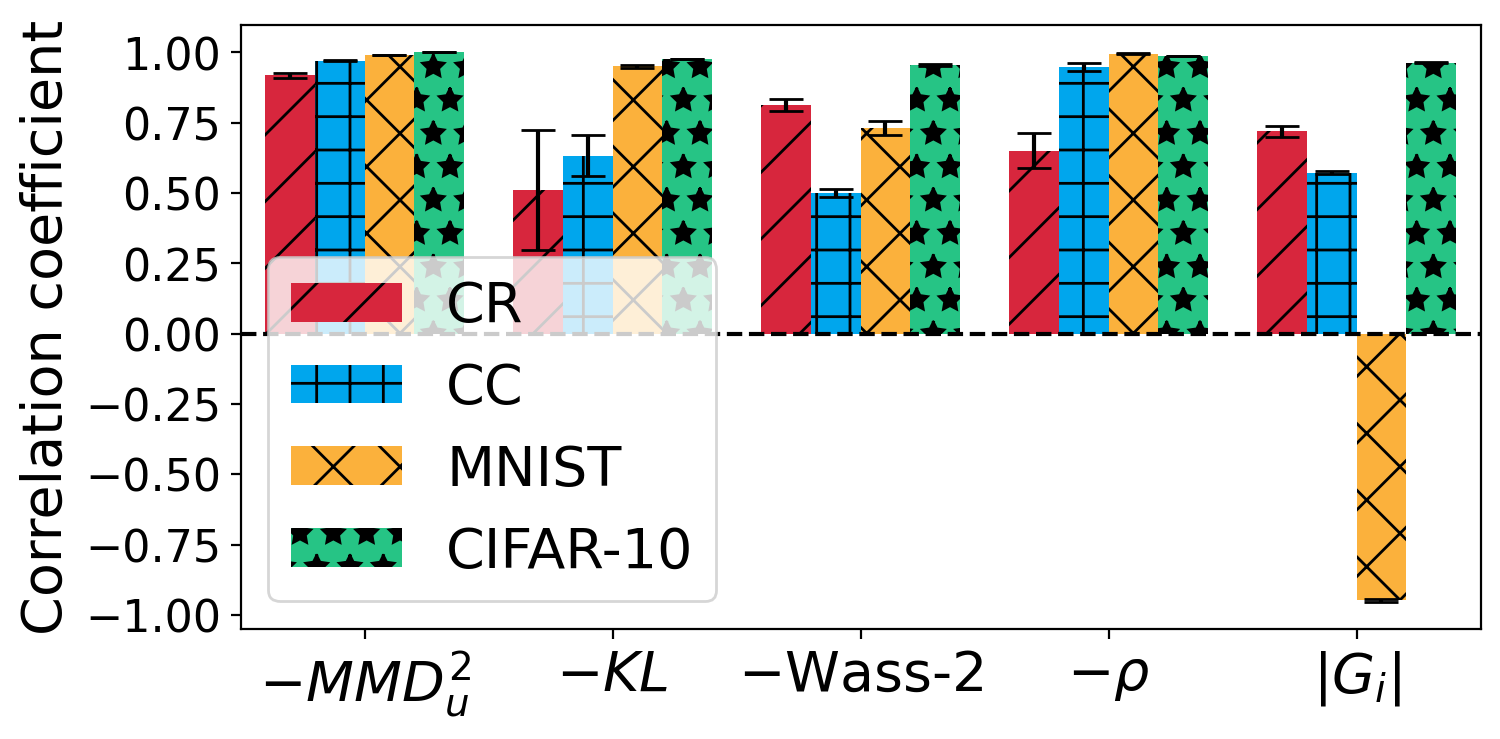}\vspace{-1mm}\\
\hspace{-3.5mm}        \small{(a) Equal disjoint split}
        & \hspace{-4.5mm} \small{(b) Unequal split}\vspace{-2mm}
\end{tabular}        
     \caption{Correlation of (negative of) metrics and $|G_i|$ with $\alpha_i$ (higher is better).}
     \label{fig:corrs}\vspace{-4mm}
\end{figure}

\noindent
\textbf{Assessing contributions of parties.} We assess whether our CGM framework can appropriately quantify the expected marginal contributions of the parties via their Shapley values (Sec.~\ref{sect:rewardscheme}).  
Results are reported in \citesupp{app:addresults}: As expected, 
very large $\alpha_i$'s are observed for parties in the unequal split with full class distribution, while the $\alpha_i$'s are typically more evenly spread in the equal disjoint split.\vspace{0.5mm} 

\noindent
\textbf{Role of inverse temperature hyperparameter $\beta$.} 
To substantiate our claim that $\beta$ in the weighted sampling algorithm (Sec.~\ref{force}) controls the trade-off between the no.~of synthetic data points as rewards vs.~negative unbiased MMD (i.e., closeness to empirical distribution associated with $T$), we report the correlation of $\beta$ with them 
in \citesupp{app:addresults}: $\beta$ is observed to be highly negatively correlated with the no.~of synthetic data points and highly positively correlated with negative unbiased MMD, which aligns with our reasoning in Sec.~\ref{force}. Also, Fig.~\ref{fig:mnistp1} shows that as $\beta$ increases, the algorithm samples fewer synthetic data points but they are more dissimilar from a party's original dataset.\vspace{0.5mm}

\noindent
\textbf{Are synthetic data rewards distributed to parties and their downstream ML task performances commensurate to their contributions?}
We firstly assess whether our CGM framework can distribute synthetic data points $G_i$ to each party $i$ as reward such that the closeness of the empirical distributions associated with $D_i \cup G_i$ vs.~reference dataset $T$ correlates well with its expected marginal contribution via the normalized Shapley value $\alpha_i$. We quantify such a closeness using 
$4$ metrics (which we take the negative of so that higher is better): (a) unbiased MMD estimate~\eqref{seb}, (b) an estimate of reverse Kullback-Leibler divergence based on $k$-nearest neighbors~\cite{perez2008kullback} averaged over $k=2,\ldots,6$, (c) Wasserstein-$2$ distance between multivariate Gaussians fit to $D_i \cup G_i$ vs.~$T$ (i.e., how Fr\'echet Inception distance for evaluating GANs is computed~\cite{heusel2017gans}), and (d) class imbalance $\rho$  calculated with $\rho_i := (1/m) \sum_{y=1}^{m}p^2_y$ where $m$ is the no.~of classes and $p_y$ is the proportion of data points in party $i$'s combined dataset $D_i \cup G_i$ belonging to class $y$. In all datasets, $T$ is equally distributed among the classes and hence achieves a minimum for $\rho$. We also measure the correlation of the no.~$|G_i|$ of synthetic data points as reward to party $i$ with $\alpha_i$. Fig.~\ref{fig:corrs} shows results of the mean and standard error of the correlations over varying $\beta = 1, 2, 4, 8$ in the weighted sampling. It can be observed that across all splits, datasets, and metrics, the negative of the metrics and $|G_i|$ mostly display highly positive correlations with $\alpha_i$, as desired. We defer the discussion of the few negative correlations to \citesupp{app:negativecorrelations}.\vspace{0.5mm}

After distributing the synthetic data rewards to the parties, we assess whether their performances on downstream ML tasks (from augmenting their real data with synthetic data) correlate well with their expected marginal contributions via $\alpha_i$. We simulate supervised learning scenarios where each party trains an SVM on its real and synthetic data and predicts the class labels on unseen real data. For the real-world CC, MNIST, and CIFAR-10 datasets, the correlations of their classification accuracies with $\alpha_i$ (averaged over $\beta$) are, respectively, $0.523$, $0.459$, and $0.174$ in the equal disjoint split, and $0.791$, $0.338$, and $0.835$ in the unequal split. We observe positive correlations overall, thus confirming our hypothesis that the parties' downstream ML task performances are commensurate to their contributions.\vspace{-0.8mm}

\section{Conclusion}
This paper has described a novel CGM framework that incentivizes collaboration among self-interested parties to contribute data to a pool for training a generative model, from which synthetic data are drawn and distributed to the parties as rewards commensurate to their contributions.
The CGM framework comprises an MMD-based data valuation function whose bounds weakly increase with a growing dataset quantity and an improved closeness of the empirical distributions associated with the dataset vs.~the reference dataset,
a reward scheme formulated as an LP for guaranteeing incentives like fairness, and a weighted sampling algorithm with the flexibility of controlling the trade-off between no.~of synthetic data points as reward vs.~the closeness described above.
For future work, we will consider deep kernels to automatically learn useful representations for data valuation and prove stronger guarantees on the non-negativity and monotonicity of our data valuation function.

\subsubsection{Acknowledgments.} This research is supported by the National Research Foundation, Singapore under its AI Singapore Programme (Award No: AISG$2$-RP-$2020$-$018$). Any opinions, findings and conclusions or recommendations expressed in this material are those of the author(s) and do not reflect the views of National Research Foundation, Singapore. Sebastian Tay is supported by the Institute for Infocomm Research of Agency for Science, Technology and Research (A*STAR).

\bibliography{main.bib}
\newpage
\onecolumn
\appendix
\section{Kernels and MMD}\label{app:kernel}
The MMD requires the use of \textit{characteristic} kernels where a kernel is characteristic if $\text{MMD}(\mathcal F, \mathcal D', \mathcal D) = 0 \Leftrightarrow \mathcal D' = \mathcal D$ \cite{fukumizu2008characteristic}. Without any assumption, it is not straightforward to determine if an arbitrary kernel is characteristic. However, \cite{sriperumbudur2008injective} have shown that a stationary kernel is characteristic if it can be written in the form $\psi(x-x’)$ where $\psi$ is a bounded, continuous, real-valued, positive definite function on $\mathbb R^d$ with a compact support, which is relatively straightforward to check. This motivates the use of stationary kernels with the MMD metric. The stationary kernels include the class of isotropic kernels whose $k(x, x’)$ only depends on the distance between data points $\lVert x-x’\rVert$. The squared exponential and Laplace kernels are examples of isotropic kernels shown by \cite{fukumizu2007kernel} to be characteristic.

In our experiments, we opt to use the squared exponential kernel
\begin{equation}
    k(x, x') = \exp\brac{-\frac{\norm{x-x'}^2}{2\ell}} \ ,
\end{equation}
since it is characteristic and commonly used as a `default' kernel as it is \textit{universal}: any continuous function in a compact space with bounded norm can be approximated to arbitrary accuracy with some linear combination of a universal kernel’s basis functions \cite{micchelli2006universal}.

\section{Justification of Assumption A}\label{app:assumptionA}
To see the validity of assumption A, suppose that there are $10$ parties $i= 1,\ldots, 10$, each of whom intends to train a binary classifier on MNIST digits to predict if a digit is $i$ or not. But, each party $i$'s dataset $D_i$ only comprises images of digit $i$ and hence a restricted subset of the support of $\mathcal D$ (i.e., imbalanced data). If every party can access images of all $10$ digits, then their classifiers
can perform better due to availability of negative examples. In general, imbalanced data in classification problems yield trained ML models that make  predictions biased towards the majority classes, which is an issue receiving significant attention \cite{krawczyk2016learning}.

\section{Proofs}

\subsection{Proof of Proposition~\ref{prop:valuefunc}}
\label{app:proofvaluefuncprop}
\begin{manualproposition}{\ref{prop:valuefunc}}
Let $k^*$ be the value of every diagonal component of $\mathbf K$ s.t.~$k^* := k(x, x)\geq k(x, x')$ for all $x,x'\in T$, and 
$\sigma_S := \left \langle s^{-2} \boldsymbol1_{[x, x' \in S]}, \mbf K \right\rangle$. Then,  $v(S)$~\eqref{eq:valuefunc} can be re-expressed as
\begin{equation*}
    v(S) = (s-1)^{-1}(\sigma_S-k^*)-{\textup{MMD}}^2_u(\mathcal F, S, T) + c
\end{equation*}
where $c$ is a constant (i.e., independent of $S$).
\end{manualproposition}

To ease exposition, we will express the unbiased estimate ${\text{MMD}}^2_u(\mathcal F, S, T)$ and biased estimate ${\text{MMD}}^2_b(\mathcal F, S, T)$ of the squared MMD in the form of summations here instead of the matrix inner products~\eqref{seb}.

\begin{proof}
Define the terms
\begin{align}
    &\sigma_S := \frac{1}{s^2}\sum_{x \in S}\sum_{x' \in S}k(x, x'),\quad \sigma_{ST} := \frac{2}{st}\sum_{x \in S}\sum_{x' \in T}k(x, x'),\quad \sigma_T := \frac{1}{t^2}\sum_{x \in T}\sum_{x' \in T}k(x, x')  \label{jumbo}\\
    &\sigma'_S := \frac{1}{s(s-1)}\sum_{x \in S}\sum_{x' \in S, x' \neq x}k(x, x'),\qquad \sigma'_T := \frac{1}{t(t-1)}\sum_{x \in T}\sum_{x' \in T, x \neq x'}k(x, x')\ .\nonumber
\end{align}
The unbiased and biased estimates of the squared MMD can then be written as
\begin{gather}
   {\text{MMD}}^2_u(\mathcal F, S, T) = \sigma'_S - \sigma_{ST} + \sigma'_T \nonumber\\ 
    {\text{MMD}}^2_b(\mathcal F, S, T) = \sigma_S - \sigma_{ST} + \sigma_T \nonumber
\end{gather}
with the same middle term. Then,
\begin{align}
    {\text{MMD}}^2_u(\mathcal F, S, T) - {\text{MMD}}^2_b(\mathcal F, S, T) &= (\sigma'_S - \sigma_S) + (\sigma'_T - \sigma_T) \nonumber\\
    &= \brac{\frac{s}{s-1}\sigma_S-\frac{1}{s(s-1)}\sum_{x\in S} k(x, x) -\sigma_S} + (\sigma'_T - \sigma_T) \nonumber\\
    &= \frac{1}{s-1}\sigma_S-\frac{k^*}{s-1} + (\sigma'_T - \sigma_T) \nonumber\\
    &= \frac{1}{s-1}(\sigma_S - k^*) + (\sigma'_T - \sigma_T) \nonumber\\
    - {\text{MMD}}^2_b(\mathcal F, S, T) &= \frac{1}{s-1}(\sigma_S - k^*) -{\text{MMD}}^2_u(\mathcal F, S, T)  + (\sigma'_T - \sigma_T) \nonumber\\
    v(S) - \sigma_T  &= \frac{1}{s-1}(\sigma_S - k^*) -{\text{MMD}}^2_u(\mathcal F, S, T)  + (\sigma'_T - \sigma_T) \label{eq:mumbo}\\
    v(S) &= \frac{1}{s-1}(\sigma_S - k^*) -{\text{MMD}}^2_u(\mathcal F, S, T)  + \sigma'_T
    \label{eq:subtarget1}
\end{align}
where~\eqref{eq:mumbo} follows from the definition of $v(S)$~\eqref{eq:valuefunc}.
Noting that $\sigma'_T$ is constant w.r.t.~$S$ and depends only on $T$ completes the proof.
\end{proof}

\subsection{Proof of Corollary~\ref{coro:valuefunc}}
\label{app:proofvaluefunccoro}

\begin{manualcorollary}{\ref{coro:valuefunc}}
Suppose that there exist some constants $\gamma$ and $\eta$ s.t.~$\gamma\leq k(x,x')\leq \eta\leq k^*$ for all $x,x'\in T$ and $x\neq x'$.
Then,
\begin{equation*}
    s^{-1}(\gamma - k^*)-{\textup{MMD}}^2_u(\mathcal F, S, T) + c \ \ \leq \ \ v(S) \ \ \leq\ \  s^{-1}(\eta - k^*)-{\textup{MMD}}^2_u(\mathcal F, S, T) + c\ .
\end{equation*}
\end{manualcorollary}

To ease exposition, we will express the unbiased estimate ${\text{MMD}}^2_u(\mathcal F, S, T)$ and biased estimate ${\text{MMD}}^2_b(\mathcal F, S, T)$ of the squared MMD in the form of summations here instead of the matrix inner products~\eqref{seb}.
\begin{proof}
From~\eqref{jumbo}, $\sigma_S$ can be rewritten as a sum of the $s(s-1)$ off-diagonal components of $\mathbf K$ and the $s$ diagonal components of $\mathbf K$:
\begin{equation*}
    \sigma_S = \frac{1}{s^2}\brac{\sum_{x \in S}\sum_{x' \in S, x \neq x'}k(x, x') + \sum_{x \in S} k^*} = \frac{1}{s^2}\brac{\sum_{x \in S}\sum_{x' \in S, x \neq x'}k(x, x') + s k^*}.
\end{equation*}
Since there exist some constants $\gamma$ and $\eta$ s.t.~$\gamma\leq k(x,x')\leq \eta\leq k^*$ for all $x,x'\in T$ and $x\neq x'$, it follows that
\begin{equation}
\begin{array}{rcccl}
    \displaystyle\frac{1}{s^2}\brac{s(s-1)\gamma + sk^*} &\leq & \sigma_S &\leq & \displaystyle\frac{1}{s^2}\brac{s(s-1)\eta + sk^*} \vspace{1mm}\\
    \displaystyle\frac{1}{s}\brac{(s-1)\gamma + k^*} &\leq & \sigma_S &\leq &  \displaystyle\frac{1}{s}\brac{(s-1)\eta + k^*}\ . 
\end{array}
\label{kingkong}
\end{equation}
Substituting~\eqref{kingkong} into the first term on the RHS of~\eqref{eq:subtarget1},
\begin{equation}
\begin{array}{rcccl}
    \displaystyle\frac{1}{s}\brac{\gamma + \frac{1}{s-1}k^*} - \frac{1}{s-1}k^* &\leq & \displaystyle\frac{1}{s-1}(\sigma_S - k^*) &\leq &\displaystyle\frac{1}{s}\brac{\eta + \frac{1}{s-1}k^*} - \frac{1}{s-1}k^* \vspace{1mm}\\ 
    \displaystyle\frac{1}{s}\brac{\gamma + \frac{1-s}{s-1}k^*} &\leq & \displaystyle\frac{1}{s-1}(\sigma_S - k^*) &\leq &\displaystyle\frac{1}{s}\brac{\eta + \frac{1-s}{s-1}k^*} \\
    \displaystyle\frac{1}{s}\brac{\gamma - k^*} &\leq &\displaystyle \frac{1}{s-1}(\sigma_S - k^*) &\leq &\displaystyle\frac{1}{s}\brac{\eta - k^*}\ .
\end{array}
\label{kingkong2}
\end{equation}
Substituting~\eqref{kingkong2} back into~\eqref{eq:subtarget1} completes the proof.
\end{proof}

\subsection{Proof of Proposition~\ref{prop:etalesser}}
\label{app:proofetalesser}

\begin{manualproposition}{\ref{prop:etalesser}}
If $v^* = v_c(N)$ and $\rho$ satisfies  $\brac{{\phi_i}/{\phi^*}}^\rho \times v^* < v_c(\{i\})$ for some party $i\in N$, then $(r_1,\ldots,r_n)$~\eqref{eq:rewarddef} may not satisfy R5 due to possibly violating F3 (i.e., strict desirability).
\end{manualproposition}

\begin{proof}
We show a proof by example. Consider the following non-negative and monotone characteristic function for data valuation $v_c$ over a set $N$ of $4$ parties:
\begin{align*}
    v_c(\emptyset) &= 0\ , & v_c(\{ 2,3 \}) &= 10\ , \\
    v_c(\{1\}) &= 9\ , & v_c(\{ 2,4 \}) &= 10\ , \\
    v_c(\{2\}) &= 9\ , & v_c(\{ 3,4 \}) &= 10\ , \\
    v_c(\{3\}) &= 10\ , & v_c(\{ 1,2,3 \}) &= 18\ , \\
    v_c(\{4\}) &= 10\ , & v_c(\{ 1,2,4 \}) &= 18\ , \\
    v_c(\{ 1,2 \}) &= 18\ , & v_c(\{ 1,3,4 \}) &= 11\ , \\
    v_c(\{ 1,3 \}) &= 10\ , & v_c(\{ 2,3,4 \}) &= 10\ , \\
    v_c(\{ 1,4 \}) &= 11\ , & v_c(\{ 1,2,3,4 \}) &= 18\ .
\end{align*}
In the above constructed example,
there exist parties $3$ and $4$ s.t.~the marginal contribution of party $4$ is more than that of party $3$ for coalition $\{1\}$,  and for all coalitions not including both parties $3$ and $4$, the marginal contribution of party $4$ is at least that of party $3$.

The Shapley values of parties $1$, $2$, $3$, and $4$ are $\phi_1={13}/{2}, \phi_2=6, \phi_3={8}/{3}, \phi_4= {17}/{6}$, and their corresponding normalized Shapley values are $\alpha_1=1, \alpha_2={3}/{13}, \alpha_3={16}/{39}, \alpha_4={17}/{39}$. We adversarially choose $\rho = 1$ so that $\brac{{\phi_i}/{\phi^*}}^\rho \times v^* = 18\alpha_i < v_c(\{i\})$ for both parties $3$ and $4$. Then, $r_3 = \max\left\{10, 18\times {16}/{39}\right\} = 10$ and $r_4 = \max\left\{10, 18\times{17}/{39}\right\} = 10$. Both parties $3$ and $4$ are assigned reward values of $r_3 = r_4 = 10$, which implies $r_4 \ngtr r_3$.
So, the (F3) strict desirability property is violated.
\end{proof}

\subsection{Proof of Proposition~\ref{prop:cgmincentives}}
\label{app:proofofetagreater}

\begin{manualproposition}{\ref{prop:cgmincentives}}
Let $0 \leq \rho \leq 1$.
Using the new definitions of R2, R3, and F4 in Definitions~\ref{def:cgmfeasibility}, \ref{def:cgmweakefficiency}, and~\ref{def:cgmstrictmonotonicity}, 
the rectified $\rho$-Shapley fair reward values $(r_1,\ldots,r_n)$~\eqref{eq:rewarddef}
satisfy\vspace{0.5mm}\\
(a) R1 to R4 if $\rho$ and $v^*$ are set to satisfy $\forall i \in N\ \ (v_c(\{i\})\leq v^*) \wedge  (\brac{{\phi_i}/{\phi^*}}^\rho \times v^* \leq v(D_i \cup G))\ ,$\vspace{0.5mm}\\
(b) R1 to R5 if $\rho > 0$ and $v^*$ are set to satisfy $\forall i \in N\ \ v_c(\{i\}) \leq \brac{{\phi_i}/{\phi^*}}^\rho \times v^* \leq v(D_i \cup G)\ ,$ and\vspace{0.5mm}\\
(c) R1 to R6 if $\rho > 0$ and $v^*$ are set to satisfy $\forall i \in N\ \ v_c(C_i) \leq \brac{{\phi_i}/{\phi^*}}^\rho \times v^* \leq v(D_i \cup G)\ .$
\end{manualproposition}

\begin{proof}
We will first prove (a).

(\textbf{R1})\ \textbf{Non-negativity.} Since $v_c$ is non-negative and $\forall i \in N\ \ v_c(\{i\})\leq v^*$ is given, it follows from~\eqref{eq:rewarddef} that $\forall i \in N\ \ r_i\geq 0$. 

(\textbf{R2})\ \textbf{CGM Feasibility.}
Since $v_c$ is monotone,
$\forall i \in N\ \ v_c(\{i\})\leq v(D_i \cup G)$.
Since $\forall i \in N\ \ \brac{{\phi_i}/{\phi^*}}^\rho \times v^*\leq v(D_i \cup G)$ is given, it follows from~\eqref{eq:rewarddef} that  $\forall i \in N\ \ r_i = \max \{v_c(\{i\}), \brac{{\phi_i}/{\phi^*}}^\rho \times v^*\}\leq v(D_i \cup G)$.

(\textbf{R3})\ \textbf{CGM Weak Efficiency.} Since $\forall i \in N\ \ (v_c(\{i\})\leq v^*)$ is given, it follows from~\eqref{eq:rewarddef} that the party $i$ with the maximum Shapley value $\phi_i = \phi^*$ should be assigned a reward value of  $r_i=\max\{v_c(\{i\}),\brac{{\phi_i}/{\phi^*}}^\rho \times v^*\} = \max\{v_c(\{i\}),v^*\} = v^*$.

(\textbf{R4})\  \textbf{Individual Rationality.}  It follows from~\eqref{eq:rewarddef} that $\forall i \in N\ \ r_i=\max\{v_c(\{i\}),\brac{{\phi_i}/{\phi^*}}^\rho \times v^*\} \geq v_c(\{i\})$.

Since we have proven above that (a) holds, (b) holds for R1 to R4 since $\forall i \in N\ \ v_c(\{i\}) \leq \brac{{\phi_i}/{\phi^*}}^\rho \times v^* \leq v^*$ as $\brac{{\phi_i}/{\phi^*}}^\rho \leq 1$, and that (c) holds for R1 to R4 since $\forall i \in N\ \ v_c(\{i\})\leq v_c(C_i) \leq \brac{{\phi_i}/{\phi^*}}^\rho \times v^* \leq v^*$ as $v_c(\{i\})\leq v_c(C_i)$ due to a monotone $v_c$.

To prove that (b) holds for R5, since $\forall i \in N\ \ v_c(\{i\}) \leq \brac{{\phi_i}/{\phi^*}}^\rho \times v^*$ is given, it follows from~\eqref{eq:rewarddef} that  $\forall i \in N\ \ r_i = \max \{v_c(\{i\}), \brac{{\phi_i}/{\phi^*}}^\rho \times v^*\}= \brac{{\phi_i}/{\phi^*}}^\rho \times v^*$ 
which is the same as that of~\citet{simcollaborative} (Theorem~\ref{thm:sim}) up to a multiplicative constant ${v^*}/{v_c(N)}$. 
Properties F1 to F3 are satisfied since their proofs remain the same as that of~\citet{simcollaborative} even if the reward value of~\citet{simcollaborative} is multiplied by a multiplicative constant. 
For property (F4) CGM Strict Monotonicity that is redefined in Definition~\ref{def:cgmstrictmonotonicity}, 
replacing all instances of $v_N$ with $v^*$ and $v'_N$ with $v'^*$ in the proof of~\citet{simcollaborative} proves that the redefined F4 holds.

Since we have proven above that (b) holds for R5, (c) holds for R5 since $\forall i \in N\ \ v_c(\{i\})\leq v_c(C_i) \leq \brac{{\phi_i}/{\phi^*}}^\rho \times v^*$ as $v_c(\{i\})\leq v_c(C_i)$ due to a monotone $v_c$.

To prove that (c) holds for R6, since $\forall i \in N\ \ v_c(C_i) \leq \brac{{\phi_i}/{\phi^*}}^\rho \times v^*$ is given and it is proven above that $\forall i \in N\ \ r_i =  \brac{{\phi_i}/{\phi^*}}^\rho \times v^*$, $\forall i \in N\ \ v_c(C_i) \leq r_i$ which guarantees that R6 holds due to the monotone $v_c$, as explained in~\cite{simcollaborative}. 
\end{proof}

\subsection{Proof of Proposition~\ref{prop:nonneg}}
\label{app:proofnonneg}

\begin{manualproposition}{\ref{prop:nonneg}}[\bf Lower bound of $k$ for non-negative $v(S)$]
Suppose that there exist some constants $\gamma$ and $\eta$ s.t.~$\gamma\leq k(x,x')\leq \eta\leq k^*$ for all $x,x'\in T$ and $x\neq x'$. Then,
$$
\forall S\subseteq T\ \ [\gamma = ({t-2s})(k^* + (s-1)\eta)/(2s(t-s))]\Rightarrow v(S)\geq 0 \ .
$$
\end{manualproposition}
To ease exposition, we will express the unbiased estimate ${\text{MMD}}^2_u(\mathcal F, S, T)$ and biased estimate ${\text{MMD}}^2_b(\mathcal F, S, T)$ of the squared MMD in the form of summations here instead of the matrix inner products~\eqref{seb}.
\begin{proof}
For $v(S)$~\eqref{eq:valuefunc} to be non-negative, we require that
\begin{equation*}
    v(S) = \frac{2}{st} \sum_{x \in S} \sum_{x' \in T} k(x, x') - \frac{1}{s^2} \sum_{x \in S} \sum_{x' \in S} k(x ,x') \geq 0
\end{equation*}
which may be rewritten as
\begin{align}
    \frac{2}{st} \sum_{x \in S} \sum_{x' \in T \setminus S} k(x, x') + \brac{\frac{2}{st} - \frac{1}{s^2}}\sum_{x \in S} \sum_{x' \in S} k(x, x') &\geq 0 \nonumber\\
    \frac{2}{st} \sum_{x \in S} \sum_{x' \in T \setminus S} k(x, x') &\geq \brac{\frac{1}{s^2} - \frac{2}{st}}\sum_{x \in S} \sum_{x' \in S} k(x, x')\ .\label{booboo}
\end{align}
To guarantee the non-negativity of $v(S)$, we can derive a sufficient condition on $k$ in terms of its lower bound $\gamma$
to ensure that~\eqref{booboo} holds. To achieve this, observe that the RHS of~\eqref{booboo} comprises
$s$ summation terms equal to $k^*$ (i.e., if $x=x'$) and $s(s-1)$ summation terms that are at most $\eta$ (i.e., if $x\neq x'$). Then, 
\begin{equation*}
\brac{\frac{1}{s^2}-\frac{2}{st}}(sk^*+s(s-1)\eta) \geq \brac{\frac{1}{s^2} - \frac{2}{st}}\sum_{x \in S} \sum_{x' \in S} k(x, x')\ .
\end{equation*}
Similarly, observe that the LHS of~\eqref{booboo} comprises $s(t-s)$ summation terms that are at least $\gamma$. Then,
\begin{equation*}
\frac{2}{st} \sum_{x \in S} \sum_{x' \in T \setminus S} k(x, x') \geq \frac{2}{st}s(t-s)\gamma\ .
\end{equation*}
Therefore, if the lower bound of the LHS of~\eqref{booboo} is at least the upper bound of the RHS of~\eqref{booboo}, then~\eqref{booboo} holds:
\begin{align*}
    \frac{2}{st}s(t-s)\gamma &\geq \brac{\frac{1}{s^2}-\frac{2}{st}}(sk^*+s(s-1)\eta) \\
    \frac{2}{t}(t-s)\gamma &\geq \brac{\frac{1}{s}-\frac{2}{t}}(k^*+(s-1)\eta) \\
    \gamma &\geq \frac{1}{t-s}\brac{\frac{t}{2s}-1}(k^*+(s-1)\eta) \\
    \gamma &\geq \frac{t - 2s}{2s(t-s)}(k^*+(s-1)\eta)\ ,
\end{align*}
which completes the proof.
\end{proof}

\subsection{Proof of Proposition \ref{prop:impossible}}\label{app:proofimpossible}
\begin{manualproposition}{\ref{prop:impossible}}
Let $\gamma$ and $\eta$ be set according to~\eqref{esso} and~\eqref{brat}. If $s< (t/2-1)$, then $\gamma >\eta\ .$
\end{manualproposition}
\begin{proof}
We will first rewrite $s = a t$ s.t.~$a \in [0, 1]$ is a fraction of the reference dataset size $t$. Substituting $s = a t$ into $\gamma$~\eqref{esso} and $\eta$~\eqref{brat} yields
\begin{align*}
    \gamma &= \frac{t - 2a t}{2a t(t-a t)}(k^*+(a t-1)\eta) \\
    &= \frac{1 - 2a}{2a t(1-a)}(k^*+(a t-1)\eta)\ ,\\
    \eta &= \frac{t}{(a t+1)(a t(t-2)+t)}k^* \\
    &= \frac{1}{(a t + 1)(a t - 2a + 1)}k^* \ .
\end{align*}

Then,
\begin{align}
    \eta - \gamma &= \eta - \frac{1 - 2a}{2a t(1-a)}(k^*+(a t-1)\eta) \nonumber\\
    &= \brac{1 - \frac{(1-2a)(a t-1)}{2 a t (1-a)}}\eta - \frac{1-2a}{2a t(1-a)}k^* \nonumber\\
    &= \brac{\brac{1 - \frac{(1-2a)(a t-1)}{2 a t (1-a)}}\brac{\frac{1}{(a t + 1)(a t - 2a + 1)}} - \frac{1-2a}{2a t(1-a)}}k^*\ .\label{whacky}
\end{align}

We need the following two lemmas before we can proceed with the main proof:

\begin{lemma}
If $\displaystyle a = \frac{1}{2} - \frac{1}{t}$, then $\gamma = \eta$.
\label{fb}
\end{lemma}

\begin{proof}
It suffices to show that the multiplicative factor of $k^*$ in~\eqref{whacky} is $0$. Substituting $a = {1}/{2} - {1}/{t}$ into the multiplicative factor of $k^*$ in~\eqref{whacky} results in
\begin{align*}
    &\brac{1 - \frac{(1-2a)(a t-1)}{2 a t (1-a)}}\brac{\frac{1}{(a t + 1)(a t - 2a + 1)}} - \frac{1-2a}{2a t(1-a)} \\ &= \brac{1 - \frac{\frac{2}{t}(\frac{t}{2}-2)}{(t-2)(\frac{1}{2}+\frac{1}{t})}}\brac{\frac{1}{\frac{t}{2}(\frac{t}{2}-1+\frac{2}{t})}} - \frac{\frac{2}{t}}{t(1-\frac{2}{t})(\frac{1}{2}+\frac{1}{t})} \\
    &= \brac{\frac{(t-2)(\frac{1}{2}+\frac{1}{t})-\frac{2}{t}(\frac{t}{2}-2)}{(t-2)(\frac{1}{2}+\frac{1}{t})}}\brac{\frac{2t}{\frac{t}{2}(t^2 -2t + 4)}} - \frac{2}{(t-2)(\frac{t}{2}+1)} \\
    &= \brac{\frac{(t-2)(t+2)-2(t-4)}{(t-2)(t+2)}}\brac{\frac{4}{t^2-2t+4}} - \frac{4}{(t-2)(t+2)} \\
    &= \brac{\frac{t^2- 2t +4}{t^2-4}}\brac{\frac{4}{t^2-2t+4}} - \frac{4}{t^2 -4} \\
    &= 0 \ ,
\end{align*}
which proves the lemma.
\end{proof}

\begin{lemma}
$\displaystyle\frac{\emph{d}}{\emph{d}a} (\eta - \gamma)>0$ for $a \in [0, 1]$.
\label{tebow}
\end{lemma}
\begin{proof}
The derivative of $\eta - \gamma$~\eqref{whacky} w.r.t.~$a$ is given by
\begin{equation}
    \frac{\text{d}}{\text{d}a} (\eta - \gamma) = \frac{2t^2a^2 + (4t-2t^2)a + t^2 -t+2}{2t(a-1)^2(ta + 1)^2}k^*.
\end{equation}
Since $k^*$ is a positive constant due to positive definite $k$ and the denominator $2t(a-1)^2(ta + 1)^2$ is always positive, it suffices to show that the numerator $2t^2a^2 + (4t-2t^2)a + t^2 -t+2$ is strictly positive for ${\text{d}}(\eta - \gamma)/\text{d}a > 0$. 
Since $t$ is always positive, the numerator is a convex quadratic function of $a$ with a minimum when its derivative is $0$.
The derivative of the numerator is
\begin{equation*}
    \frac{d}{da} (2t^2a^2 + (4t-2t^2)a + t^2 -t+2) = 4t^2a + (4t-2t^2)\ .
\end{equation*}
By setting it to $0$,
\begin{align*}
    4t^2a + (4t-2t^2) &= 0 \\
    4t^2a &= 2t^2 - 4t \\
    a &= \frac{t-2}{2t}\ .
\end{align*}
By substituting $a= {t-2}/(2t)$ into the numerator,
\begin{align*} 
    & 2t^2\brac{\frac{t-2}{2t}}^2 + (4t-2t^2)\brac{\frac{t-2}{2t}} + t^2 -t+2 \\
    &= 2t^2\brac{\frac{(t-2)^2}{4t^2}}+(2-t)(t-2)+t^2 -t +2 \\
    &= \frac{1}{2}(t-2)^2-(t-2)^2+t^2-t+2 \\
    &= -\frac{1}{2}t^2 +2t -2 + t^2 - t + 2 \\
    &= \frac{1}{2}t^2 + t
\end{align*}
which is always positive and hence proves the lemma.
\end{proof}

Lemmas~\ref{fb} and~\ref{tebow} together imply that $\eta - \gamma < 0$ when $\displaystyle a < \frac{1}{2} - \frac{1}{t}$ (i.e., when $\displaystyle s = at < \frac{t}{2} - 1$), hence implying $\gamma > \eta$ which completes the proof.
\end{proof}

\section{Theorem and Properties F1 to F4 defining (R5) Fairness Incentive in~\cite{simcollaborative}}
\label{app:rewardincentives}
We have reproduced below the main result from~\cite{simcollaborative} as well as the formal definitions of properties F1 to F4 defining the (R5) fairness incentive in~\cite{simcollaborative} based on our notations.

\begin{theorem}[\bf\cite{simcollaborative}]
 Let $0 \leq \rho \leq 1$, $\phi^* := \max_{i\in N}{\phi_i}$, and $C_i := \{j\in N | \phi_j \leq \phi_i\}$ and $r_i := \brac{{\phi_i}/{\phi^*}}^\rho \times v_c(N)$ for each party $i\in N$. Then, $(r_1,\ldots,r_n)$ satisfy (a) R1 to R3 and R5 if $\rho > 0$, (b) R1 to R5 if $\rho\leq \min_{i \in N}{{\log(v_{c}(\{i\})/v_c(N))}/{\log(\phi_i/\phi^*)}}$, (c) R1 to R6 if $\rho\leq \min_{i \in N}{{\log(v_c(C_i)/v_c(N))}/{\log(\phi_i/\phi^*)}}$, and (d) R7 but not R5 if $\rho = 0$. 
\label{thm:sim}
\end{theorem}

(\textbf{R5})\ \textbf{Fairness.} The reward values $(r_1,\ldots,r_n)$ must satisfy F1 to F4 defined below:

(\textbf{F1})\ \textbf{Uselessness.} If the marginal contribution of party $i$ is zero for any coalition (e.g., when $D_i=\emptyset$),
then party $i$ should be assigned a reward value of $0$:
\begin{equation*}
    \forall i \in N\ \ [\forall C \subseteq N \setminus \{i\}\ \ v_c(C \cup \{i\}) = v_c(C)] \Rightarrow r_i = 0 \ .
\end{equation*}

(\textbf{F2})\ \textbf{Symmetry.} 
If the marginal contributions of parties $i$ and $j$ are the same for any coalition (e.g., when $D_i = D_j$),
then they should be assigned the same reward value:
\begin{equation*}
     \forall i,j \in N\ \text{s.t.}\ i \neq j\ \ [\forall C \subseteq N \setminus \{i, j\}\ \  v_c(C \cup \{i\}) = v_c(C \cup \{j\})] \Rightarrow r_i = r_j \ .
\end{equation*}

(\textbf{F3})\ \textbf{Strict Desirability.} If 
the marginal contribution of party $i$ is more than that of party $j$ for at least a coalition,
but the reverse is not true,
then party $i$ should receive a larger reward value than $j$:
\begin{equation*}
\hspace{-1.7mm}
\begin{array}{l}
\displaystyle\forall i,j \in N\ \text{s.t.}\ i \neq j\ \ [\exists C \subseteq N \setminus \{i,j\}\ \ v_c(C\cup \{i\}) > v_c(C\cup \{j\})]\ \logicaland \vspace{1mm}\\ 
\displaystyle\qquad\qquad\qquad\qquad\ [\forall B \subseteq N \setminus \{i,j\}\ \ v_c(B\cup \{i\}) \geq v_c(B\cup \{j\})] \Rightarrow r_i > r_j \ .
\end{array}
\end{equation*}

(\textbf{F4})\ \textbf{Strict Monotonicity.} Let $v_c$ and $v'_c$ denote any two characteristic functions for data valuation with the same domain $2^N$, and $r_i$ and $r'_i$ be the corresponding reward values assigned to party $i$.
If the marginal contribution of party $i$ is larger under $v'_c$ than $v_c$ (e.g., by including a larger dataset) for at least a coalition, \emph{ceteris paribus}, then party $i$ should be assigned a larger reward value under $v'_c$ than $v_c$:
\begin{equation*}
\hspace{-1.7mm}
\begin{array}{l}
\displaystyle\forall i \in N \ [\exists C \subseteq N \setminus \{i\}\ v'_c(C\cup\{i\}) > v_c(C\cup\{i\})] \logicaland [\forall B \subseteq N \setminus \{i\}\ v'_c(B\cup\{i\}) \geq v_c(B\cup\{i\})] \vspace{1mm}\\
\qquad\quad\ \displaystyle\logicaland\ [\forall A \subseteq N\setminus\{i\}\ v'_c(A) = v_c(A)] \logicaland (v'_c(N) > r_i)\Rightarrow r'_i > r_i \ .
\end{array}    
\end{equation*}

\section{Comparison with \cite{simcollaborative}}\label{app:comparison}
In this section, we consolidate and highlight 3 main contributions of our work that set us apart from that of \cite{simcollaborative}. 

Firstly, our problem setting is novel and completely different from that of \cite{simcollaborative}. They propose a collaborative process in which parties pool their datasets via a trusted central party who then distributes trained Bayesian supervised learning models as rewards commensurate with the parties' contributions. Our CGM framework pools data via a trusted central party, trains a generative model (e.g., a GAN) on this pooled dataset, and then distributes synthetically generated data as rewards. Consequently, our CGM framework offers a number of advantages over that of \cite{simcollaborative}: Distributing trained models as rewards limits each party's flexibility to experiment with different model architectures and hyperparameters. If more competitive model architectures emerge in the future, the parties cannot take advantage of these new architectures without reinitiating the collaboration. Another limitation of distributing trained models as rewards is that it precludes the possibility of performing a different learning task on the same dataset as the model is tied to a specific task. Our CGM framework does not suffer from these limitations and gives flexibility to the parties since there is no assumption on whether all parties share a common downstream learning task, the task of interest to each party (e.g., supervised or unsupervised, classification or regression), or the type of model used by each party. In particular, with the synthetic data reward, each party can now optimize over model architectures and hyperparameters, train new model architectures emerging in the future, and train separate models for different learning tasks. It is clear that these model- and task-agnostic benefits for downstream learning tasks necessitate several different considerations and solution concepts, as discussed below.

Secondly, the data valuation function is different from that of \cite{simcollaborative}. We propose to exploit the biased MMD for our data valuation function and provide new theoretical perspectives on the suitability of using biased MMD for data valuation. In contrast, \cite{simcollaborative} use the mutual information between the parameters of a Bayesian supervised learning model and the dataset being valued. Their valuation function relies on specifying a prior for the model parameters and observing the reduction in entropy upon training the model with that dataset. In addition, we have a significant section on kernel selection for use with the MMD and theoretical results guiding the selection process; this section has no relation to the work of \cite{simcollaborative}.

Thirdly, the incentive conditions from \cite{simcollaborative} do not easily generalize to fit our CGM framework and require appropriate modifications to be used in our framework (Sec.~\ref{lalaland}). In our work, the maximum reward value is not as straightforward to obtain as that in \cite{simcollaborative} in which it simply corresponds to the best possible trained model. In our CGM framework, there are upper bounds on the values of every party's reward and choosing the maximum reward value requires our formulation of the linear program in Sec.~\ref{lalaland}. This allows us to pick the globally optimal combination of maximum reward value $v^*$ and $\rho$ given the constraints: Choosing the largest possible $v^*$ ensures maximum group welfare, while choosing $\rho$ to be close to $1$ makes the assigned reward values closer in proportion to the normalized Shapley values or to purely Shapley fair.

\section{Increasing the Size of $G$ for LP Feasibility}\label{app:increaseG}
When the synthetic dataset $G$ is too small, the value of synthetic data may be insufficient for satisfying incentives R1 to R6. Concretely, it may not hold that the upper bound $v(D_i\cup G)$ on any party $i$'s assigned reward value is at least its lower bound $v_c(C_i)$ (where $C_i := \{j\in N | \phi_j \leq \phi_i\}$), a condition of which is required to satisfy R1 to R6 (Proposition~\ref{prop:cgmincentives}c). So, no reward value is feasible for some party.

In the last sentence of Sec.~\ref{lalaland}, we claim that if this condition is not yet satisfied, we may increase the size of $G$ by simply sampling more synthetic data points from the generative model until $v(D_i\cup G)$ is at least the requisite lower bound for every party $i\in\mathcal{N}$. We provide an intuition for why increasing the size of $G$ resolves this issue based on the reasonable assumptions that when $G$ is increased in size, for all $i$, (a) $\text{MMD}^2_u(\mathcal F, D_i\cup G, T)$ weakly decreases, and (b) $\text{MMD}^2_u(\mathcal F, \bigcup_{j\in C_i}D_j, T)$ weakly increases. The first is reasonable since both $D_i \cup G$ and $T$ now contain more synthetic data points and their associated distributions are expected to be closer, while the second is reasonable since $T$ now contains more synthetic data points and 
its associated distribution may not be closer to that of $\bigcup_{j\in C_i}D_j$.
Under these assumptions, when $G$ increases in size, using Corollary~\ref{coro:valuefunc}, the bounds of $v(D_i\cup G) - v_c(C_i)$ for all $i$ will increase, as desired, since $|D_i\cup G|$ increases and $|\bigcup_{j\in C_i}D_j|$ stays the same. Empirically, we observed that increasing the size of $G$ resolves the issue when it arises.

\section{Weighted Sampling Algorithm}\label{app:alg}
Algorithm~\ref{alg:rewardrealization} gives the pseudocode of our weighted sampling algorithm described in Sec.~\ref{force}.

\begin{algorithm}[h]
  \caption{CGM Weighted Sampling}
  \label{alg:rewardrealization}
\begin{algorithmic}[1]
    \STATE {\bfseries Input:} Distribution $\mathcal G$ produced by generative model, reward vector $(r_1, \ldots,r_n)$, parties' datasets $D_1,\ldots,D_n$, inverse temperature hyperparameter $\beta$
    \STATE Initialize synthetic data reward
    $G_i = \emptyset$ to each party $i\in N$
    \STATE Draw synthetic dataset $G$ from $\mathcal G$
    \FOR{{\bfseries each} party $i\in N$}
    \STATE Create a replicate $G'$ of synthetic dataset $G$
    \STATE $\mu_i \leftarrow v(D_i)$
    \WHILE{$\mu_i < r_i$}
    \STATE $\Delta_{x} \leftarrow v(D_i \cup G_i \cup \{x\})\ - \mu_i$ for all $x \in G'$
    \STATE $\bar \Delta_{x} \leftarrow \textsc{Normalize}(\Delta_{x})$ for all $x \in G'$
    \STATE $p(x) \leftarrow {\exp{\brac{\beta \bar \Delta_x}}}/{\sum_{x'\in G'} \exp{\brac{\beta \bar \Delta_{x'}}}}$ for all $x \in G'$
    \STATE $x \leftarrow \textsc{Sample}(G', p)$
    \STATE $G_i \leftarrow G_i \cup \{x\}$
    \STATE $\mu_i \leftarrow \mu_i + \Delta_x$
    \STATE $G' \leftarrow G' \setminus \{x\}$
    \ENDWHILE
    \ENDFOR
    \RETURN synthetic data rewards $G_1,\ldots,G_n$
\end{algorithmic}
\end{algorithm}

\section{Computational Details}\label{app:computation}
In this section, we analyze the overall time complexity of our end-to-end reward scheme. This includes the calculation of the reward values (Sec.~\ref{lalaland}) and the weighted sampling algorithm for the distribution of synthetic data rewards (Sec.~\ref{force}).

The first source of computational load comes from computing the Shapley value of each party. Computing the Shapley value exactly incurs $\mathcal O(n!)$ time, which is intractable if the number of parties $n$ is large. To circumvent this issue of large $n$, we can follow the recent work of \cite{Ghorbani2019} by considering the Shapley value to be the expected marginal contribution based on a uniform distribution over permutations of parties and estimating it via Monte Carlo sampling which incurs $\mathcal O(mn)$ time where $m$ is the number of permutations sampled. The value of $m$ required to achieve a small estimation error with high probability can be bounded when either the variance or the range of the marginal contributions is known \cite{MalekiTHRR13}. After we have the Shapley values, we compute the reward values using the LP. The LP has only $2$ variables and $2n + 2$ constraints, which can be solved in $\mathcal O(n^{1.5})$ time using interior point methods \cite{vaidya1989speeding}. 

We now discuss the computational load arising from the weighted sampling algorithm (Algorithm \ref{alg:rewardrealization}). Computing $v$ using~\eqref{eq:valuefunc} incurs $\mathcal{O}(s(s+t))$
time. This time complexity is quadratic in $s$, which makes directly using \eqref{eq:valuefunc} for Algorithm \ref{alg:rewardrealization} intractable. Instead of naively recomputing $v$ for every synthetic data point $x$, the time needed to compute $\Delta_{x}$ can be reduced by performing a sequential update of $v$. Define the terms
\begin{gather}
    \sigma_S := \frac{1}{s^2}\inn{\boldsymbol 1_{[x, x'\in S]}}{\mbf K} \\
    \sigma_{ST} := \frac{2}{st}\inn{\boldsymbol 1_{[x \in S, x'\in T]}}{\mbf K}
\end{gather}
such that $v(S) = \sigma_{ST} - \sigma_{S}$. Consider how these quantities change by adding another point $x^* \in T \setminus S$ to $S$. Define these updated quantities
\begin{gather}
    \sigma_S^* := \frac{1}{s^2}\inn{\boldsymbol 1_{[x, x'\in S \cup \{x^*\}]}}{\mbf K} \\
    \sigma_{ST}^* := \frac{2}{st}\inn{\boldsymbol 1_{[x \in S \cup \{x^*\}, x'\in T]}}{\mbf K}
\end{gather}
such that $v(S \cup \{x^*\}) = \sigma^*_{ST} - \sigma^*_{S}$. It can be shown that $\sigma_S^*$ and $\sigma_{ST}^*$ can be computed from $\sigma_S$ and $\sigma_{ST}$ with the following:
\begin{gather}
    \sigma_S^* = \frac{s^2}{(s+1)^2}\sigma_S + \frac{2}{(s+1)^2}\sum_{x \in S}k(x, x^*) + \frac{1}{(s+1)^2}k(x^*, x^*)\\
    \sigma_{ST}^* = \frac{s}{s+1}\sigma_{ST} + \frac{2}{t(s+1)}\sum_{x \in T}k(x, x^*) \ .
\end{gather}
By storing the values of $\sigma_S$ and $\sigma_{ST}$ at every iteration where $S=D_i \cup G_i$ (i.e., $s=|D_i \cup G_i|$),  $\Delta_{x}$ can be recomputed for each synthetic data point $x$ in $\mathcal{O}(s+t)$ time. This time complexity is linear in $s$, a significant improvement over the quadratic time incurred by computing $v(S \cup \{x^*\})$ naively. This update has to be performed for every synthetic data point in $G$ to obtain $\Delta_{x}$. Such an update is performed at most $|G|$ times which is the largest number of synthetic data points that can be distributed to each party as reward. Finally,  computing $k(x, x')$ for isotropic kernels (see \citesupp{app:kernel} for a discussion on isotropic kernels) is linear in the dimensionality of the data $d$. The weighted sampling algorithm thus incurs $\mathcal O(n|G|^2(s+t)d)$ time.

\section{Additional Experimental Details and Results}\label{app:expdetails}
\subsection{Simulated credit ratings (CR)}
The CR dataset is a uniform mixture of $5$ $2$-D Gaussians where the Gaussians have means $(0.435, 0.0259)$, $(0.550, 0.435)$, $(0.420, 0.330)$, $(0.205, 0.619)$, and $(0.300, 0.267)$. All Gaussians have a covariance matrix of $({1}/{200})\mbf I$. Synthetic data are trivially obtained by sampling more data points from the mixture model. Each party $i$ has a dataset $D_i$ of size $1$K, and $|G| = 100$K synthetic data points are generated.

\subsection{Credit card fraud dataset (CC)}
The CC (DbCL 1.0) dataset is reduced to a $4$-D dataset by selecting the first $4$ principal components.
 
The $5$ classes are obtained by separating the dataset into $5$ percentiles for the `Amount' variable (i.e., $0$-$20$, $20$-$40$, $40$-$60$, $60$-$80$, and $80$-$100$). Synthetic data are obtained by using kernel density estimation with a Gaussian kernel and a bandwidth of $0.05$ to fit a probability distribution to the dataset and then sampling from this distribution. Each party $i$ has a dataset $D_i$ of size $5$K, and $|G|=100$K synthetic data points are generated.

\subsection{MNIST}
The MNIST (CC BY-SA 3.0) images are reduced to an $8$-D embedding using a non-linear dimensionality reduction algorithm called UMAP~\cite{mcinnes2018umap} with default settings. Synthetic data are obtained by sampling from a pre-trained MMD GAN~\cite{binkowski2018demystifying} (BSD-$3$-Clause). Samples are automatically labeled  by training a classifier obtained from \url{https://github.com/pytorch/examples/tree/master/mnist} (BSD-$3$-Clause) and taking its predictions on the samples. Each party $i$ has a dataset $D_i$ of size $5$K, and $|G|=100$K synthetic data points are generated.

\subsection{CIFAR-10}
 The CIFAR-10 images are first transformed by passing them through a pre-trained Inception network and extracting the vision-relevant features at the final pool3 layer (as is done when computing the Fr\'echet Inception distance (FID) for evaluating GANs \cite{heusel2017gans}) using code from~\cite{Seitzer2020FID} (Apache License 2.0). These features are then also reduced to an $8$-D embedding using UMAP with default settings as above. As in MNIST, synthetic data are obtained by sampling from a pre-trained MMD GAN~\cite{binkowski2018demystifying}. Samples are automatically labeled  by training the ResNet-18 classifier obtained from \url{https://github.com/kuangliu/pytorch-cifar} (MIT License) and taking its predictions on the samples. Each party $i$ has a dataset $D_i$ of size $5$K, and $|G|=100$K synthetic data points are generated.

\subsection{Optimization details}
For the linear program in Sec.~\ref{sect:linprog}, we set the weight $\epsilon$ as $0.001$. For the binary search algorithm that is used to find the minimum length-scale of the squared exponential kernel (Sec.~\ref{sect:kernelselection}), we set the initial length-scale to $1$ and the number of binary search iterations to $20$.

\subsection{Data splits}
Tables~\ref{jjlim},~\ref{wahcoy},~\ref{wahpiang}, and~\ref{wahlau} detail the proportion of data in each class for each party's dataset  in each data split.
\begin{table}[h]
\centering
\caption{CR/CC dataset with equal disjoint split.}
\begin{tabular}{@{}cllllll@{}}
\toprule
                       &   & \multicolumn{5}{c}{Class label}  \\ \midrule
                       &   & 0    & 1    & 2    & 3    & 4    \\
\multirow{5}{*}{Party} & 1 & 0.96 & 0.01 & 0.01 & 0.01 & 0.01 \\
                       & 2 & 0.01 & 0.96 & 0.01 & 0.01 & 0.01 \\
                       & 3 & 0.01 & 0.01 & 0.96 & 0.01 & 0.01 \\
                       & 4 & 0.01 & 0.01 & 0.01 & 0.96 & 0.01 \\
                       & 5 & 0.01 & 0.01 & 0.01 & 0.01 & 0.96 \\ \cmidrule(l){2-7} 
\end{tabular}
\label{jjlim}
\end{table}

\begin{table}[h]
\centering
\caption{CR/CC dataset with unequal split.}
\begin{tabular}{@{}cllllll@{}}
\toprule
                       &   & \multicolumn{5}{c}{Class label}  \\ \midrule
                       &   & 0    & 1    & 2    & 3    & 4    \\
\multirow{5}{*}{Party} & 1 & 0.20 & 0.20 & 0.20 & 0.20 & 0.20 \\
                       & 2 & 0.20 & 0.20 & 0.20 & 0.20 & 0.20 \\
                       & 3 & 0.58 & 0.39 & 0.01 & 0.01 & 0.01 \\
                       & 4 & 0.01 & 0.20 & 0.58 & 0.20 & 0.01 \\
                       & 5 & 0.01 & 0.01 & 0.01 & 0.39 & 0.58 \\ \cmidrule(l){2-7} 
\end{tabular}
\label{wahcoy}
\end{table}

\begin{table}[H]
\centering
\caption{MNIST/CIFAR-10 dataset with equal disjoint split.}
\begin{tabular}{@{}clllllllllll@{}}
\toprule
                       &   & \multicolumn{10}{c}{Class label}                                              \\ \midrule
                       &   & 0     & 1     & 2     & 3     & 4     & 5     & 6     & 7     & 8     & 9     \\
\multirow{5}{*}{Party} & 1 & 0.480 & 0.480 & 0.005 & 0.005 & 0.005 & 0.005 & 0.005 & 0.005 & 0.005 & 0.005 \\
                       & 2 & 0.005 & 0.005 & 0.480 & 0.480 & 0.005 & 0.005 & 0.005 & 0.005 & 0.005 & 0.005 \\
                       & 3 & 0.005 & 0.005 & 0.005 & 0.005 & 0.480 & 0.480 & 0.005 & 0.005 & 0.005 & 0.005 \\
                       & 4 & 0.005 & 0.005 & 0.005 & 0.005 & 0.005 & 0.005 & 0.480 & 0.480 & 0.005 & 0.005 \\
                       & 5 & 0.005 & 0.005 & 0.005 & 0.005 & 0.005 & 0.005 & 0.005 & 0.005 & 0.480 & 0.480 \\ \cmidrule(l){2-12} 
\end{tabular}
\label{wahpiang}
\end{table}

\begin{table}[h]
\centering
\caption{MNIST/CIFAR-10 dataset with unequal split.}
\begin{tabular}{@{}clllllllllll@{}}
\toprule
                       &   & \multicolumn{10}{c}{Class label}                                              \\ \midrule
                       &   & 0     & 1     & 2     & 3     & 4     & 5     & 6     & 7     & 8     & 9     \\
\multirow{5}{*}{Party} & 1 & 0.100 & 0.100 & 0.100 & 0.100 & 0.100 & 0.100 & 0.100 & 0.100 & 0.100 & 0.100 \\
                       & 2 & 0.100 & 0.100 & 0.100 & 0.100 & 0.100 & 0.100 & 0.100 & 0.100 & 0.100 & 0.100 \\
                       & 3 & 0.290 & 0.290 & 0.195 & 0.195 & 0.005 & 0.005 & 0.005 & 0.005 & 0.005 & 0.005 \\
                       & 4 & 0.005 & 0.005 & 0.100 & 0.100 & 0.290 & 0.290 & 0.100 & 0.100 & 0.005 & 0.005 \\
                       & 5 & 0.005 & 0.005 & 0.005 & 0.005 & 0.005 & 0.005 & 0.195 & 0.195 & 0.290 & 0.290 \\ \cmidrule(l){2-12} 
\end{tabular}
\label{wahlau}
\end{table}

\subsection{Discussion of negative correlations}\label{app:negativecorrelations}
The strongly negative correlation of $|G_i|$ with $\alpha_i$ in the unequal split for MNIST may be explained by the observation that parties with full class distribution (hence high $\alpha_i$) have very small unbiased MMD (i.e., $2$ orders of magnitude lower than parties with restricted class distributions), as compared to the same parties for the other datasets (i.e., $1$ order of magnitude difference). According to Proposition~\ref{prop:valuefunc}, the weighted sampling algorithm thus distributes relatively fewer synthetic data points to these parties, hence giving negative correlation. We hypothesize that MNIST displays this behavior as the $2$-D embeddings (\citesupp{app:addresults}) show that MNIST is cleanly separated into class clusters while the other datasets are not, which means that in an MNIST class, all data points are very similar to each other, hence resulting in a party with full class distribution having a very similar dataset as reference dataset $T$.

The slightly negative correlation between $-KL$ and $\alpha_i$ in the equal disjoint split for MNIST may be explained by the specific interaction between the data distribution of the dimensionality-reduced MNIST dataset, the equal disjoint split of data among the parties, and the KL divergence finite sample estimator based on $k$-th nearest neighbor distance (Perez-Cruz, 2008). As mentioned above, the classes in the MNIST dataset are much more cleanly separated. This suggests that the data points in a class are close to each other in the Euclidean space. For equal disjoint split in Fig. 2a, every party has the majority of its data in only two classes. So, its data points in dataset $S$ become even closer to each other in the Euclidean space and each point in $S$ is very likely to have its $k$-th nearest neighbor (in the same dataset $S$ with a distance denoted by $d_S$) belong to the same class.Each point in $S$ is very likely to have its $k$-th nearest neighbor (in the reference dataset $T$ with a distance denoted by $d_T$) belong to the same class as well. For each point in $S$, $d_S$ and $d_T$ are thus very likely to be approximately equal since for MNIST, the data points in a class are very close to each other. Since the KL divergence estimator compares $S$ vs. $T$ based on the distances of each point in $S$ to its $k$-th nearest neighbor in $S$ vs. in $T$, the estimator is likely to report that $S$ and $T$ are very similar for small values of $k$, even when they are not since $T$ contains data points in all other classes as well. The estimator is thus a poor one under the equal disjoint split and the pathological data distribution of the MNIST dataset, the latter of which also causes the negative correlation  observed in Fig. 2b. Such an issue does not surface with smoother data distributions such as the other 3 datasets in our experiments.
\subsection{Compute resources and software environments}
The experiments are run on two machines:
\begin{enumerate}
    \item Ubuntu $18$.$04$.$5$ LTS, Intel(R) Xeon(R) CPU E$5$-$2683$ v$4$ @ $2$.$10$GHz, NVIDIA GeForce GTX $1080$ Ti (CUDA $11$.$0$), and
    \item Ubuntu $20$.$04$.$2$ LTS, Intel(R) Xeon(R) Silver $4108$ CPU @ $1$.$80$GHz, NVIDIA TITAN V (CUDA $11$.$0$).
\end{enumerate}
The software environments used are Anaconda and Python. Refer to the \verb|environment.yml| file in the repository for the full list of Python packages used.

\subsection{Additional experimental results}
\label{app:addresults}
\textbf{Assessing contributions of parties.} Tables~\ref{jarvis} and~\ref{junior} show the normalized Shapley values $\alpha_1, \ldots, \alpha_5$  of the corresponding parties $1,\ldots,5$ for different datasets and splits. It can be observed that in equal disjoint split (Table~\ref{jarvis}), the $\alpha_i$'s are relatively more evenly spread compared to that of the unequal split where the $\alpha_i$'s of the first two parties are very large (as expected) and that of the remaining parties are relatively lower.

\begin{table}[H]
\centering
\caption{Normalized Shapley value $\alpha_i$ of party $i$ for different datasets with equal disjoint split.}
\begin{tabular}{@{}llllll@{}}
\toprule
         & \multicolumn{5}{c}{Party}             \\ \midrule
         & 1     & 2     & 3     & 4     & 5     \\
CR      & 0.417 & 0.710 & 1.0   & 0.320 & 0.881 \\
CC       & 0.252 & 0.962 & 1.0 & 0.888 & 0.513 \\
MNIST    & 0.530 & 0.569 & 0.946 & 0.853 & 1.0   \\
CIFAR-10 & 0.379 & 1.0   & 0.683 & 0.907 & 0.117 \\ \bottomrule
\end{tabular}
\label{jarvis}
\end{table}

\begin{table}[H]
\centering
\caption{Normalized Shapley value $\alpha_i$ of party $i$ for different datasets with unequal split.}
\begin{tabular}{@{}llllll@{}}
\toprule
         & \multicolumn{5}{c}{Party}             \\ \midrule
         & 1     & 2     & 3     & 4     & 5     \\
CR      & 0.996 & 1.0   & 0.233 & 0.647 & 0.321 \\
CC       & 1.0   & 0.956 & 0.153 & 0.824 & 0.583 \\
MNIST    & 1.0   & 0.998 & 0.256 & 0.517 & 0.368 \\
CIFAR-10 & 1.0   & 0.971 & 0.264 & 0.477 & 0.173 \\ \bottomrule
\end{tabular}
\label{junior}
\end{table}

\textbf{Role of inverse temperature hyperparameter $\beta$.} Tables~\ref{eatprata} and~\ref{hungryhippo} plot the mean and standard errors of the correlation coefficients between the inverse temperature hyperparameter $\beta$ vs.~number of synthetic data points/unbiased MMD that are computed across all $5$ parties for each dataset and split. It can be observed that 
$\beta$ is highly negatively correlated with the number of synthetic data points and
unbiased MMD
(equivalently, highly positively correlated with  negative unbiased MMD), which aligns with our reasoning in Sec.~\ref{force}.
\begin{table}[H]
\centering
\caption{Correlation of inverse temperature hyperparameter $\beta$ with number of synthetic data points and unbiased MMD in equal disjoint split (lower is better).}
\begin{tabular}{@{}lllll@{}}
\toprule
                  & CR               & CC                  & MNIST              & CIFAR-10           \\ \midrule
No. of synthetic data points & -0.851 +/- 0.0233 & -0.906 +/- 6.23e-03 & -0.869 +/- 0.00506 & -0.906 +/- 0.00793 \\
Unbiased MMD      & -0.737 +/- 0.0794 & -0.944 +/- 0.0122   & -0.679 +/- 0.100   & -0.0839 +/- 0.287 
\end{tabular}
\label{eatprata}
\end{table}
\begin{table}[H]
\centering
\caption{Correlation of inverse temperature hyperparameter $\beta$ with number of synthetic data points and unbiased MMD in unequal split (lower is better).}
\begin{tabular}{@{}lllll@{}}
\toprule
                  & CR                & CC                  & MNIST             & CIFAR-10           \\ \midrule
No. of synthetic data points & -0.834 +/- 0.00649 & -0.926 +/- 0.0132   & -0.837 +/- 0.0164 & -0.945 +/- 0.00792 \\
Unbiased MMD      & -0.928 +/- 0.0274  & -0.963 +/- 8.23e-03 & -0.868 +/- 0.0183 & -0.622 +/- 0.171  
\end{tabular}
\label{hungryhippo}
\end{table}

\textbf{Full results.} Tables~\ref{jingle1},~\ref{jingle2},~\ref{jingle3},~\ref{jingle4},~\ref{jingle5},~\ref{jingle6},~\ref{jingle7}, and~\ref{jingle8} show the full experimental results which include $|G_i|$, unbiased MMD, and $v(D_i \cup G_i)$ of each party $i$, and the optimized values of $v^*$, $\rho$ and kernel length-scales.

\begin{table}[H]
\centering
\caption{CR dataset with equal disjoint split ($v^* = 0.536$,  $\rho = 0.0943$, kernel length-scale $= 0.265$).}
\begin{tabular}{@{}llllll@{}}
\toprule
                                          & \multicolumn{5}{c}{Party}                                                                \\ \midrule
                                          & 1               & 2               & 3                & 4               & 5               \\
$|G_i|$ / $\text{MMD}_u$, $\beta = 1$ & 1510 / 4.22e-02 & 1854 / 1.68e-02 & 6123 / -5.87e-05 & 1437 / 5.43e-02 & 2071 / 6.24e-03 \\
$|G_i|$ / $\text{MMD}_u$, $\beta = 2$ & 1239 / 4.22e-02 & 1451 / 1.68e-02 & 4412 / -7.95e-05 & 1253 / 5.43e-02 & 1560 / 6.19e-03 \\
$|G_i|$ / $\text{MMD}_u$, $\beta = 4$ & 1080 / 4.22e-02 & 1189 / 1.68e-02 & 3139 / -1.07e-04 & 1083 / 5.42e-02 & 1189 / 6.18e-03 \\
$|G_i|$ / $\text{MMD}_u$, $\beta = 8$ & 1010 / 4.22e-02 & 1062 / 1.68e-02 & 3314 / -1.01e-04 & 1014 / 5.42e-02 & 1019 / 6.15e-03 \\ \midrule
$r_i = v(D_i \cup G_i)$                         & 0.493        & 0.519        & 0.536        & 0.481        & 0.529        \\ \bottomrule
\end{tabular}
\label{jingle1}
\end{table}

\begin{table}[H]
\centering
\caption{CR dataset with unequal split ($v^* = 0.0898$,  $\rho = 0.00751$, kernel length-scale $= 0.0672$)}
\begin{tabular}{@{}llllll@{}}
\toprule
                                          & \multicolumn{5}{c}{Party}                                                                    \\ \midrule
                                          & 1                 & 2                 & 3               & 4                & 5               \\
$|G_i|$ / $\text{MMD}_u$, $\beta = 1$ & 13484 / -6.86e-05 & 36481 / -3.27e-05 & 2823 / 7.28e-04 & 2344 / 1.22e-05  & 2661 / 5.04e-04 \\
$|G_i|$ / $\text{MMD}_u$, $\beta = 2$ & 6976 / -1.20e-04  & 18314 / -5.54e-05 & 1981 / 6.60e-04 & 1630 / -6.05e-05 & 1783 / 4.24e-04 \\
$|G_i|$ / $\text{MMD}_u$, $\beta = 4$ & 3849 / -1.94e-04  & 10732 / -8.59e-05 & 1556 / 6.11e-04 & 1122 / -1.44e-04 & 1335 / 3.65e-04 \\
$|G_i|$ / $\text{MMD}_u$, $\beta = 8$ & 2162 / -2.94e-04  & 5425 / -1.50e-04  & 1407 / 5.89e-04 & 985 / -1.75e-04  & 1204 / 3.39e-04 \\ \midrule
$r_i = v(D_i \cup G_i)$                         & 0.0898            & 0.0898            & 0.0888          & 0.0895           & 0.089           \\ \bottomrule
\end{tabular}
\label{jingle2}
\end{table}

\begin{table}[H]
\centering
\caption{CC dataset with equal disjoint split ($v^* = 0.00849$,  $\rho  = 0.254$, kernel length-scale $= 0.387$).}
\begin{tabular}{@{}llllll@{}}
\toprule
                                          & \multicolumn{5}{c}{Party}                                                                  \\ \midrule
                                          & 1               & 2                & 3                 & 4               & 5               \\
$|G_i|$ / $\text{MMD}_u$, $\beta = 1$ & 3268 / 2.38e-03 & 10182 / 1.50e-05 & 47824 / -2.13e-05 & 7319 / 1.67e-04 & 4634 / 1.22e-03 \\
$|G_i|$ / $\text{MMD}_u$, $\beta = 2$ & 2810 / 2.38e-03 & 7799 / 2.83e-06  & 36080 / -2.66e-05 & 5278 / 1.52e-04 & 3516 / 1.21e-03 \\
$|G_i|$ / $\text{MMD}_u$, $\beta = 4$ & 2315 / 2.37e-03 & 5459 / -1.45e-05 & 24222 / -3.65e-05 & 3946 / 1.37e-04 & 2580 / 1.19e-03 \\
$|G_i|$ / $\text{MMD}_u$, $\beta = 8$ & 2086 / 2.36e-03 & 4196 / -2.75e-05 & 17303 / -4.70e-05 & 3137 / 1.26e-04 & 2101 / 1.18e-03 \\ \midrule
$r_i = v(D_i \cup G_i)$                         & 0.00598         & 0.00841          & 0.00849           & 0.00824         & 0.00716         \\ \bottomrule
\end{tabular}
\label{jingle3}
\end{table}

\begin{table}[H]
\centering
\caption{CC dataset with unequal split ($v^* = 0.00365$,  $\rho  = 0.186$, kernel length-scale $= 0.274$).}
\begin{tabular}{@{}llllll@{}}
\toprule
                                          & \multicolumn{5}{c}{Party}                                                                   \\ \midrule
                                          & 1                 & 2                 & 3               & 4               & 5               \\
$|G_i|$ / $\text{MMD}_u$, $\beta = 1$ & 66545 / -1.78e-05 & 13855 / -2.67e-05 & 3001 / 9.45e-04 & 7067 / 4.26e-05 & 5295 / 2.47e-04 \\
$|G_i|$ / $\text{MMD}_u$, $\beta = 2$ & 53627 / -2.09e-05 & 10774 / -3.71e-05 & 2749 / 9.41e-04 & 5387 / 2.93e-05 & 3895 / 2.32e-04 \\
$|G_i|$ / $\text{MMD}_u$, $\beta = 4$ & 38954 / -2.65e-05 & 8183 / -4.95e-05  & 2455 / 9.36e-04 & 3999 / 1.45e-05 & 2804 / 2.16e-04 \\
$|G_i|$ / $\text{MMD}_u$, $\beta = 8$ & 31119 / -3.15e-05 & 6365 / -6.16e-05  & 2102 / 9.29e-04 & 3186 / 3.46e-06 & 2128 / 2.04e-04 \\ \midrule
$r_i = v(D_i \cup G_i)$                         & 0.00365           & 0.00362           & 0.00258         & 0.00352         & 0.0033          \\ \bottomrule
\end{tabular}
\label{jingle4}
\end{table}

\begin{table}[H]
\centering
\caption{MNIST dataset with equal disjoint split ($v^* = 0.276$,  $\rho = 0.520$, kernel length-scale $= 3.36$).}
\begin{tabular}{@{}llllll@{}}
\toprule
                                          & \multicolumn{5}{c}{Party}                                                                \\ \midrule
                                          & 1               & 2               & 3               & 4               & 5                \\
$|G_i|$ / $\text{MMD}_u$, $\beta = 1$ & 3511 / 7.79e-02 & 3346 / 7.04e-02 & 8959 / 8.12e-03 & 6663 / 2.22e-02 & 27698 / 2.92e-04 \\
$|G_i|$ / $\text{MMD}_u$, $\beta = 2$ & 3115 / 7.79e-02 & 2797 / 7.04e-02 & 7203 / 8.11e-03 & 5499 / 2.22e-02 & 19614 / 2.84e-04 \\
$|G_i|$ / $\text{MMD}_u$, $\beta = 4$ & 2729 / 7.79e-02 & 2406 / 7.04e-02 & 5805 / 8.09e-03 & 4773 / 2.22e-02 & 14871 / 2.83e-04 \\
$|G_i|$ / $\text{MMD}_u$, $\beta = 8$ & 2600 / 7.79e-02 & 2251 / 7.04e-02 & 5281 / 8.09e-03 & 4397 / 2.22e-02 & 12803 / 2.83e-04 \\ \midrule
$r_i = v(D_i \cup G_i)$                         & 0.199           & 0.206           & 0.269           & 0.255           & 0.276            \\ \bottomrule
\end{tabular}
\label{jingle5}
\end{table}

\begin{table}[H]
\centering
\caption{MNIST dataset with unequal split ($v^* = 0.0764$,  $\rho = 0.00988$, kernel length-scale $= 1.53$)}
\begin{tabular}{@{}llllll@{}}
\toprule
                                          & \multicolumn{5}{c}{Party}                                                                    \\ \midrule
                                          & 1                & 2                & 3                & 4                & 5                \\
$|G_i|$ / $\text{MMD}_u$, $\beta = 1$ & 4624 / -8.70e-05 & 3723 / -9.44e-05 & 13384 / 9.81e-04 & 13852 / 4.57e-04 & 12985 / 7.09e-04 \\
$|G_i|$ / $\text{MMD}_u$, $\beta = 2$ & 2492 / -1.14e-04 & 2512 / -1.11e-04 & 9301 / 9.65e-04  & 9569 / 4.39e-04  & 9509 / 6.96e-04  \\
$|G_i|$ / $\text{MMD}_u$, $\beta = 4$ & 1661 / -1.28e-04 & 1542 / -1.30e-04 & 6784 / 9.55e-04  & 7330 / 4.32e-04  & 7030 / 6.86e-04  \\
$|G_i|$ / $\text{MMD}_u$, $\beta = 8$ & 1365 / -1.35e-04 & 1063 / -1.41e-04 & 6034 / 9.50e-04  & 6708 / 4.28e-04  & 6332 / 6.82e-04  \\ \midrule
$r_i = v(D_i \cup G_i)$                         & 0.0764           & 0.0764           & 0.0754           & 0.0759           & 0.0757           \\ \bottomrule
\end{tabular}
\label{jingle6}
\end{table}

\begin{table}[H]
\centering
\caption{CIFAR-10 dataset with equal disjoint split ($v^* = 0.500$,  $\rho  = 0.894$, kernel length-scale $= 2.88$).}
\begin{tabular}{@{}llllll@{}}
\toprule
                                          & \multicolumn{5}{c}{Party}                                                              \\ \midrule
                                          & 1              & 2                & 3               & 4               & 5              \\
$|G_i|$ / $\text{MMD}_u$, $\beta = 1$ & 539 / 2.93e-01 & 15490 / 2.91e-03 & 1348 / 1.48e-01 & 3425 / 4.48e-02 & 339 / 4.30e-01 \\
$|G_i|$ / $\text{MMD}_u$, $\beta = 2$ & 484 / 2.93e-01 & 13367 / 2.92e-03 & 1197 / 1.47e-01 & 2824 / 4.48e-02 & 319 / 4.29e-01 \\
$|G_i|$ / $\text{MMD}_u$, $\beta = 4$ & 447 / 2.93e-01 & 11754 / 2.91e-03 & 1051 / 1.48e-01 & 2335 / 4.48e-02 & 279 / 4.29e-01 \\
$|G_i|$ / $\text{MMD}_u$, $\beta = 8$ & 421 / 2.93e-01 & 10647 / 2.90e-03 & 945 / 1.47e-01  & 2105 / 4.48e-02 & 271 / 4.30e-01 \\ \midrule
$r_i = v(D_i \cup G_i)$                         & 0.210          & 0.500            & 0.356           & 0.458           & 0.0735         \\ \bottomrule
\end{tabular}
\label{jingle7}
\end{table}

\begin{table}[H]
\centering
\caption{CIFAR-10 dataset with unequal split ($v^* = 0.133$,  $\rho  = 0.835$, kernel length-scale $= 1.24$).}
\begin{tabular}{@{}llllll@{}}
\toprule
                                          & \multicolumn{5}{c}{Party}                                                             \\ \midrule
                                          & 1               & 2               & 3              & 4               & 5              \\
$|G_i|$ / $\text{MMD}_u$, $\beta = 1$ & 7638 / 2.23e-03 & 5411 / 5.48e-03 & 745 / 9.12e-02 & 1676 / 6.34e-02 & 534 / 1.04e-01 \\
$|G_i|$ / $\text{MMD}_u$, $\beta = 2$ & 6919 / 2.23e-03 & 4901 / 5.48e-03 & 694 / 9.12e-02 & 1555 / 6.33e-02 & 508 / 1.04e-01 \\
$|G_i|$ / $\text{MMD}_u$, $\beta = 4$ & 6229 / 2.23e-03 & 4407 / 5.47e-03 & 647 / 9.12e-02 & 1431 / 6.33e-02 & 479 / 1.04e-01 \\
$|G_i|$ / $\text{MMD}_u$, $\beta = 8$ & 5729 / 2.22e-03 & 3994 / 5.47e-03 & 606 / 9.12e-02 & 1352 / 6.33e-02 & 445 / 1.04e-01 \\ \midrule
$r_i = v(D_i \cup G_i)$                         & 0.133           & 0.129           & 0.0436         & 0.0715          & 0.0307         \\ \bottomrule
\end{tabular}
\label{jingle8}
\end{table}

\textbf{Visualizations of synthetic data rewards.} The plots below visualize the synthetic data points $G_1,\ldots,G_5$ (i.e., in $2$-D embedding using UMAP~\cite{mcinnes2018umap}for CC, MNIST, and CIFAR-10) as rewards to the corresponding parties $1,\ldots,5$ over varying inverse temperature hyperparameters $\beta=1,2,4,8$ for different datasets and splits. 
In these plots, the grey dots denote the entire synthetic dataset $G$, the blue dots denote party $i$'s original dataset $D_i$, and the red dots denote the synthetic data points $G_i$ as reward to  party $i$. The more opaque the red dots, the earlier in time the weighted sampling algorithm samples these synthetic data points.
These plots complement our observations reported in the main paper: As the inverse temperature hyperparameter $\beta$ increases, the algorithm samples fewer synthetic data points $G_i$ but 
they are more dissimilar from a party's original dataset $D_i$. This is consistent with our reasoning in Sec.~{\ref{force}} that $\beta$ controls the trade-off between $|G_i|$ vs.~negative unbiased MMD. In addition, the plots show that more synthetic data points $G_i$ tend to be distributed to a party with a higher $\alpha_i$ as reward, which was empirically verified by the generally positive correlations of $|G_i|$ with $\alpha_i$ in Fig.~\ref{fig:corrs}. 
\begin{figure}[H]
    \centering
    \includegraphics[width=0.8\textwidth]{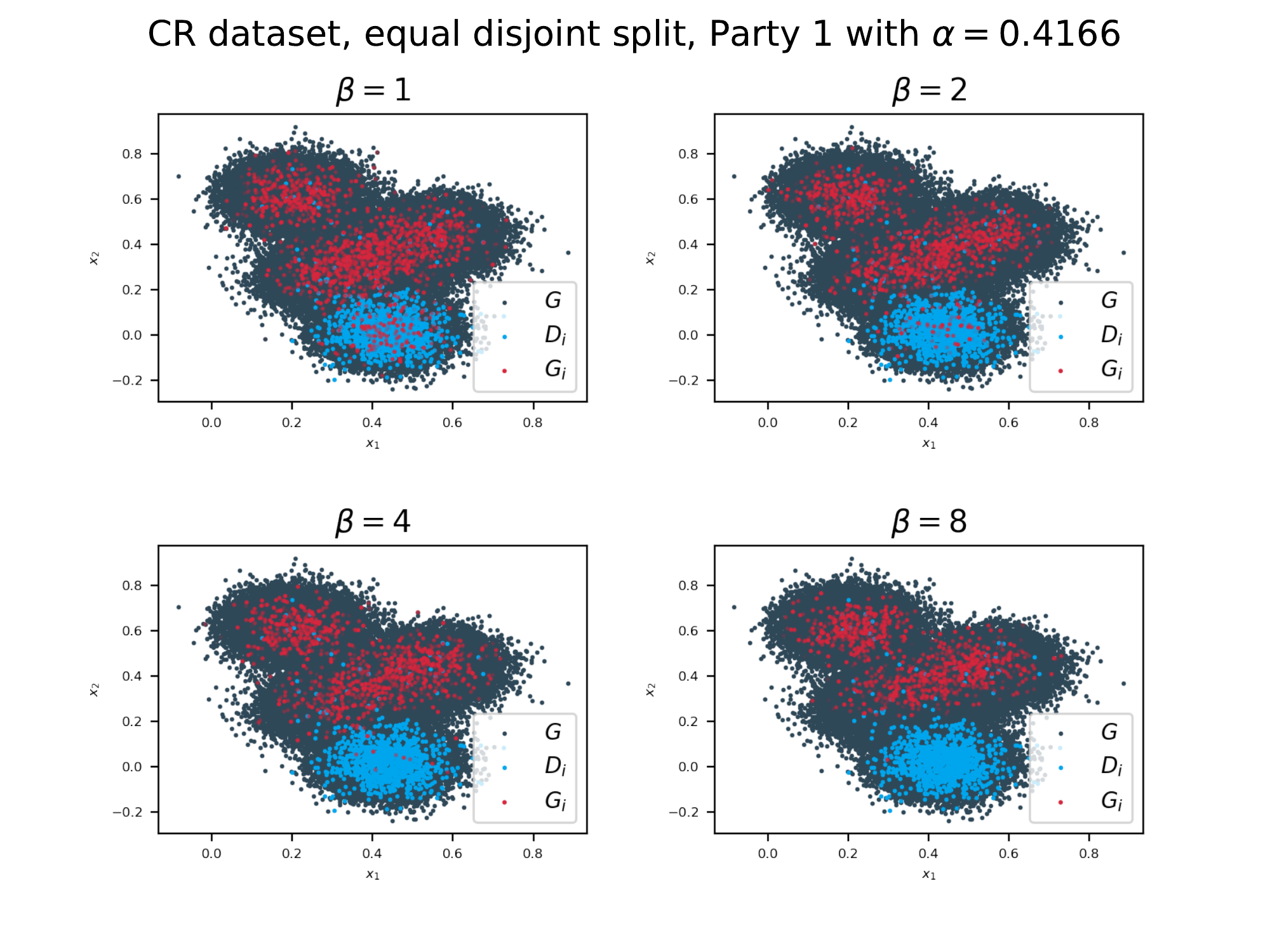}
\end{figure}

\begin{figure}[H]
    \centering
    \includegraphics[width=0.8\textwidth]{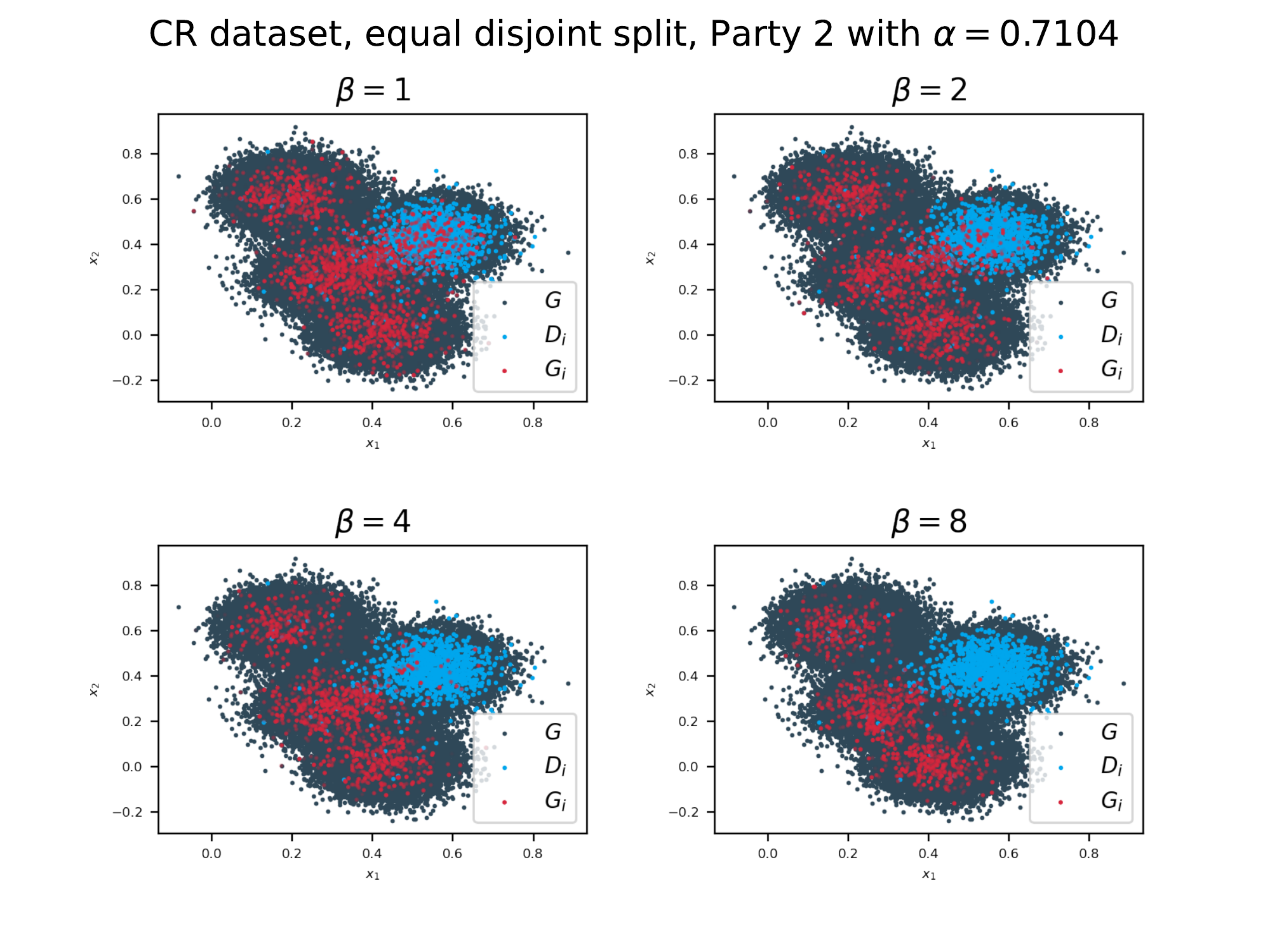}
\end{figure}
\begin{figure}[H]
    \centering
    \includegraphics[width=0.8\textwidth]{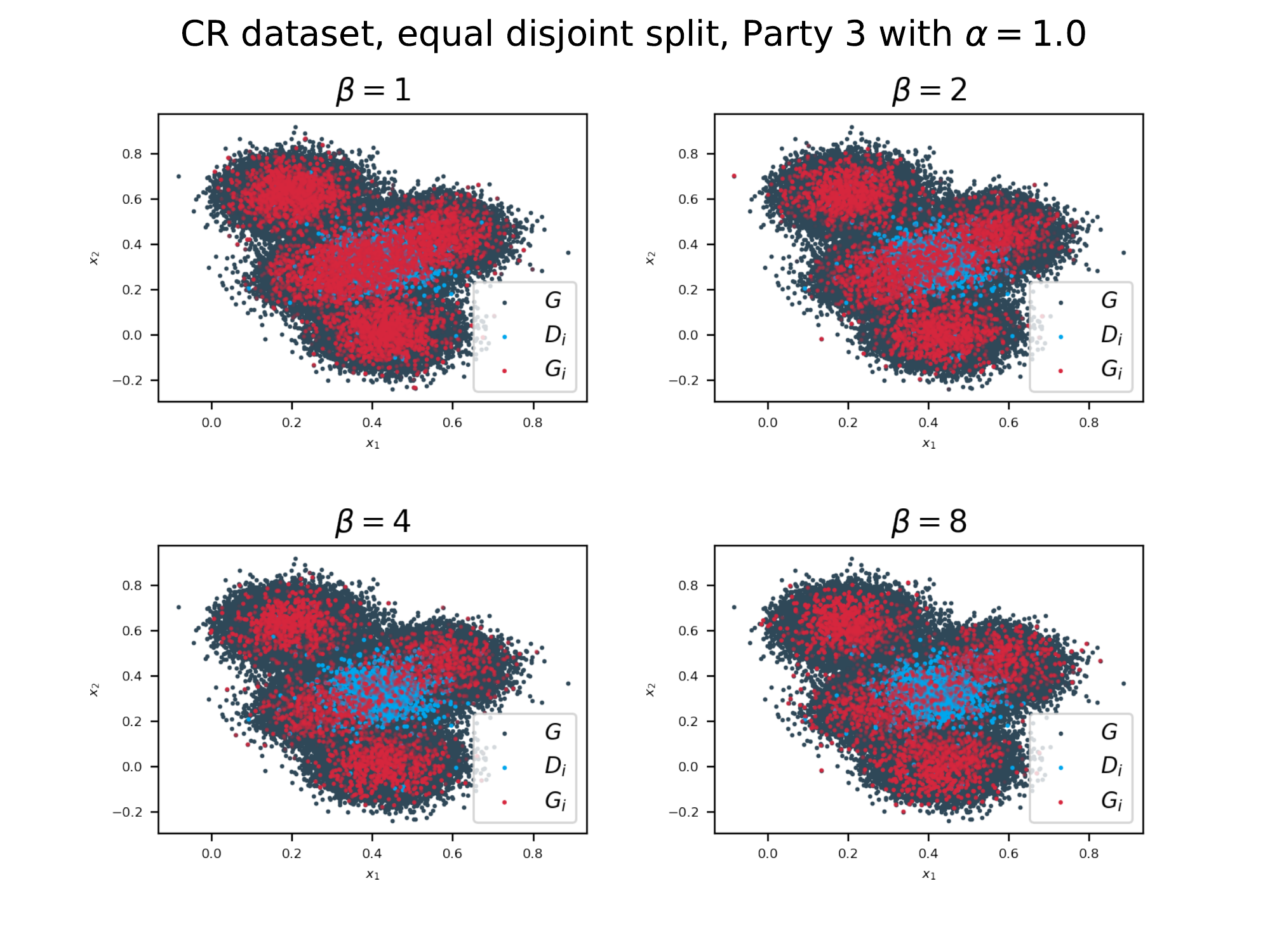}
\end{figure}
\begin{figure}[H]
    \centering
    \includegraphics[width=0.8\textwidth]{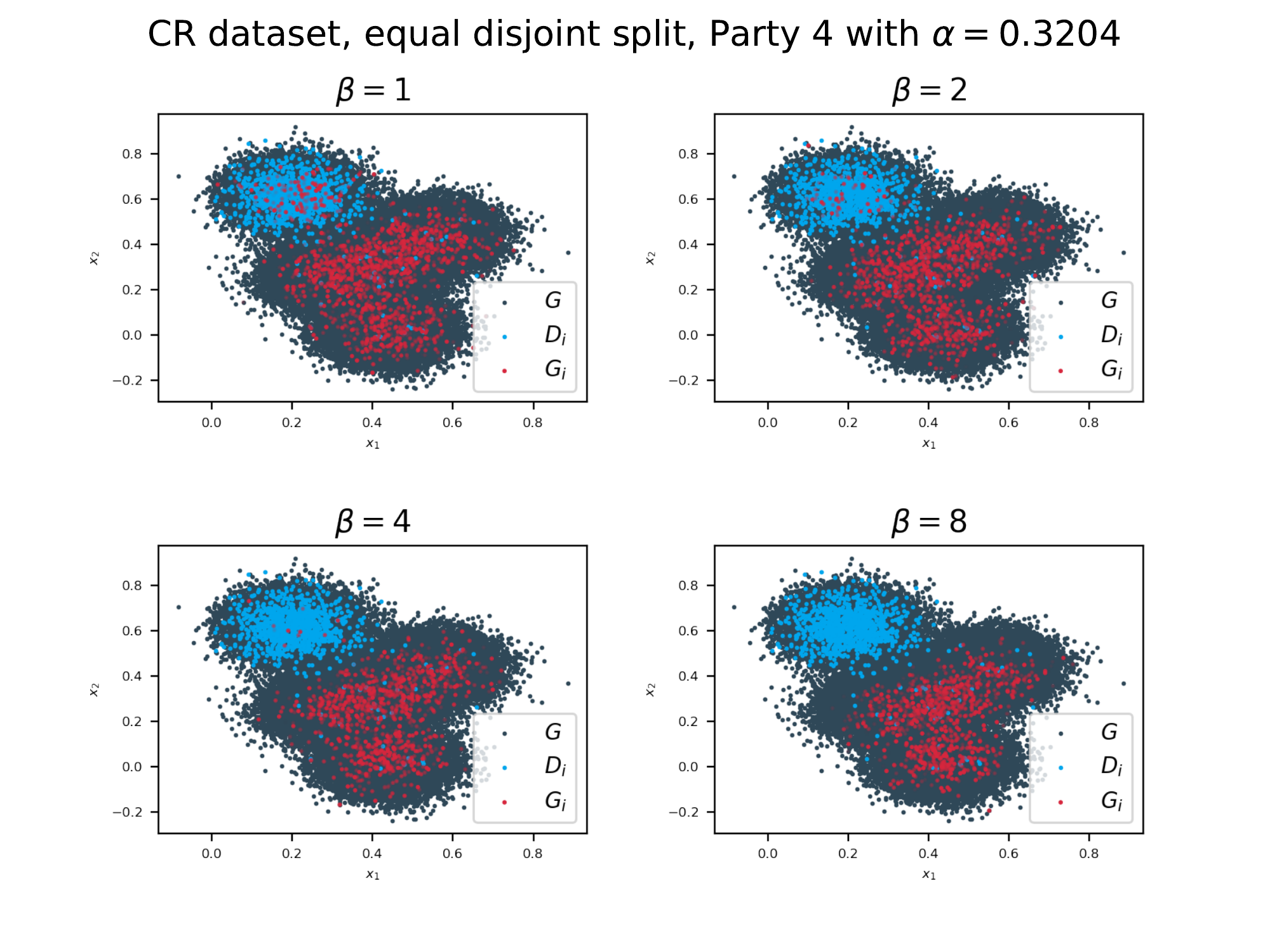}
\end{figure}
\begin{figure}[H]
    \centering
    \includegraphics[width=0.8\textwidth]{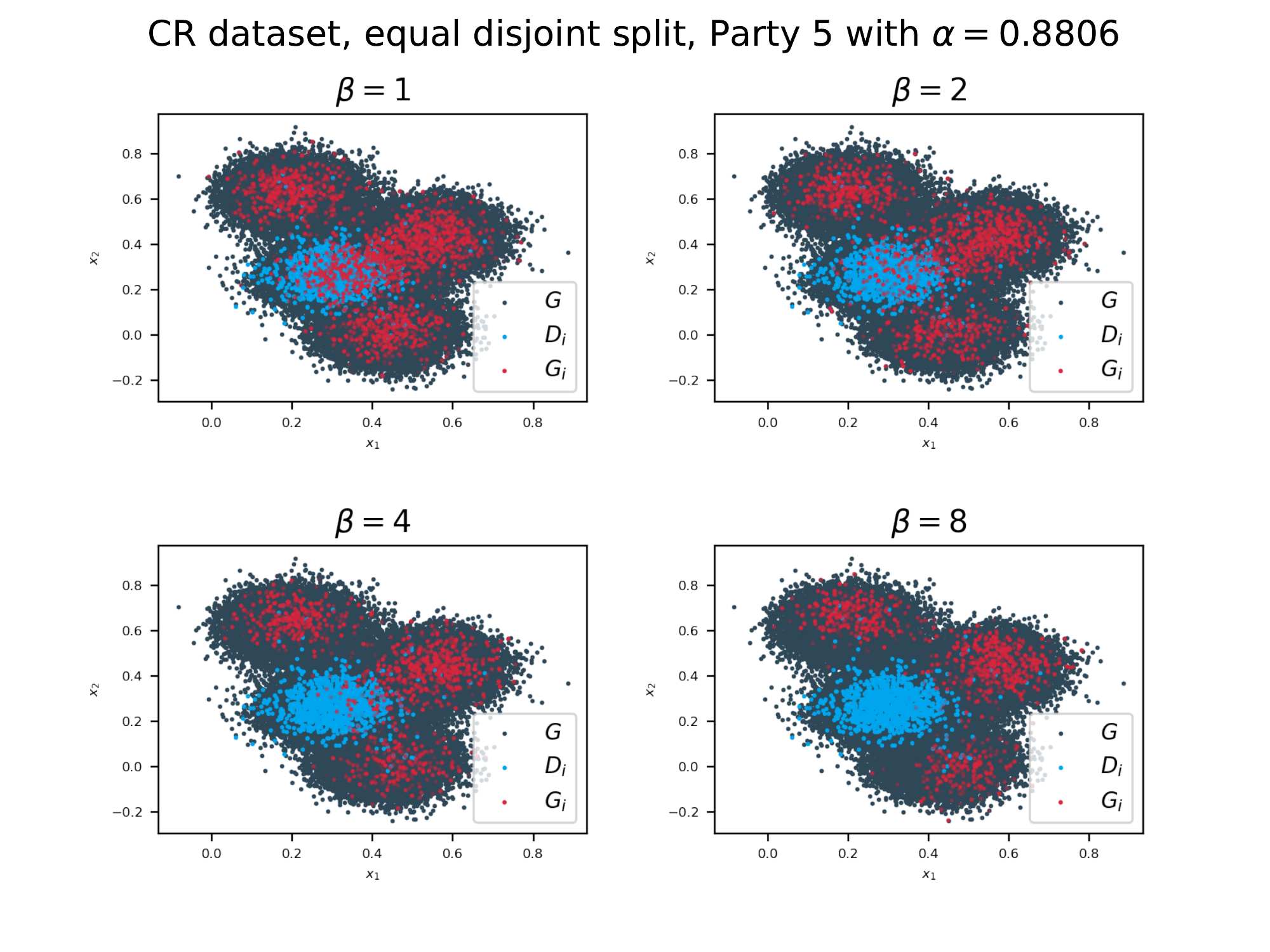}
\end{figure}
\begin{figure}[H]
    \centering
    \includegraphics[width=0.8\textwidth]{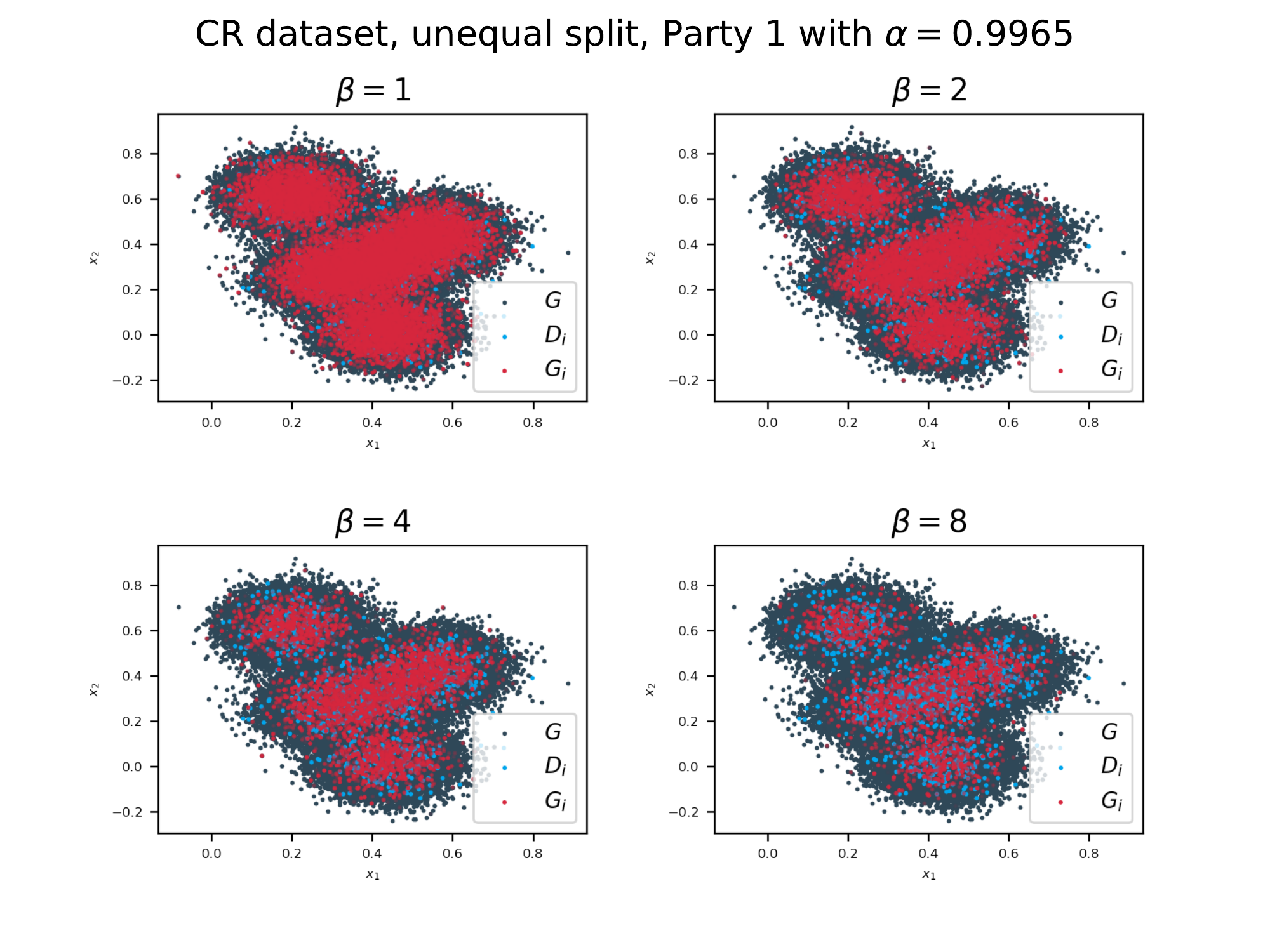}
\end{figure}
\begin{figure}[H]
    \centering
    \includegraphics[width=0.8\textwidth]{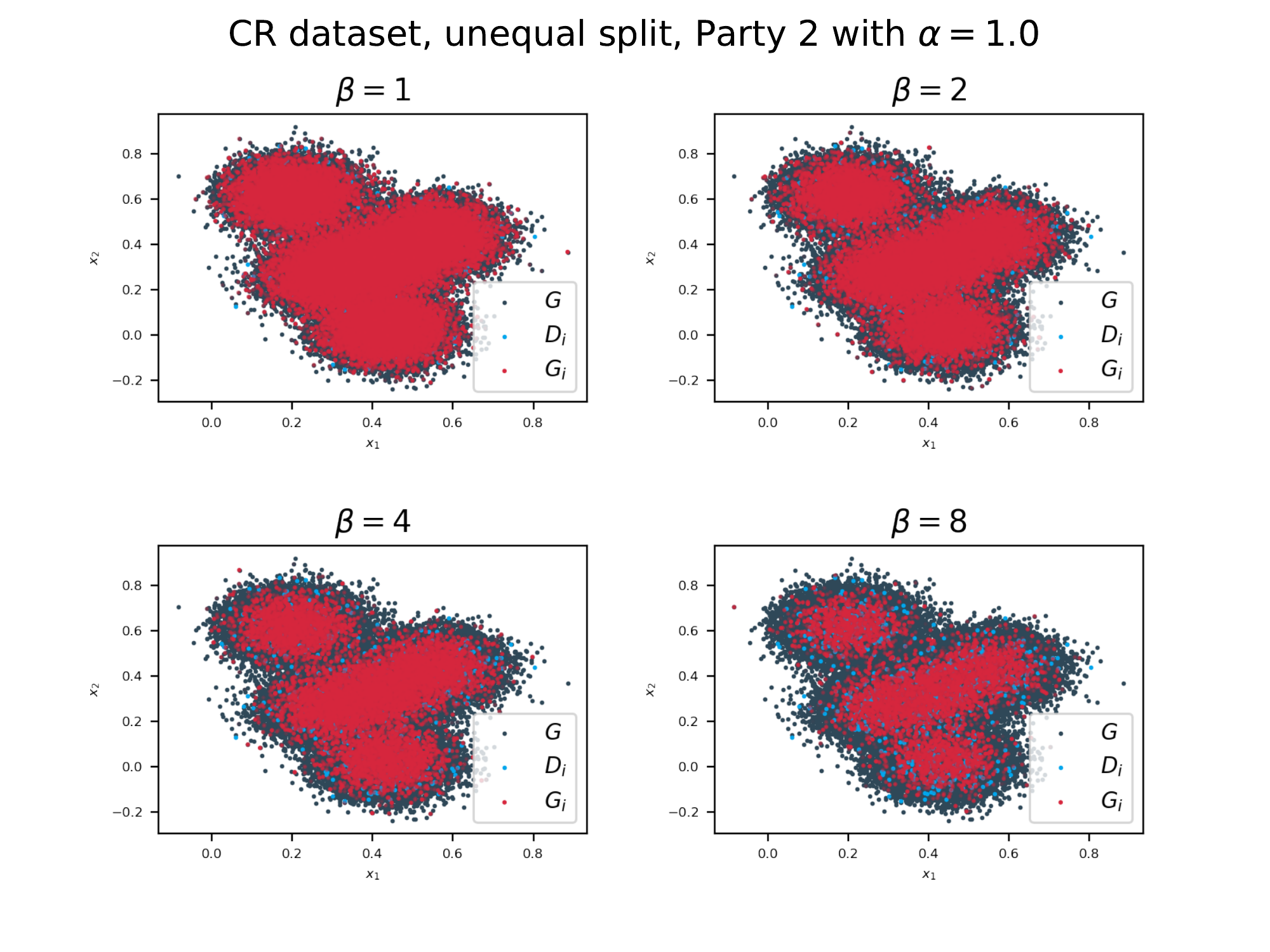}
\end{figure}
\begin{figure}[H]
    \centering
    \includegraphics[width=0.8\textwidth]{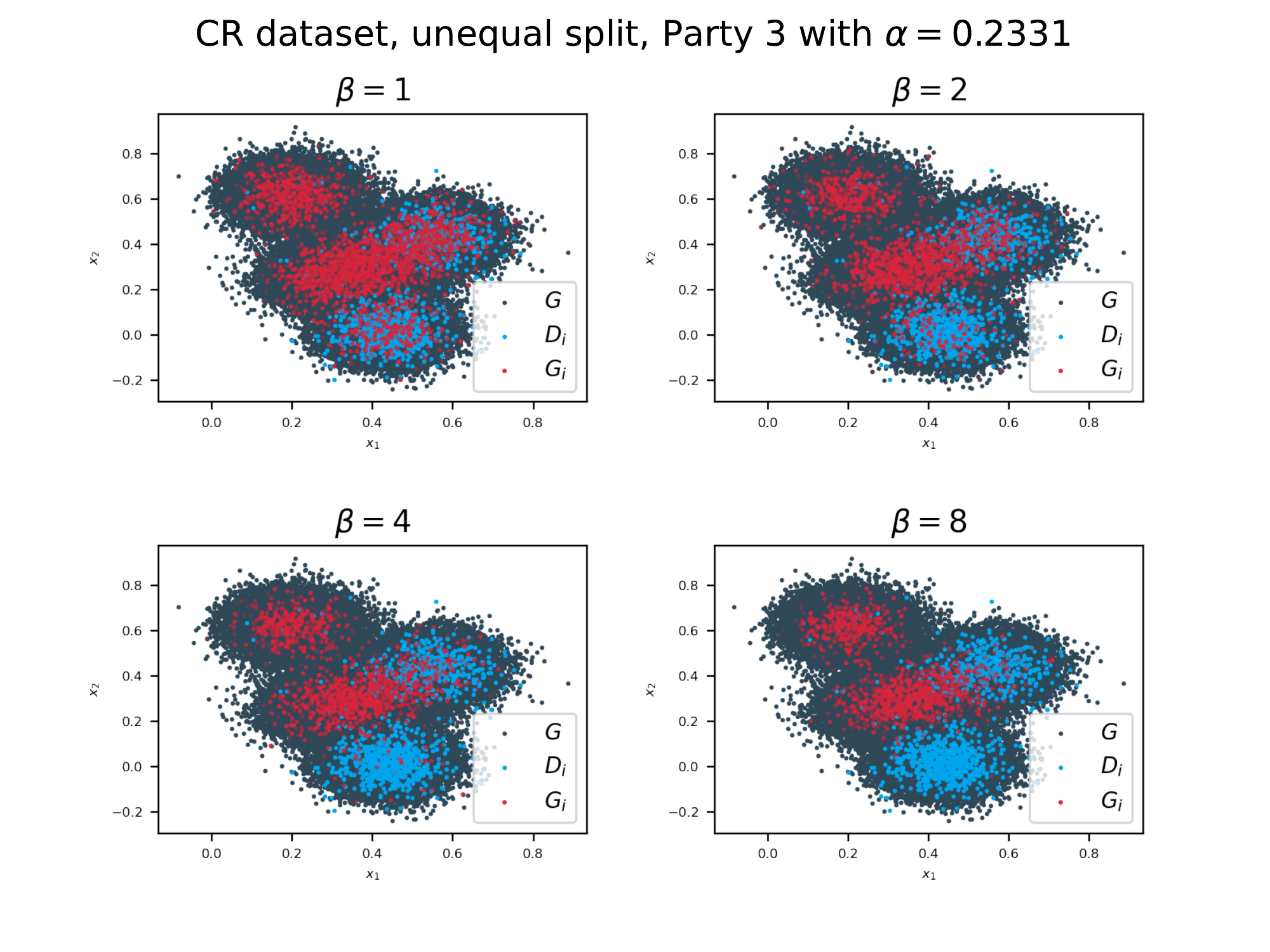}
\end{figure}
\begin{figure}[H]
    \centering
    \includegraphics[width=0.8\textwidth]{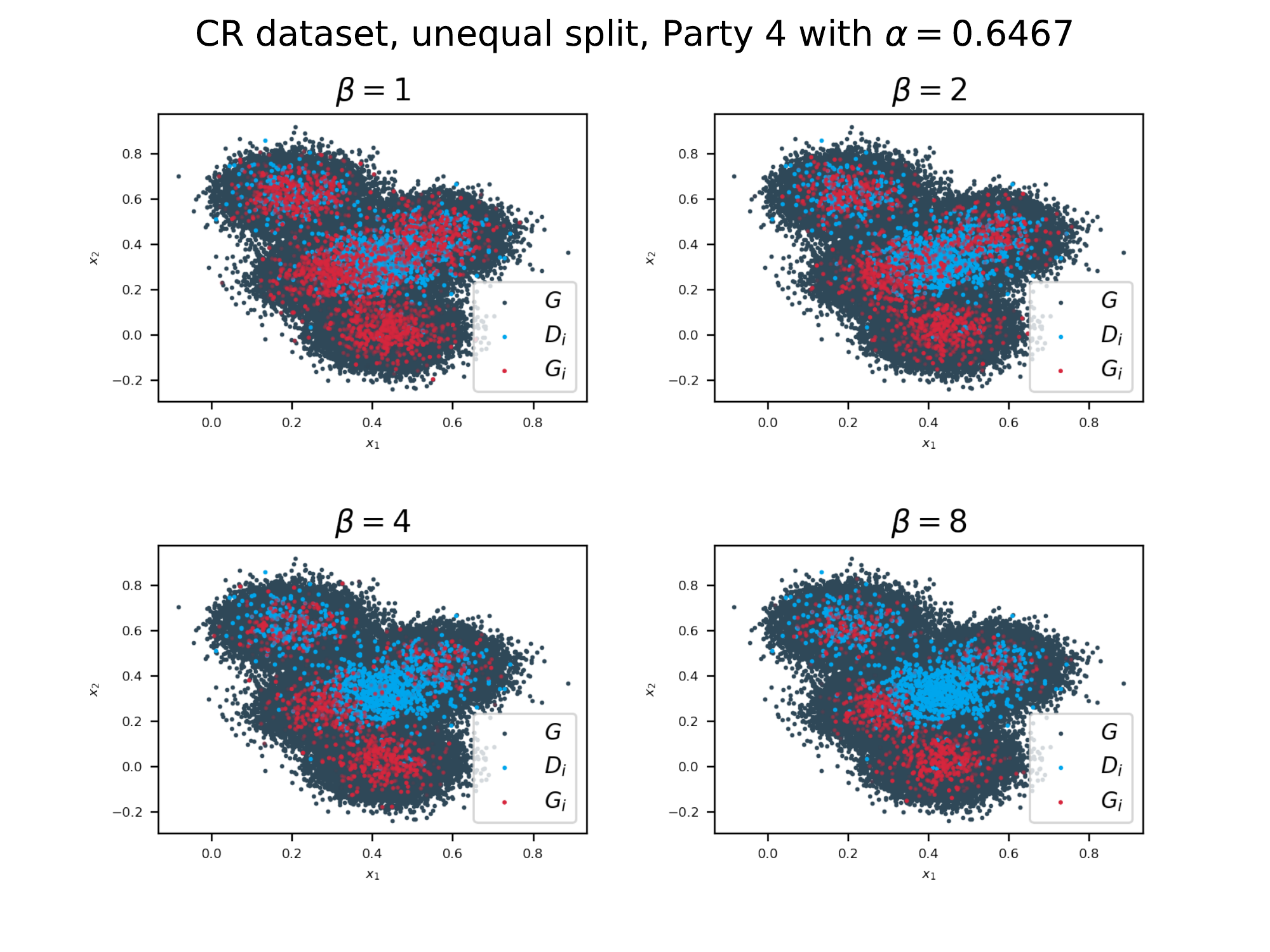}
\end{figure}
\begin{figure}[H]
    \centering
    \includegraphics[width=0.8\textwidth]{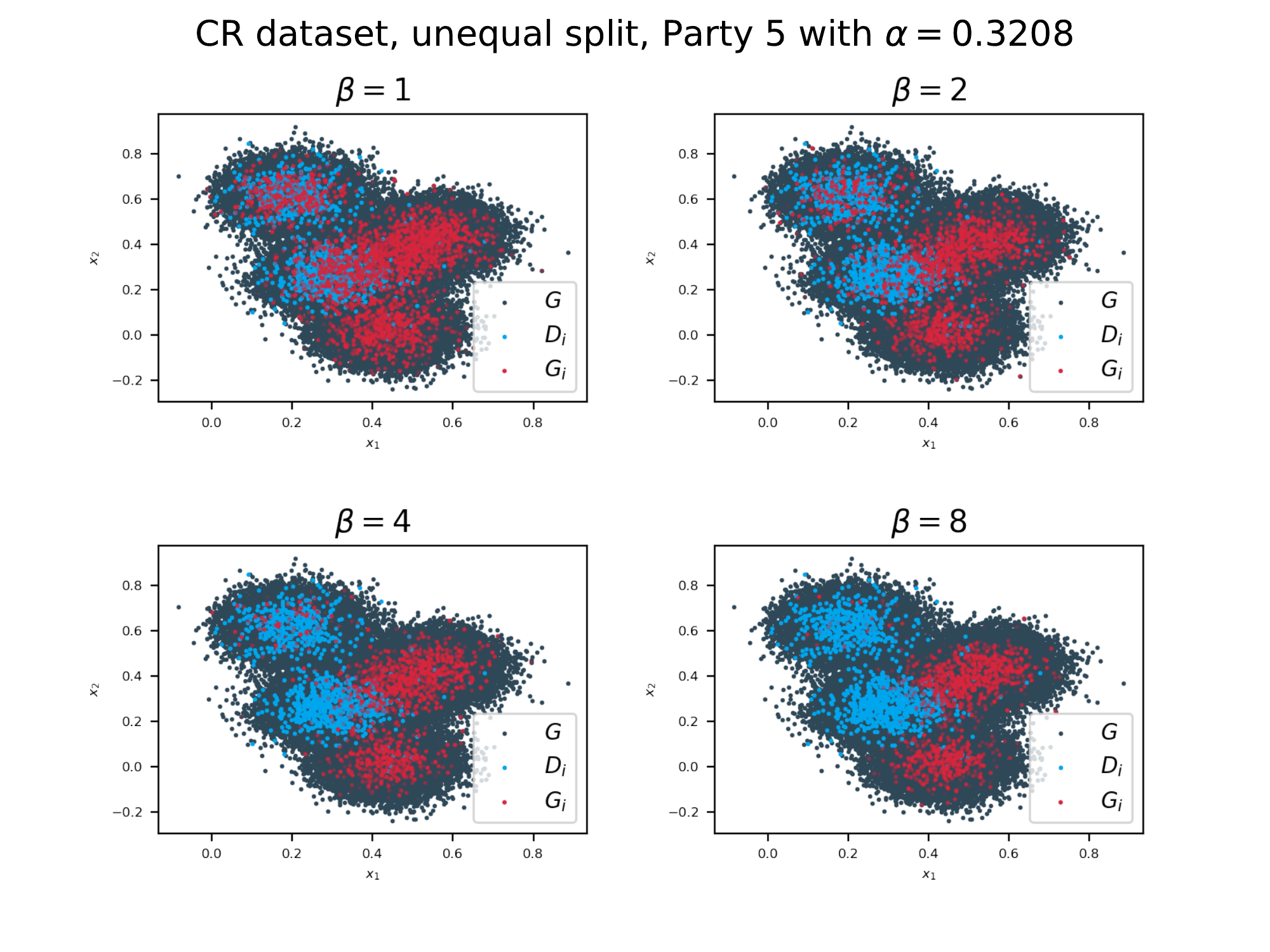}
\end{figure}
\begin{figure}[H]
    \centering
    \includegraphics[width=0.8\textwidth]{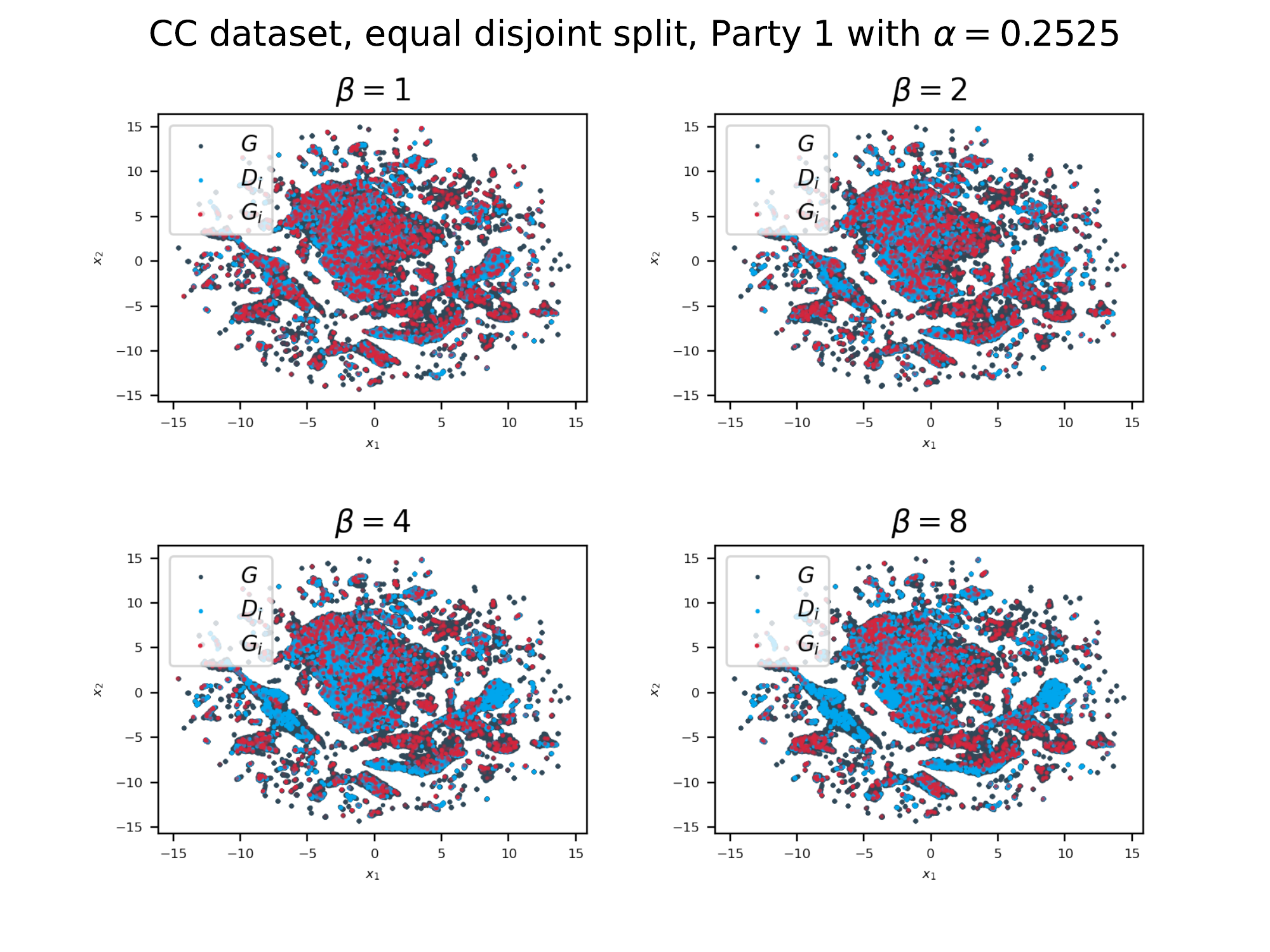}
\end{figure}
\begin{figure}[H]
    \centering
    \includegraphics[width=0.8\textwidth]{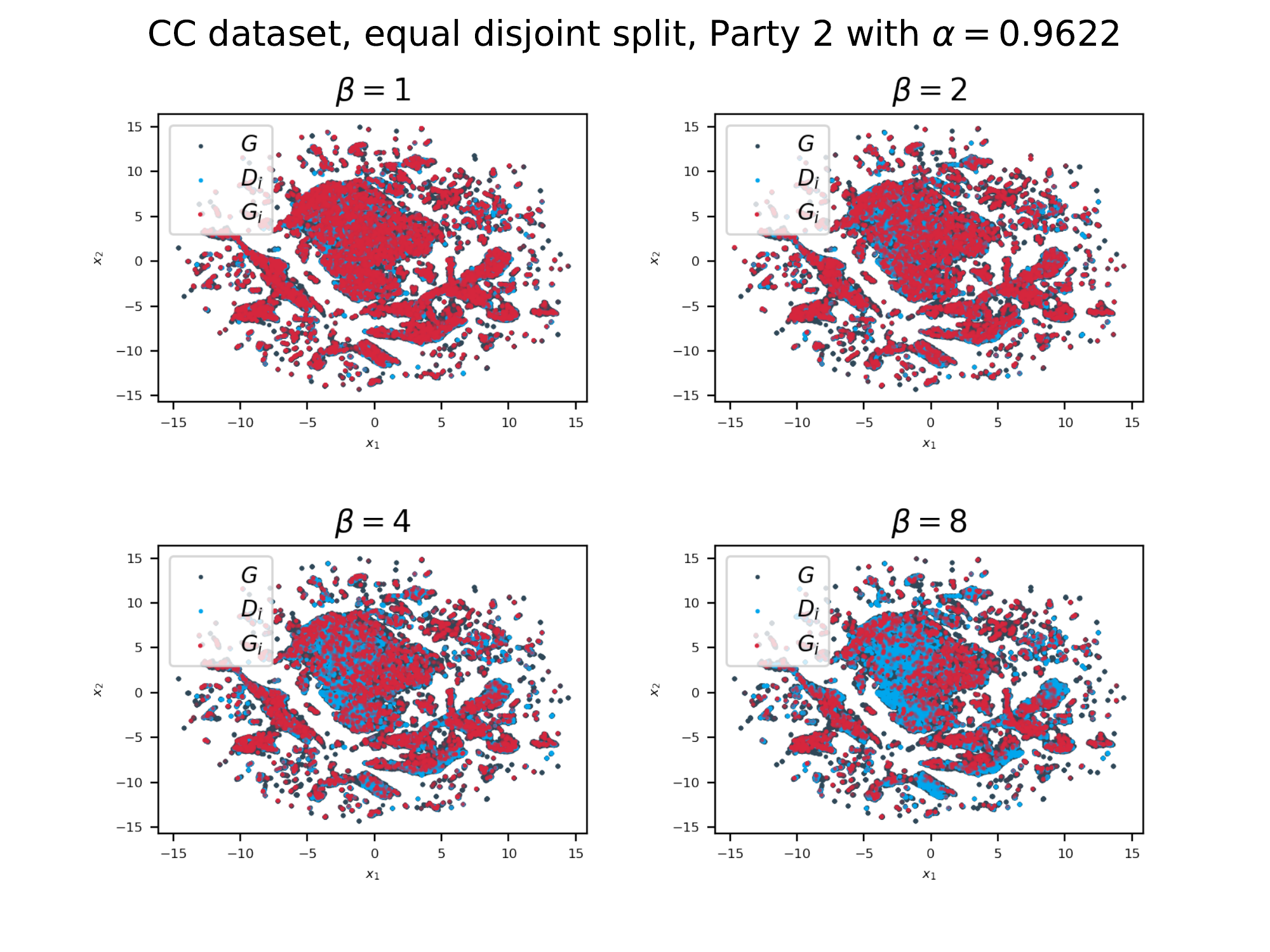}
\end{figure}
\begin{figure}[H]
    \centering
    \includegraphics[width=0.8\textwidth]{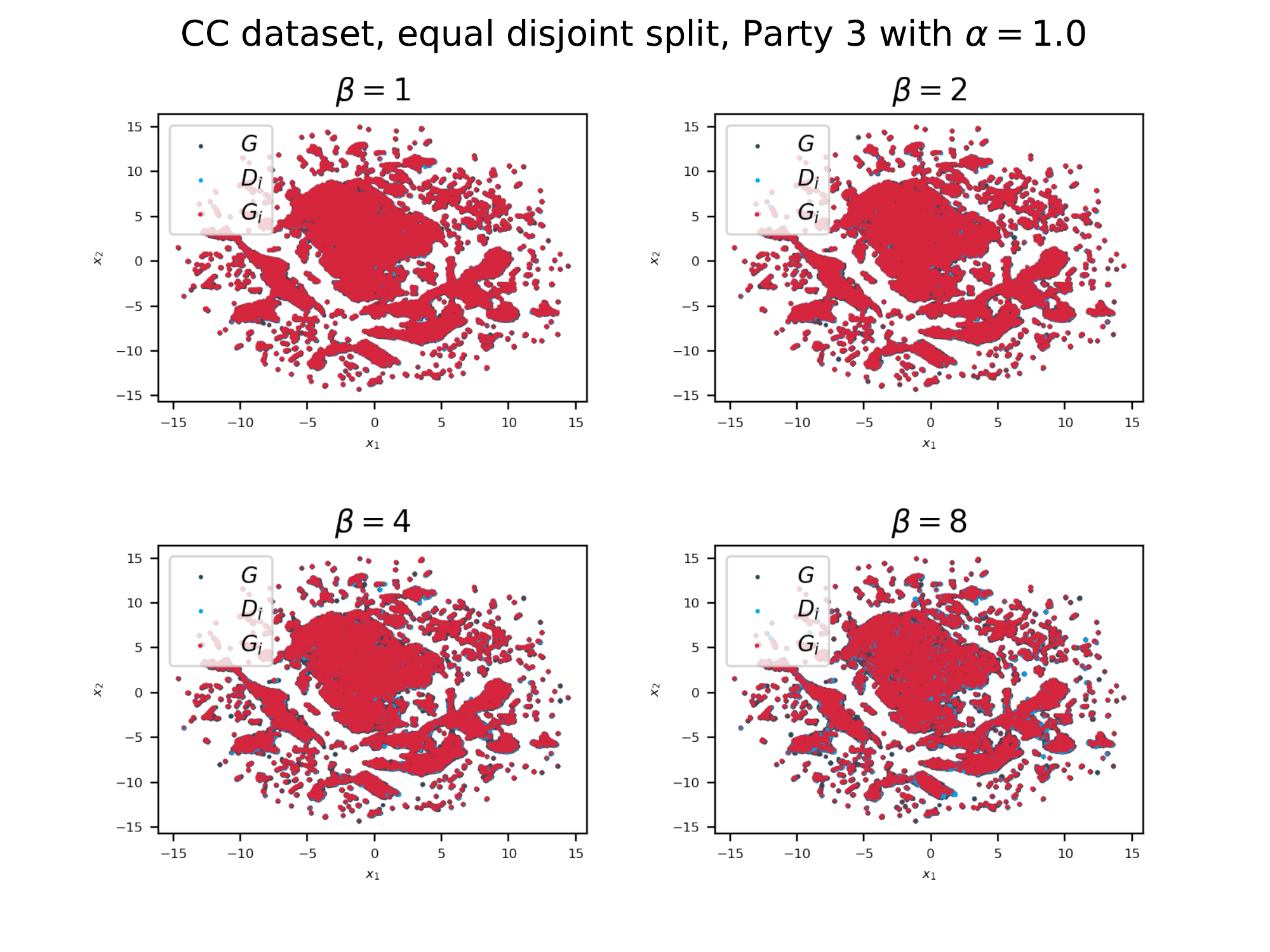}
\end{figure}
\begin{figure}[H]
    \centering
    \includegraphics[width=0.8\textwidth]{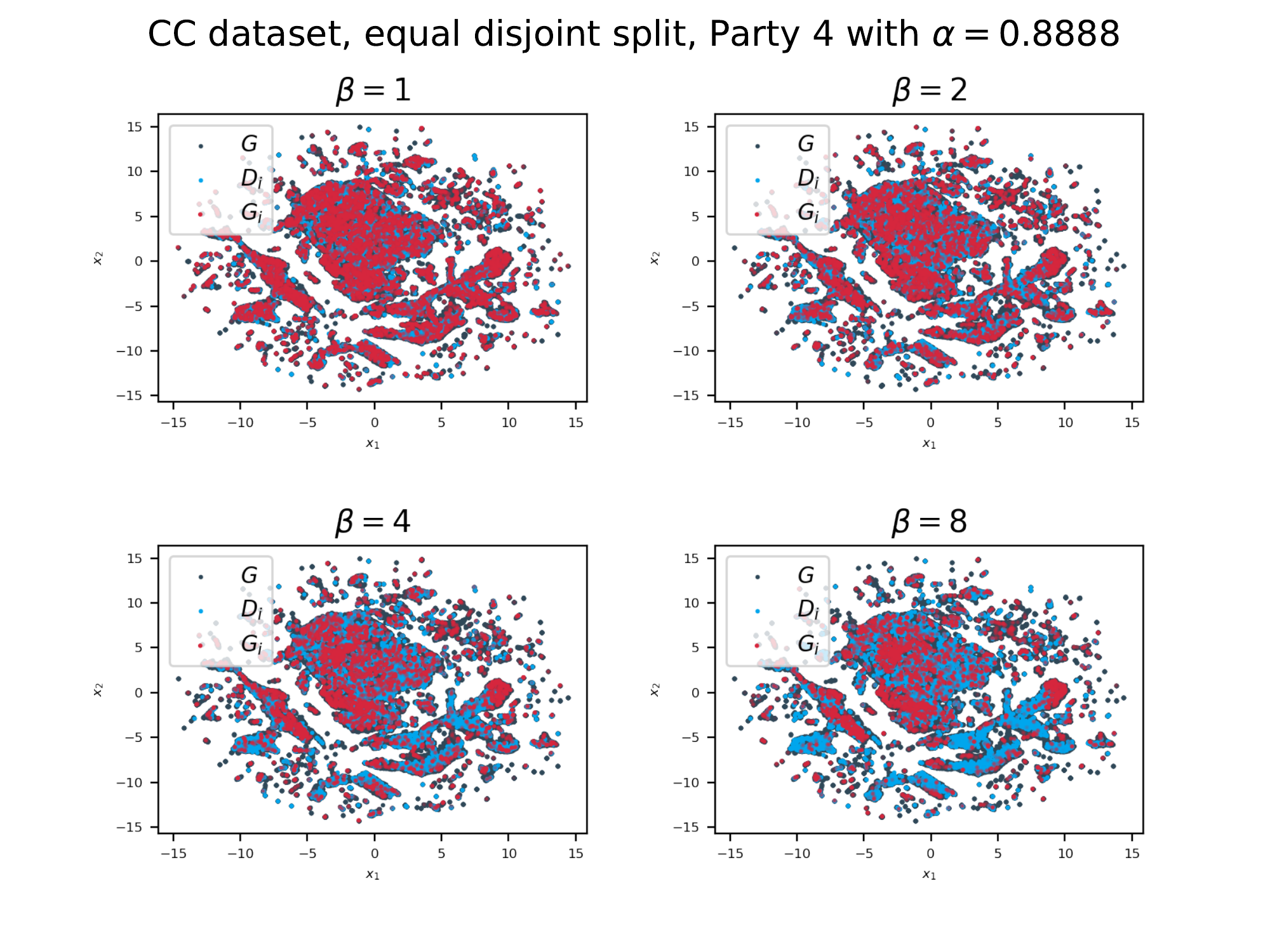}
\end{figure}
\begin{figure}[H]
    \centering
    \includegraphics[width=0.8\textwidth]{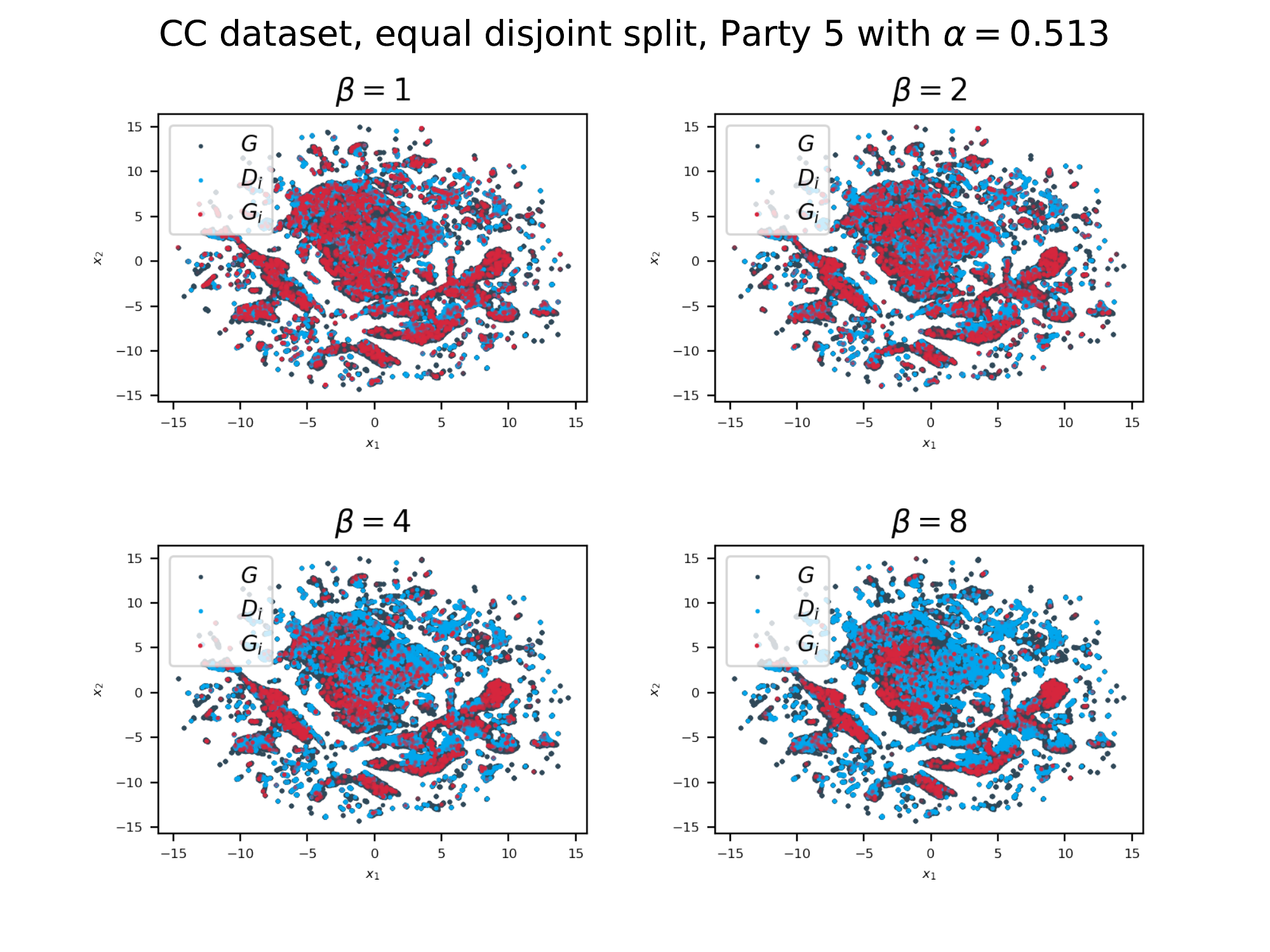}
\end{figure}
\begin{figure}[H]
    \centering
    \includegraphics[width=0.8\textwidth]{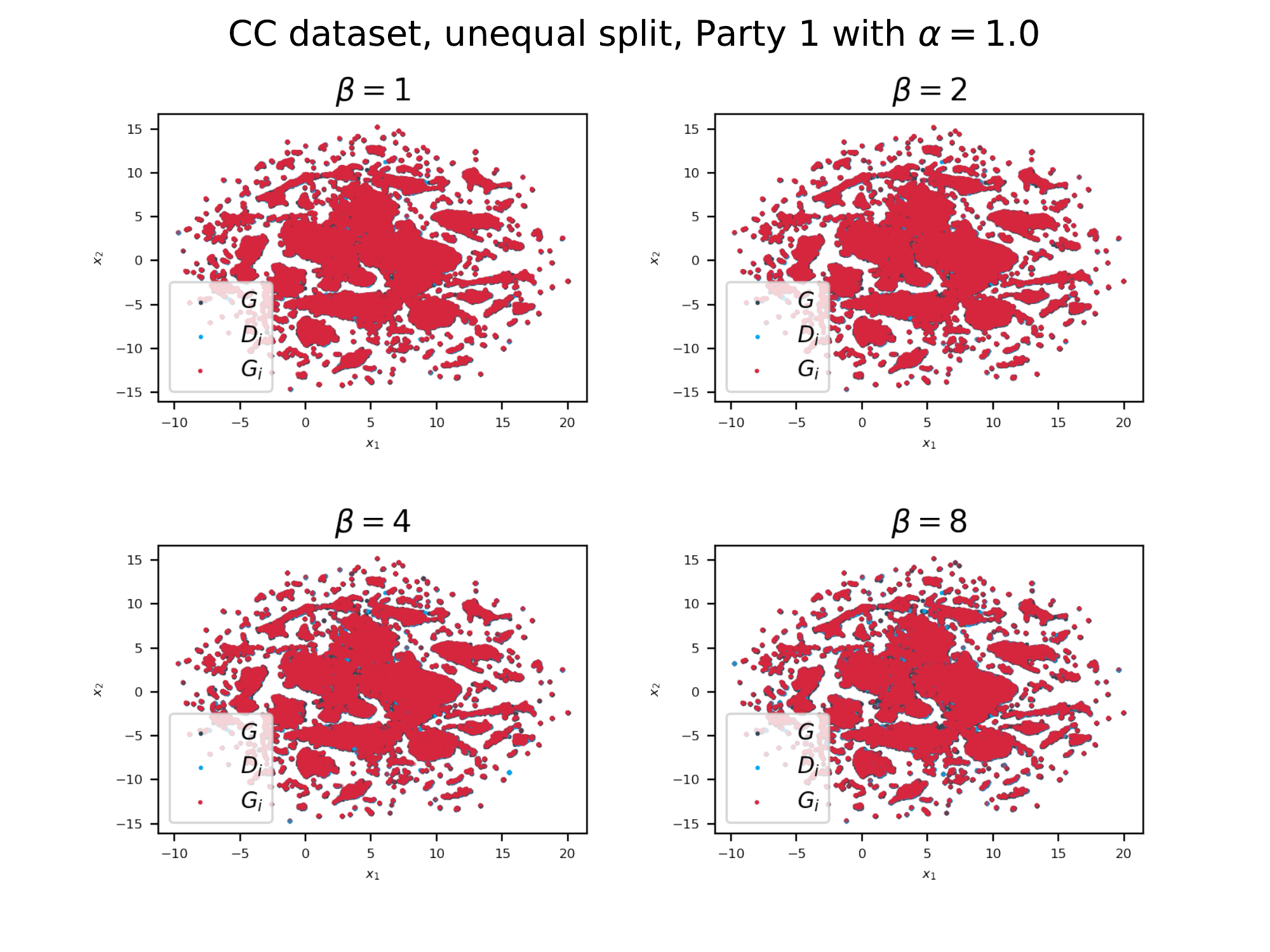}
\end{figure}
\begin{figure}[H]
    \centering
    \includegraphics[width=0.8\textwidth]{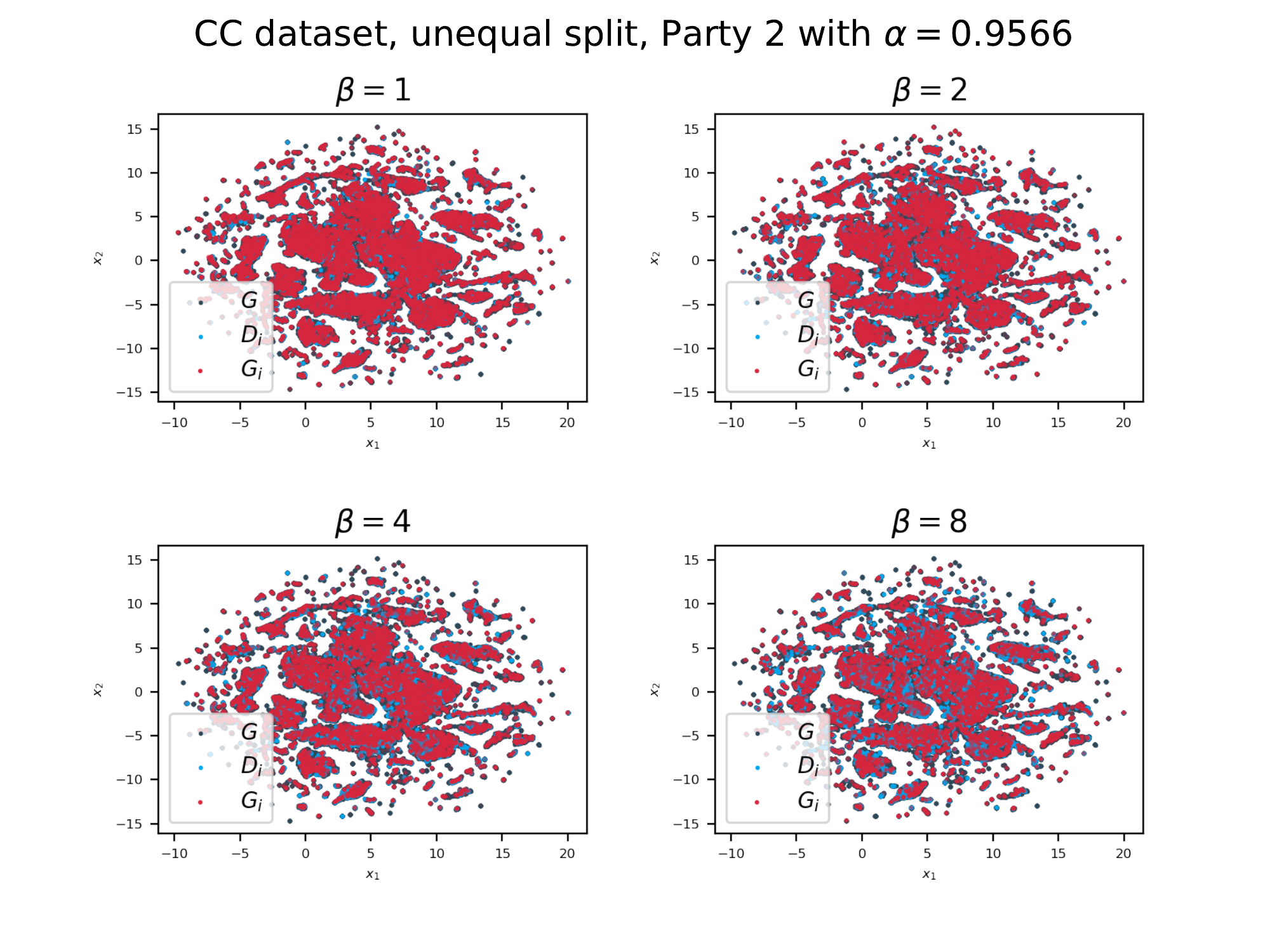}
\end{figure}
\begin{figure}[H]
    \centering
    \includegraphics[width=0.8\textwidth]{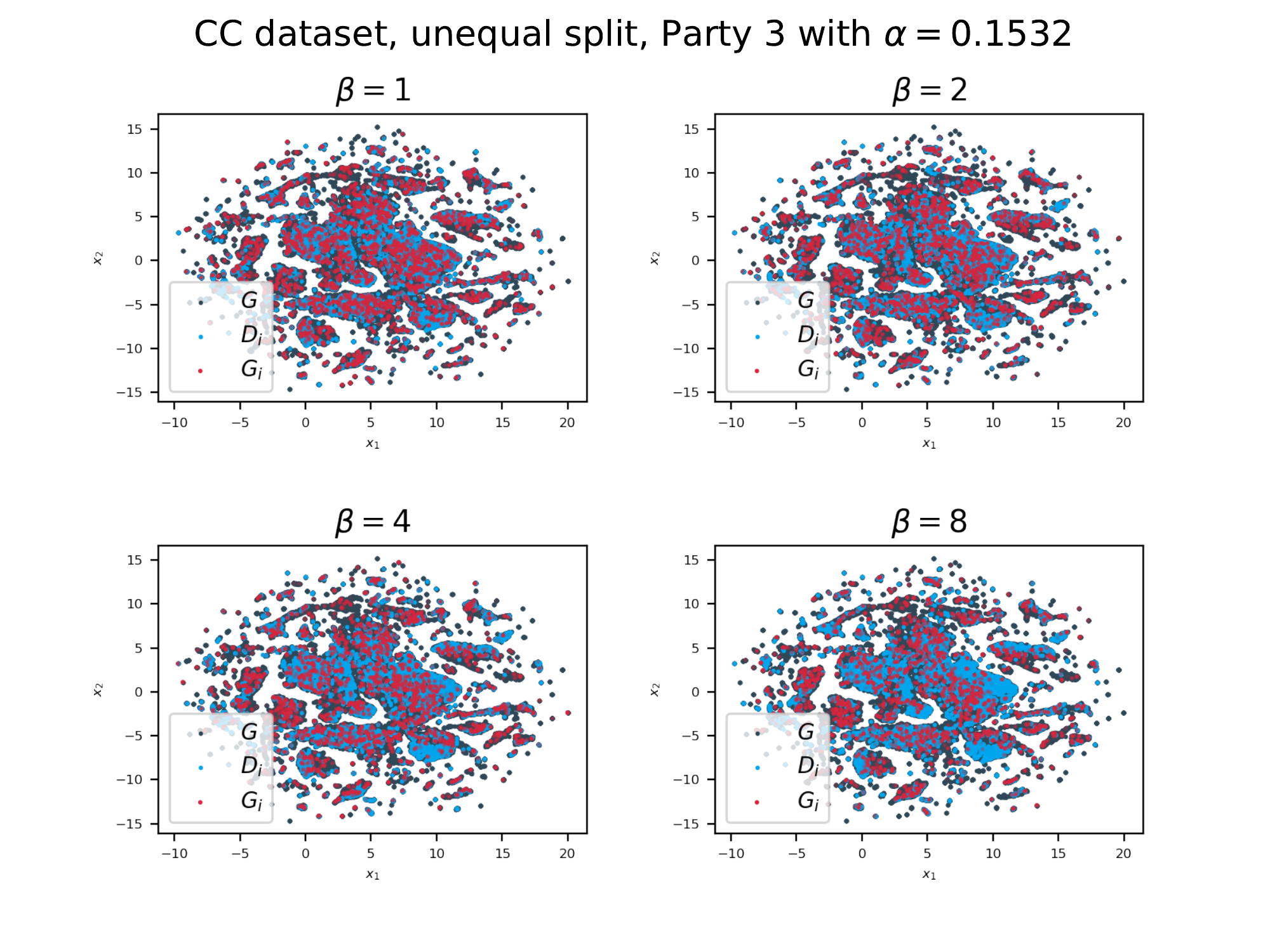}
\end{figure}
\begin{figure}[H]
    \centering
    \includegraphics[width=0.8\textwidth]{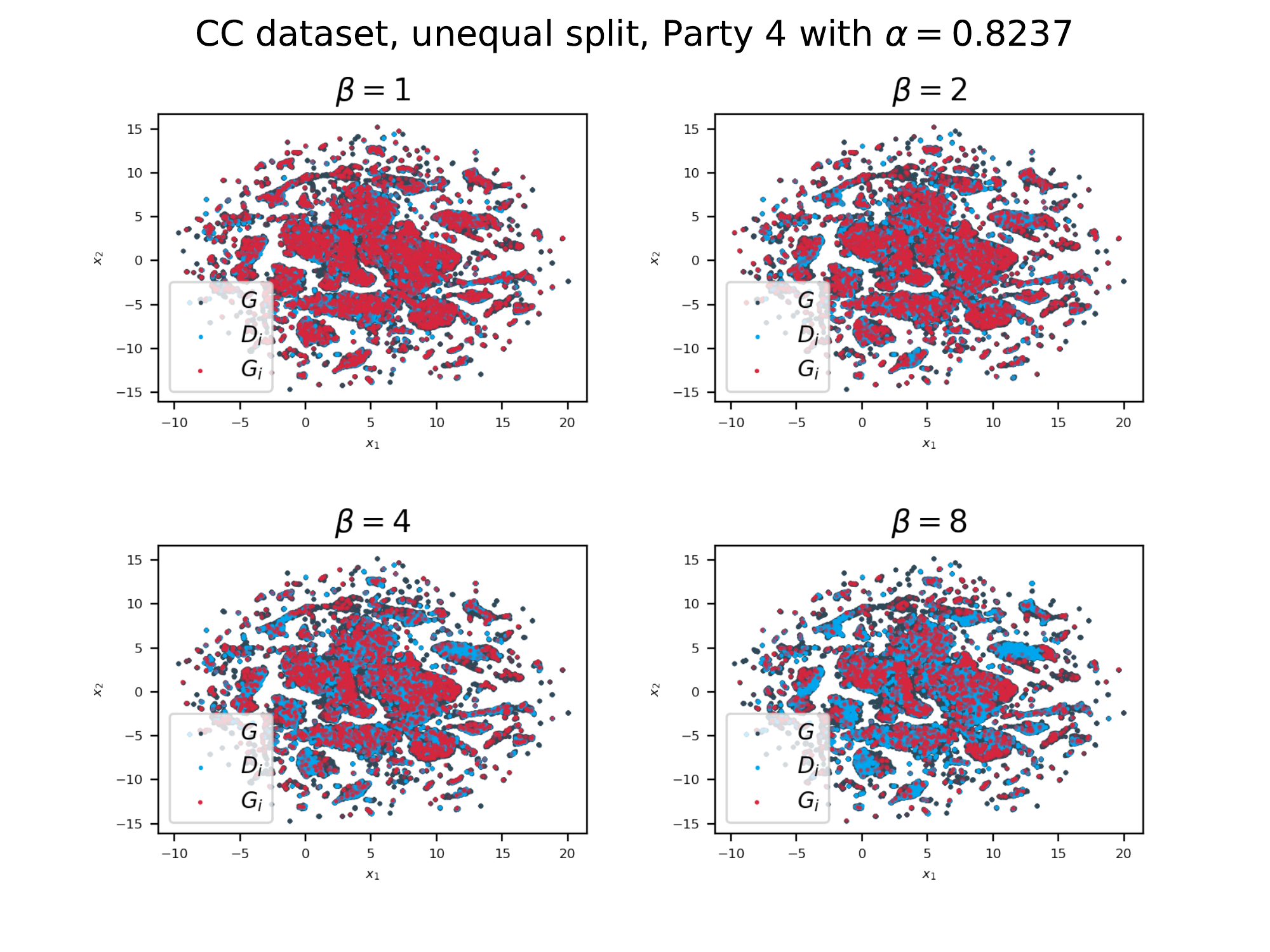}
\end{figure}
\begin{figure}[H]
    \centering
    \includegraphics[width=0.8\textwidth]{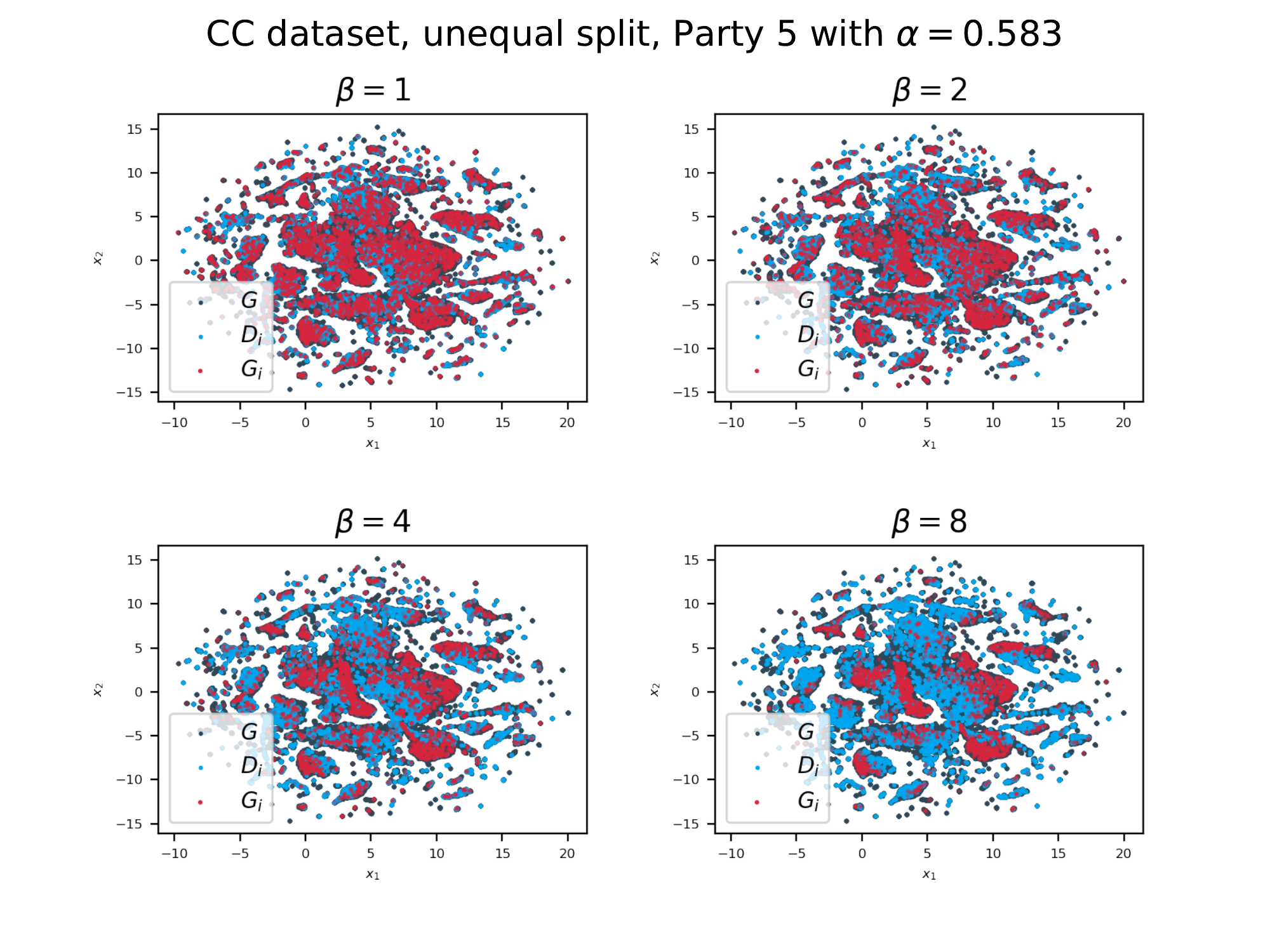}
\end{figure}
\begin{figure}[H]
    \centering
    \includegraphics[width=0.8\textwidth]{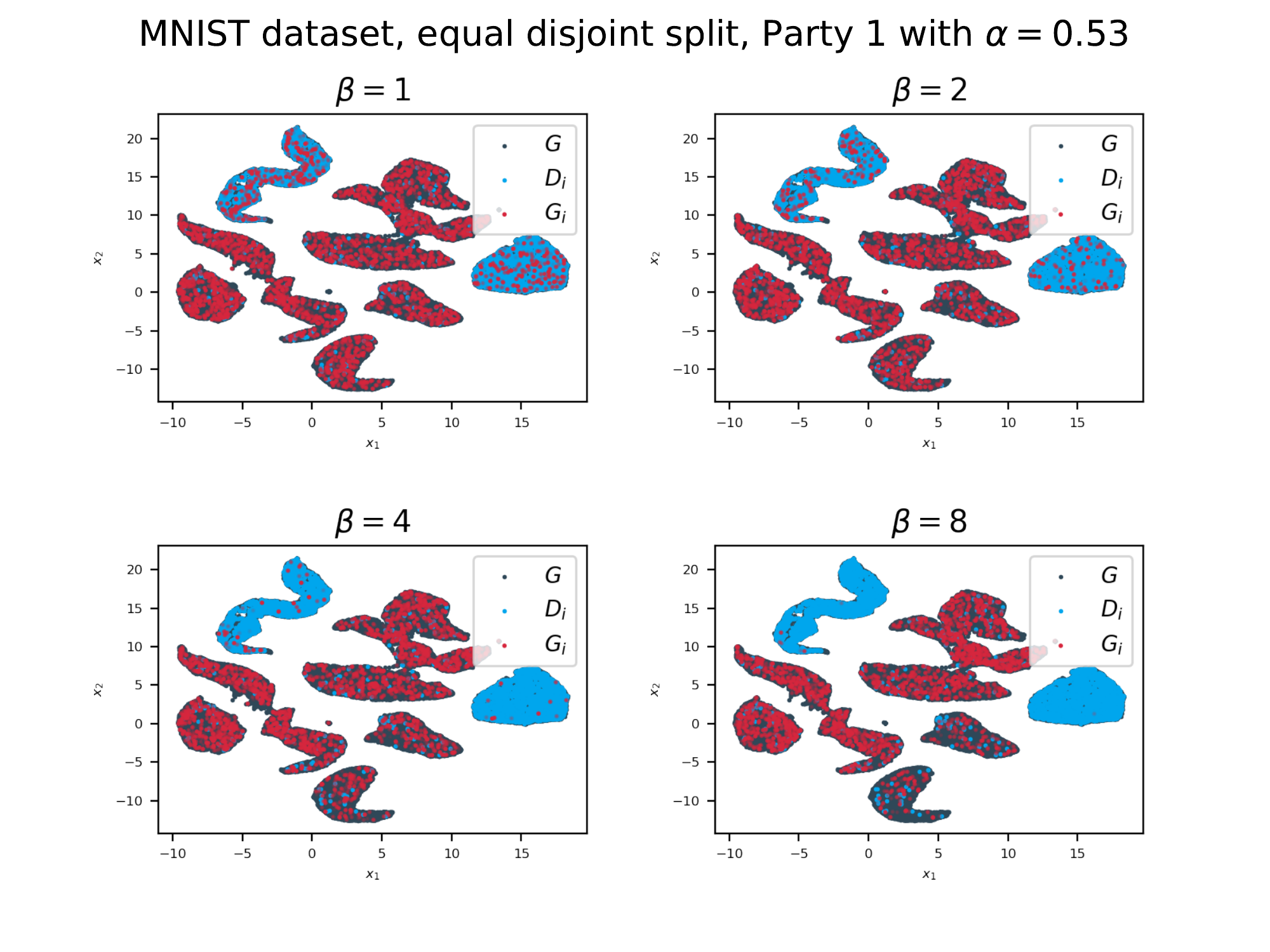}
\end{figure}
\begin{figure}[H]
    \centering
    \includegraphics[width=0.8\textwidth]{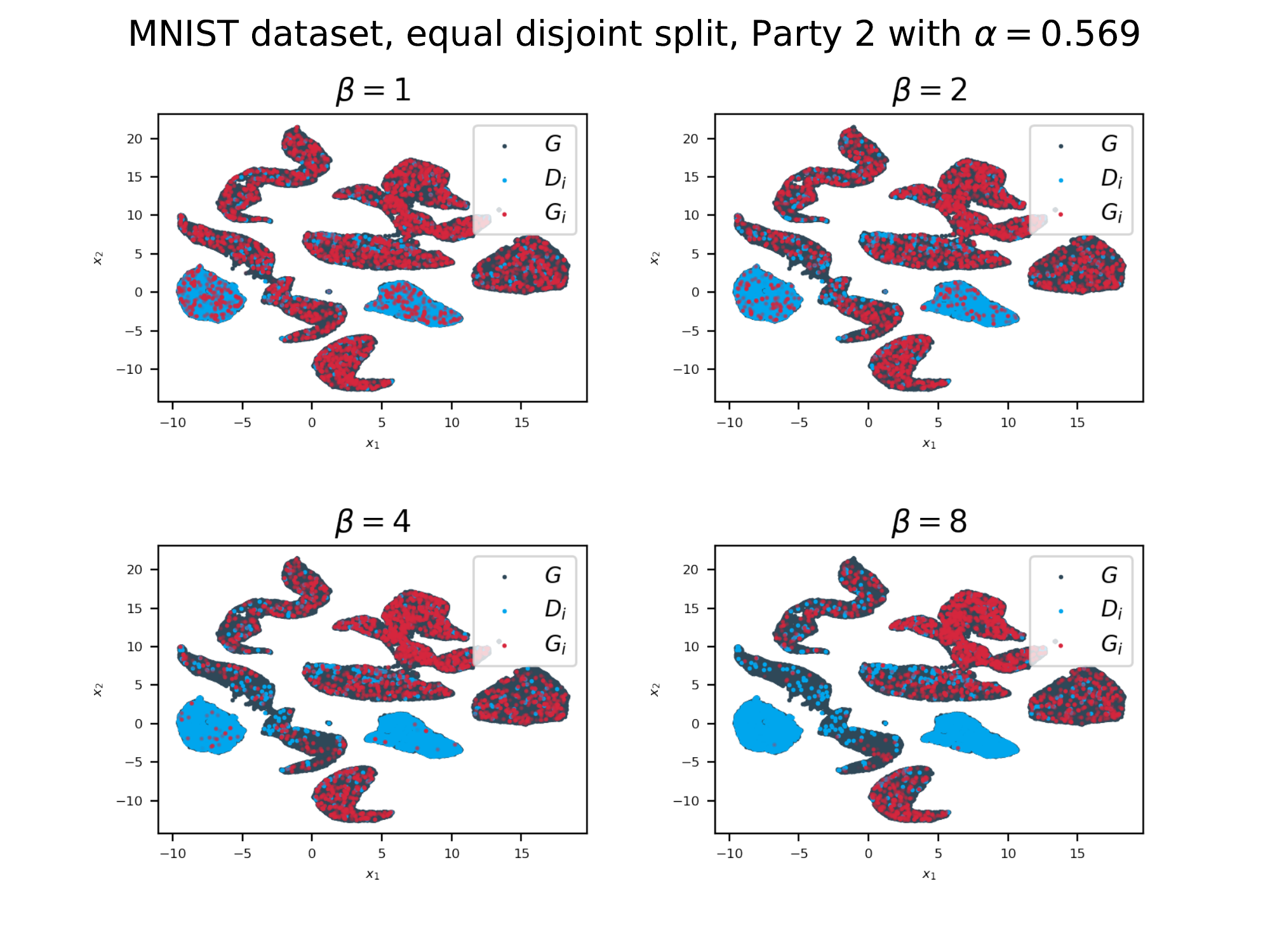}
\end{figure}
\begin{figure}[H]
    \centering
    \includegraphics[width=0.8\textwidth]{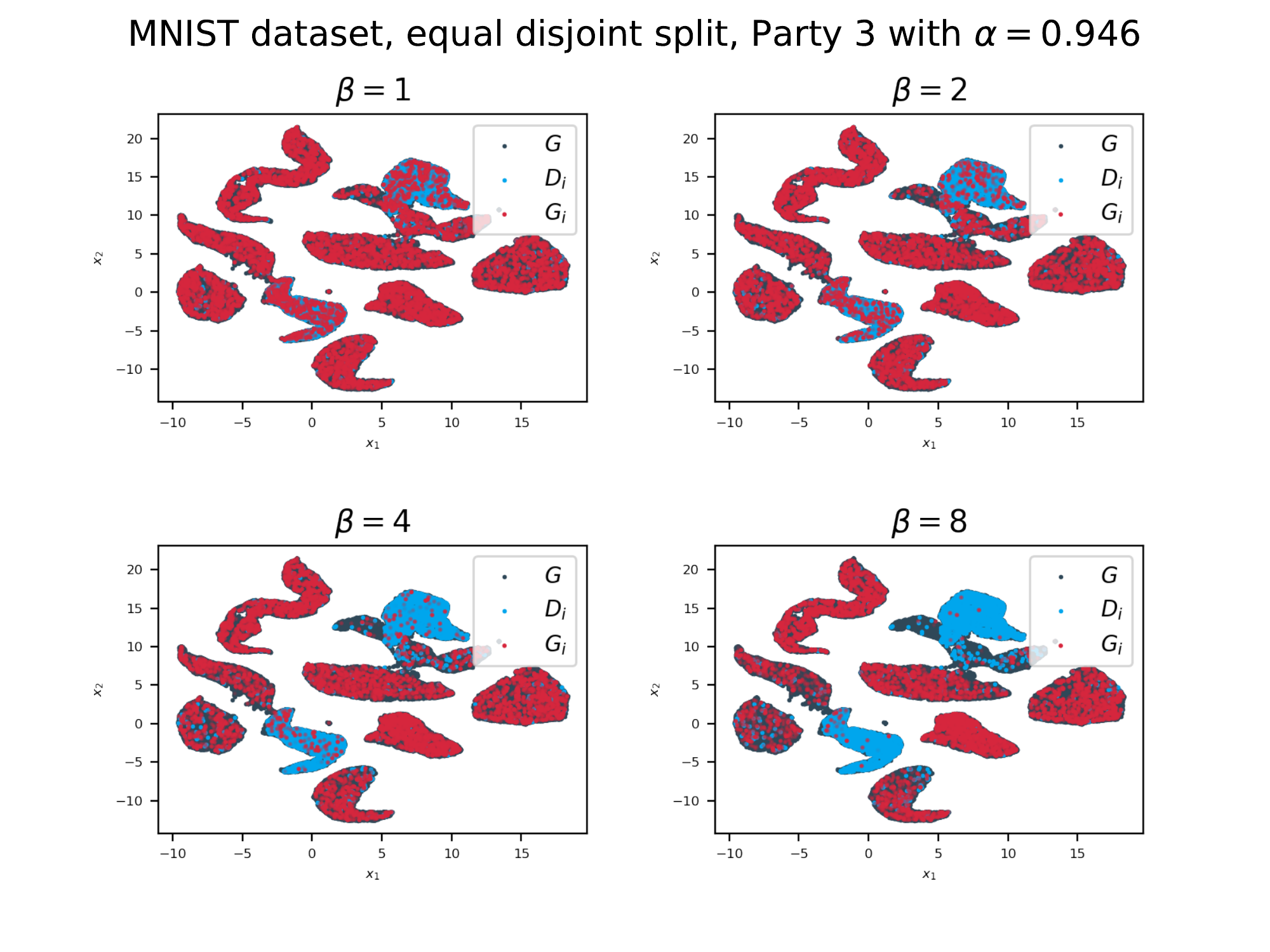}
\end{figure}
\begin{figure}[H]
    \centering
    \includegraphics[width=0.8\textwidth]{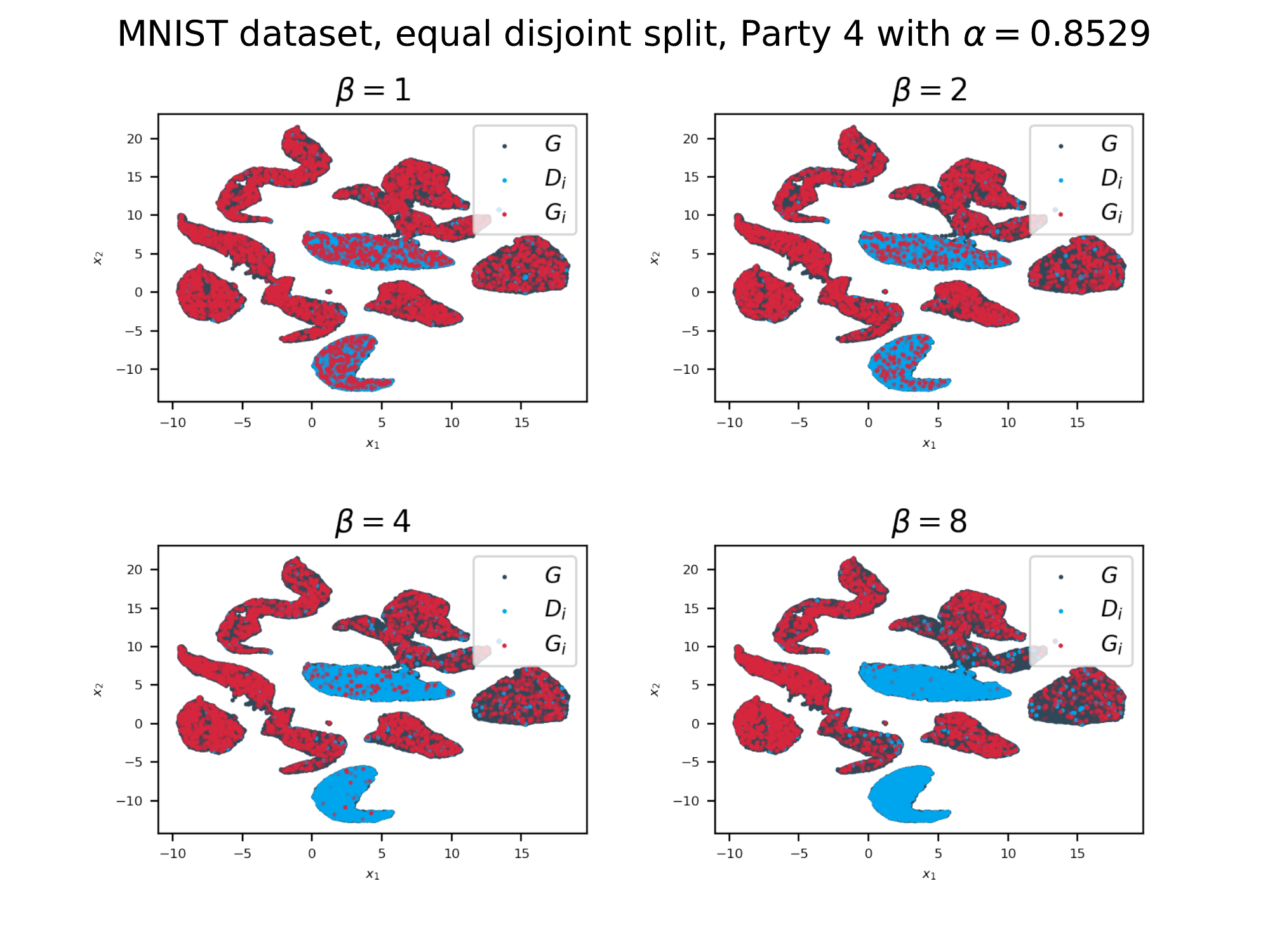}
\end{figure}
\begin{figure}[H]
    \centering
    \includegraphics[width=0.8\textwidth]{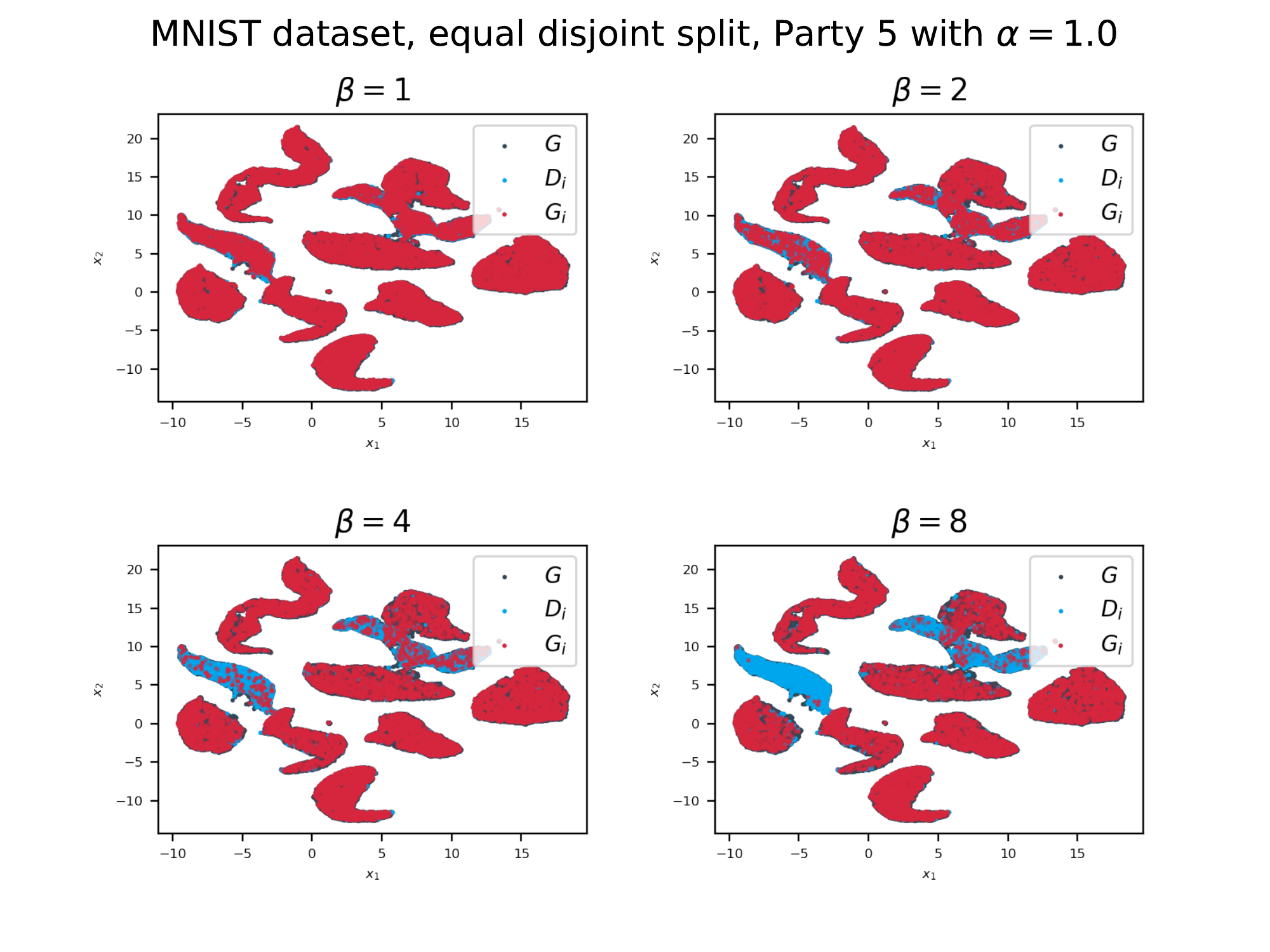}
\end{figure}
\begin{figure}[H]
    \centering
    \includegraphics[width=0.8\textwidth]{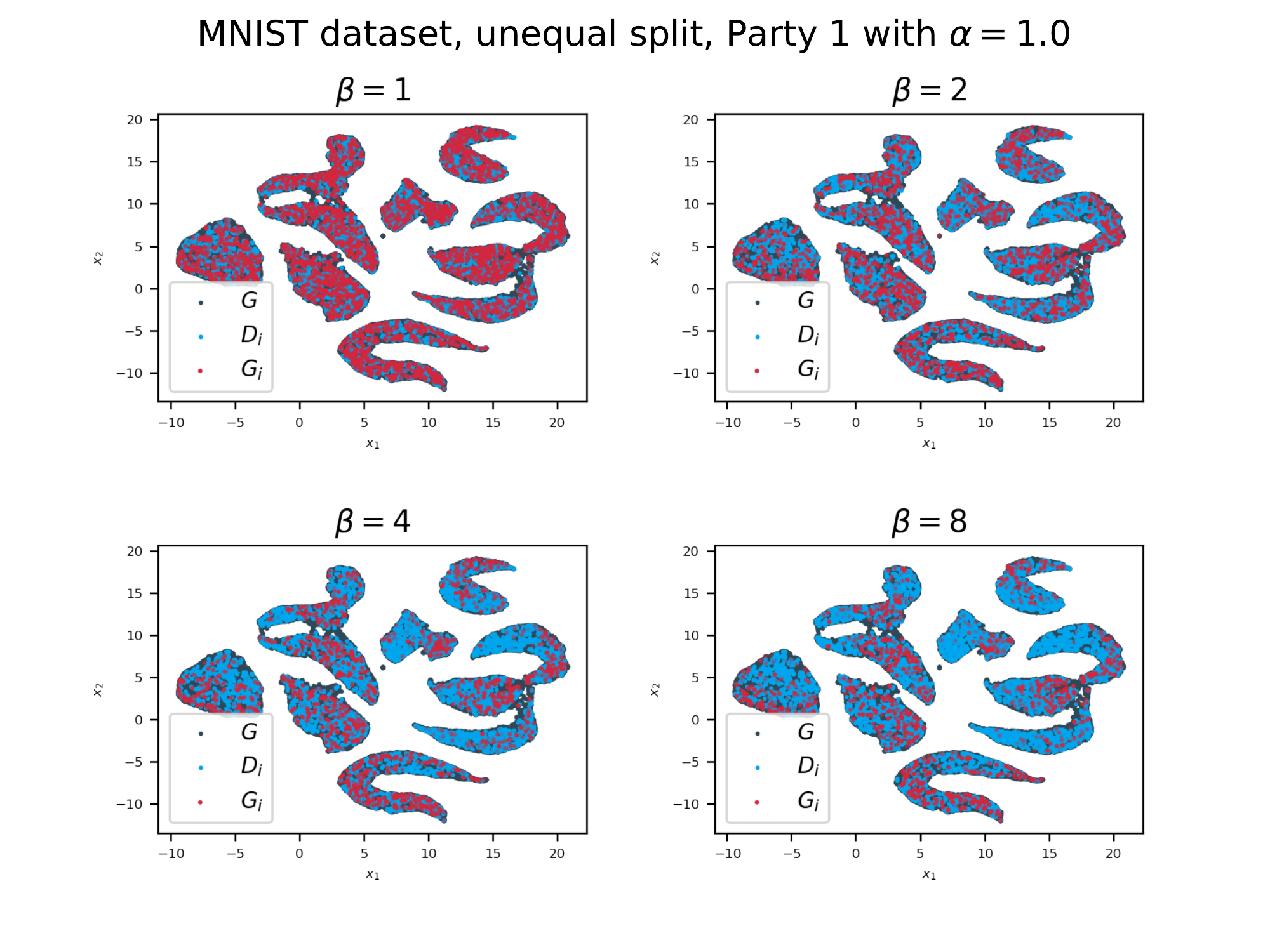}
\end{figure}
\begin{figure}[H]
    \centering
    \includegraphics[width=0.8\textwidth]{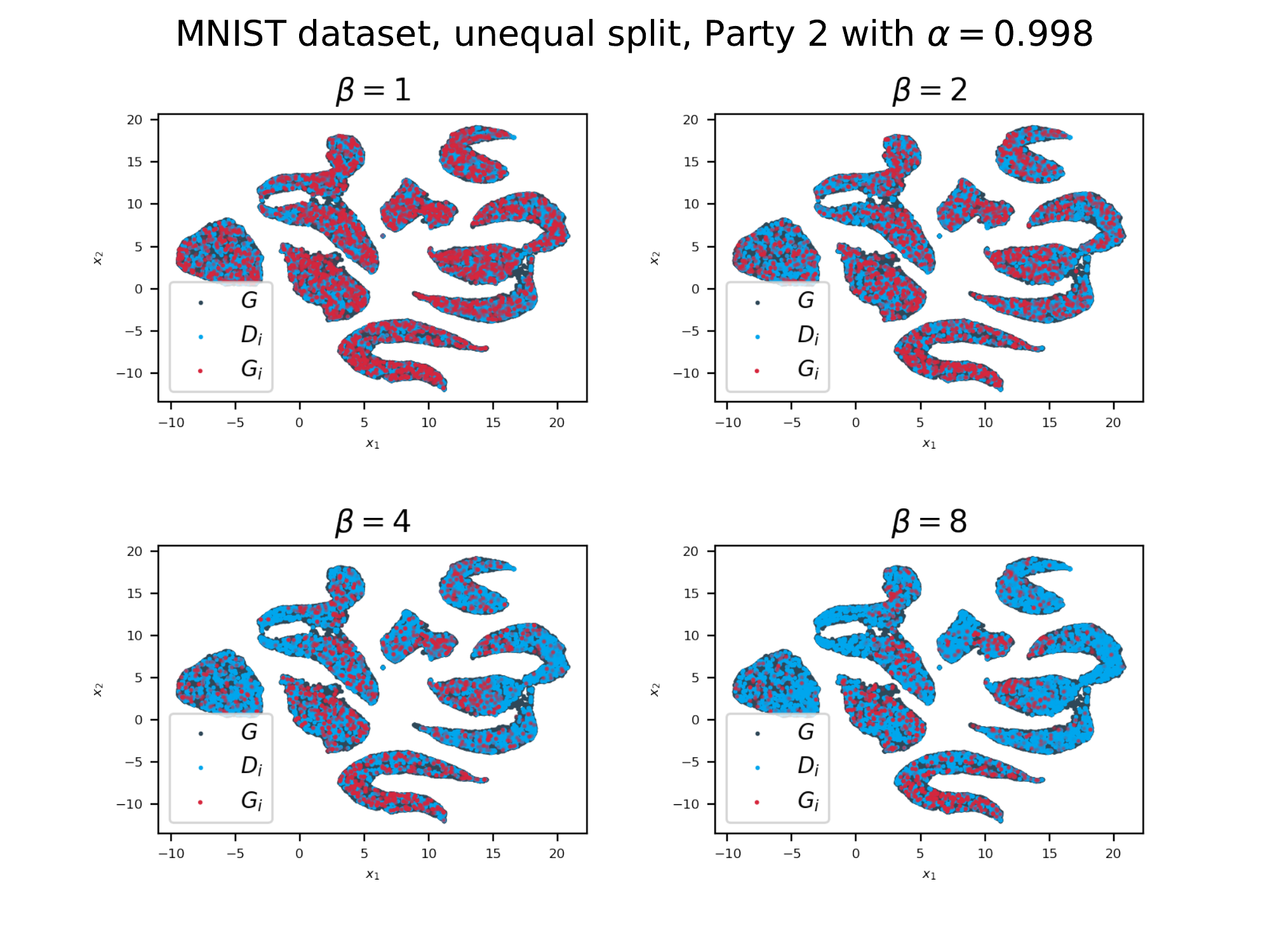}
\end{figure}
\begin{figure}[H]
    \centering
    \includegraphics[width=0.8\textwidth]{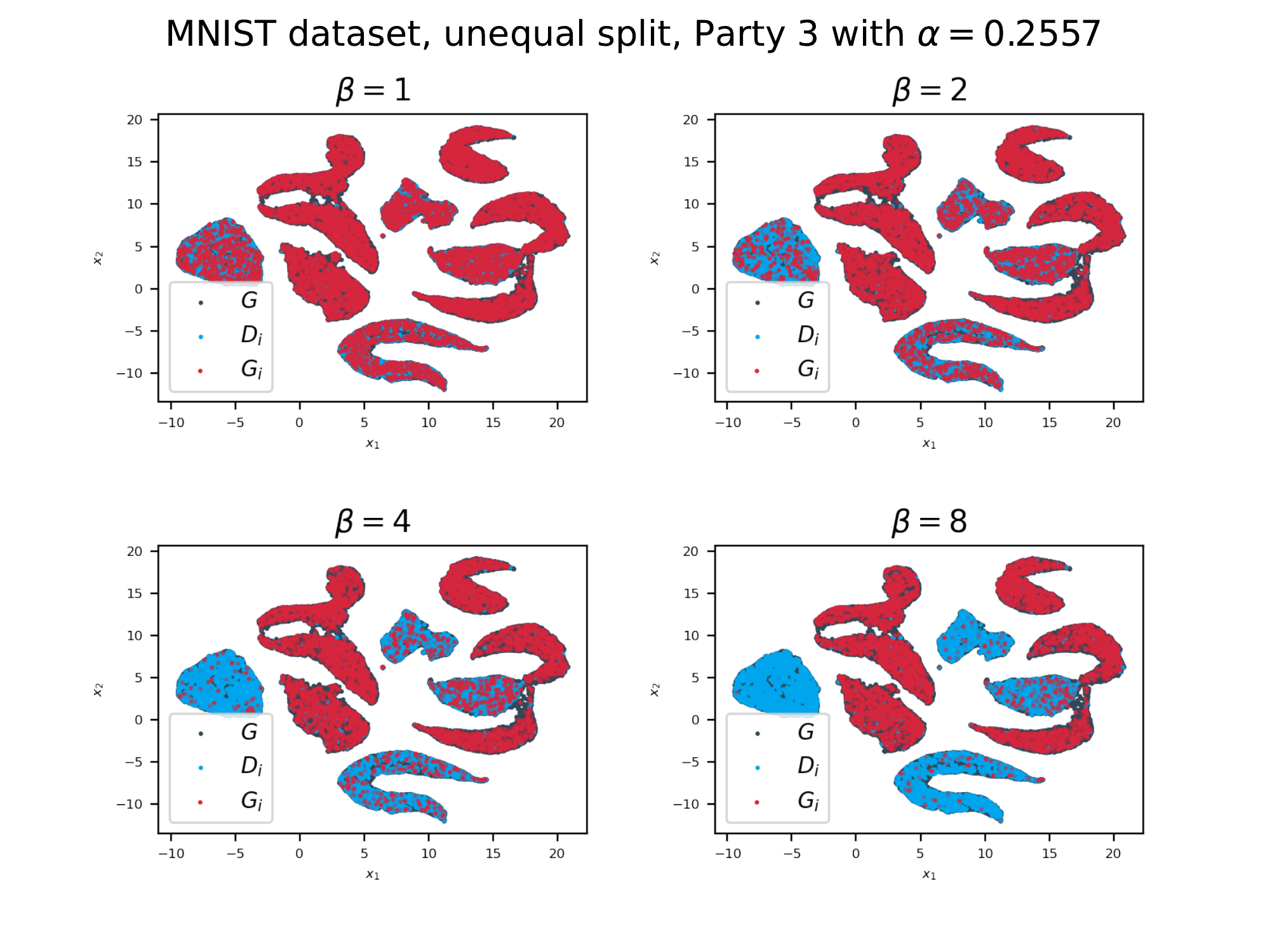}
\end{figure}
\begin{figure}[H]
    \centering
    \includegraphics[width=0.8\textwidth]{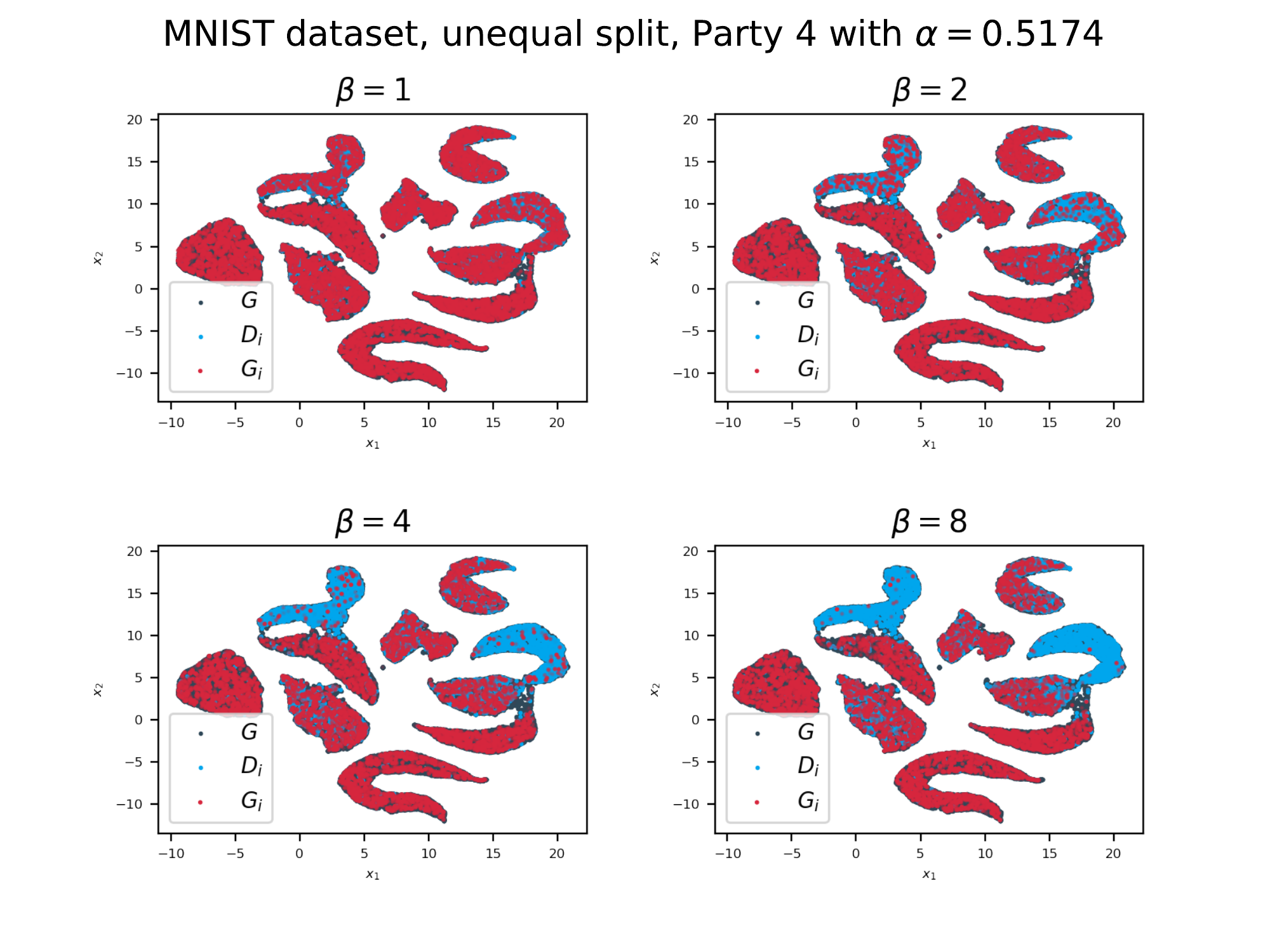}
\end{figure}
\begin{figure}[H]
    \centering
    \includegraphics[width=0.8\textwidth]{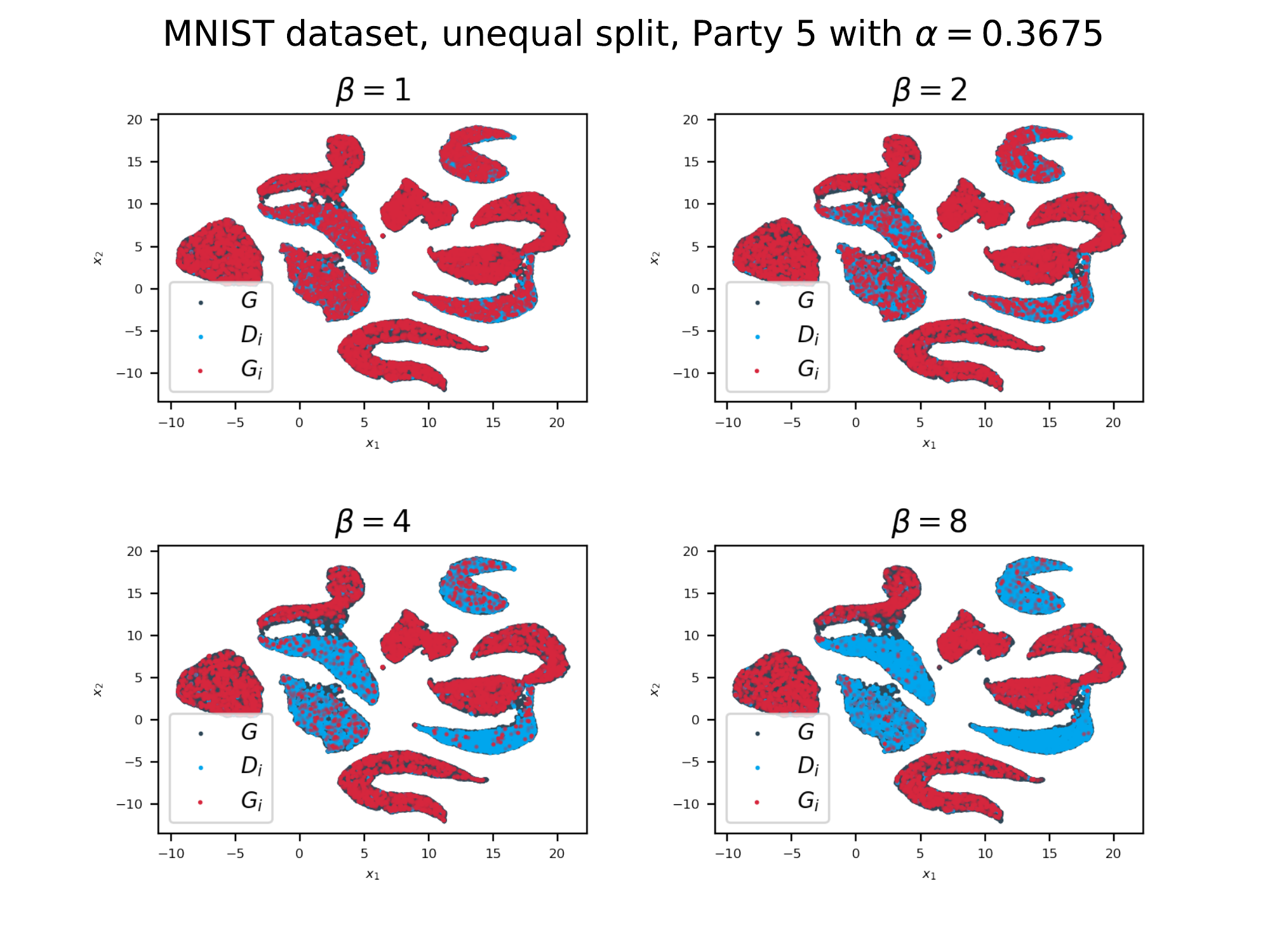}
\end{figure}
\begin{figure}[H]
    \centering
    \includegraphics[width=0.8\textwidth]{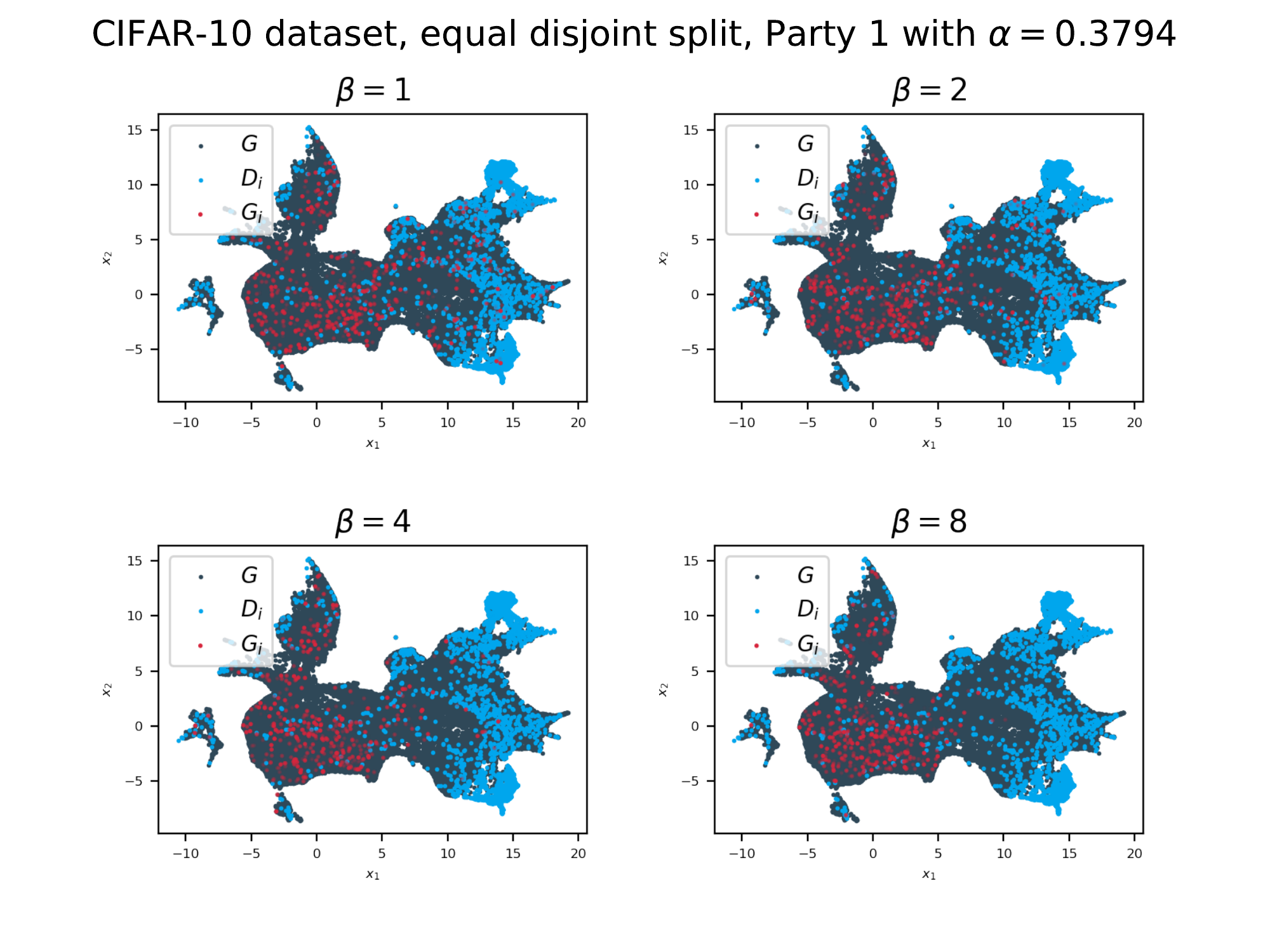}
\end{figure}
\begin{figure}[H]
    \centering
    \includegraphics[width=0.8\textwidth]{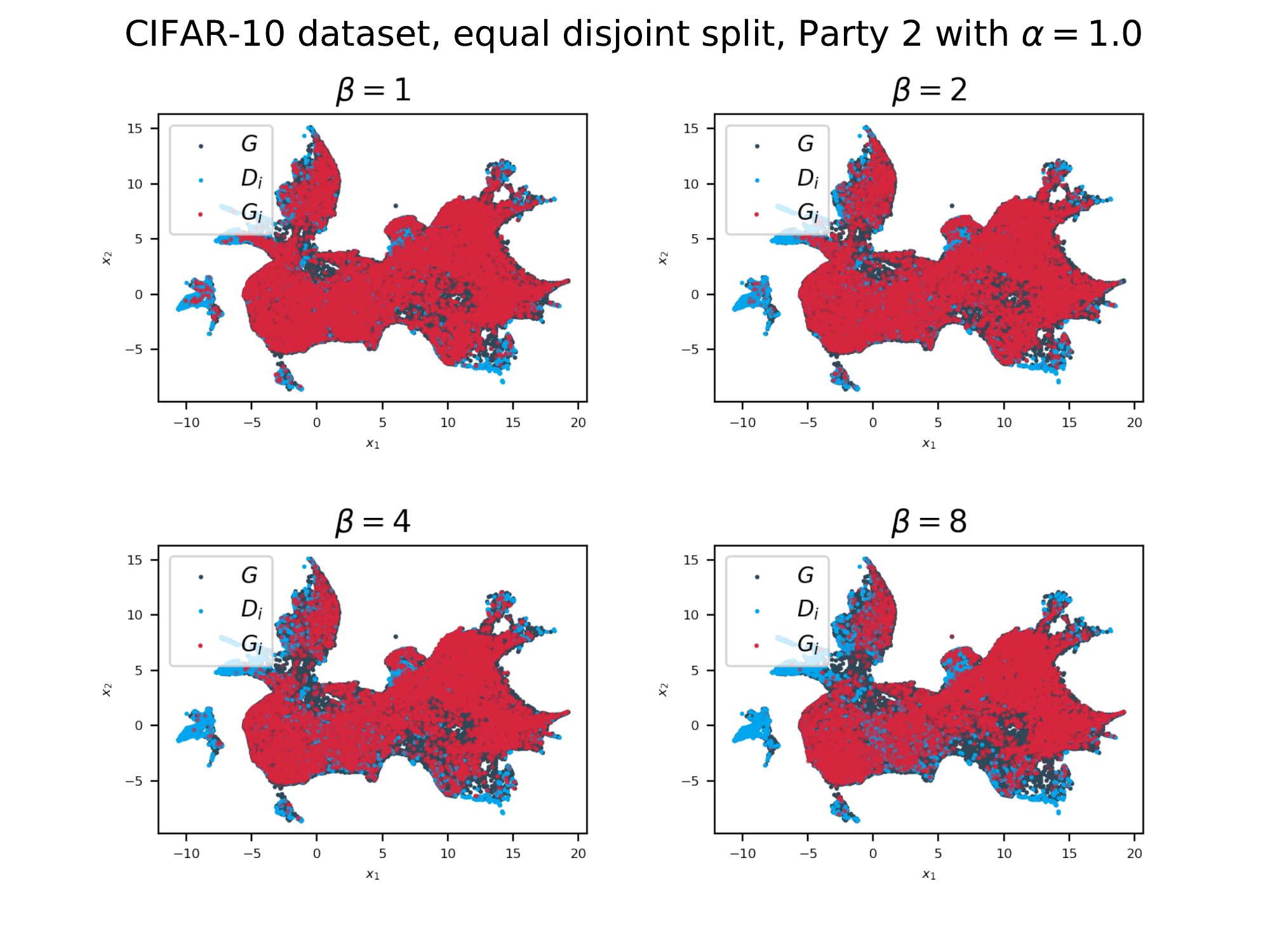}
\end{figure}
\begin{figure}[H]
    \centering
    \includegraphics[width=0.8\textwidth]{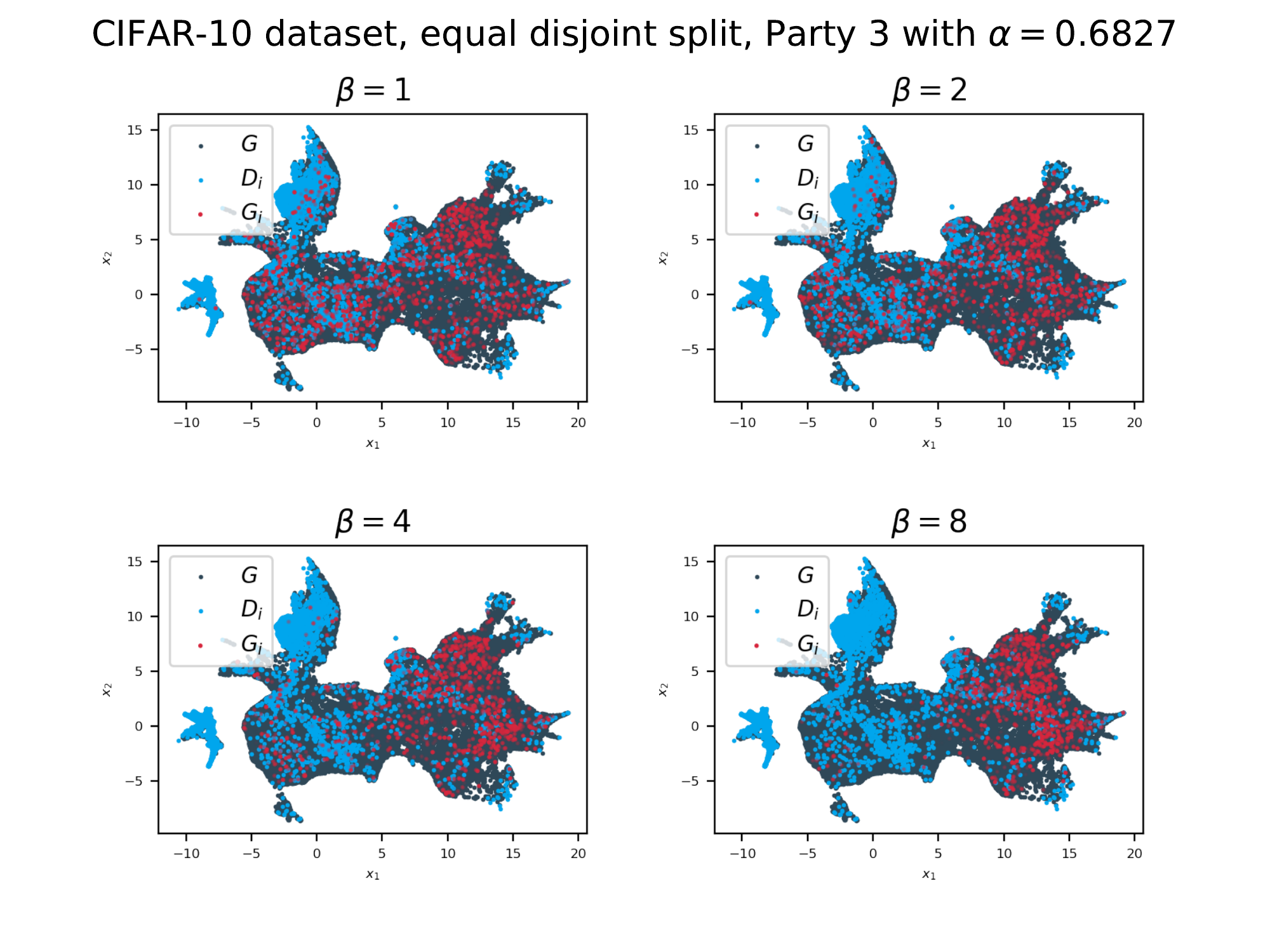}
\end{figure}
\begin{figure}[H]
    \centering
    \includegraphics[width=0.8\textwidth]{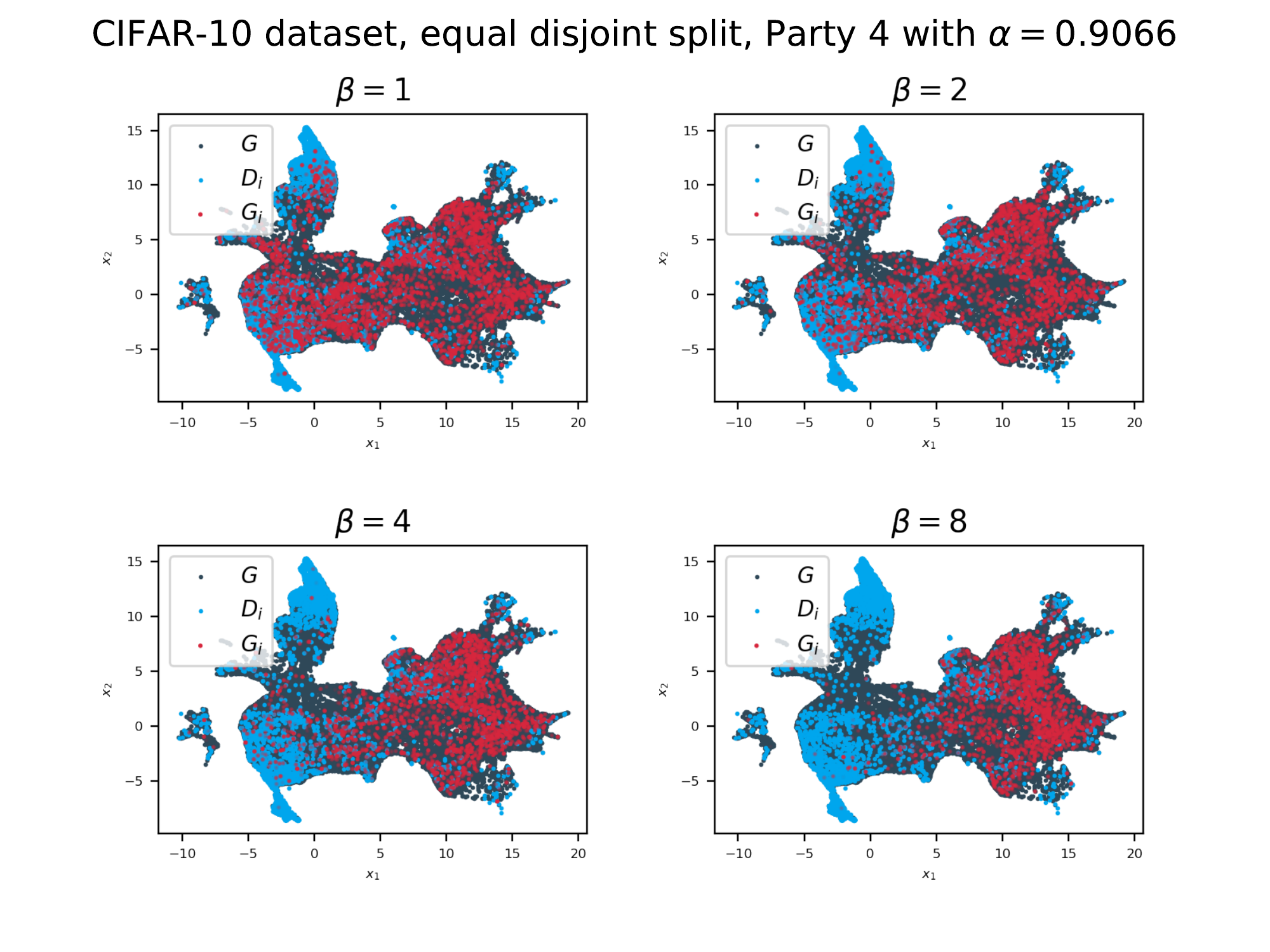}
\end{figure}
\begin{figure}[H]
    \centering
    \includegraphics[width=0.8\textwidth]{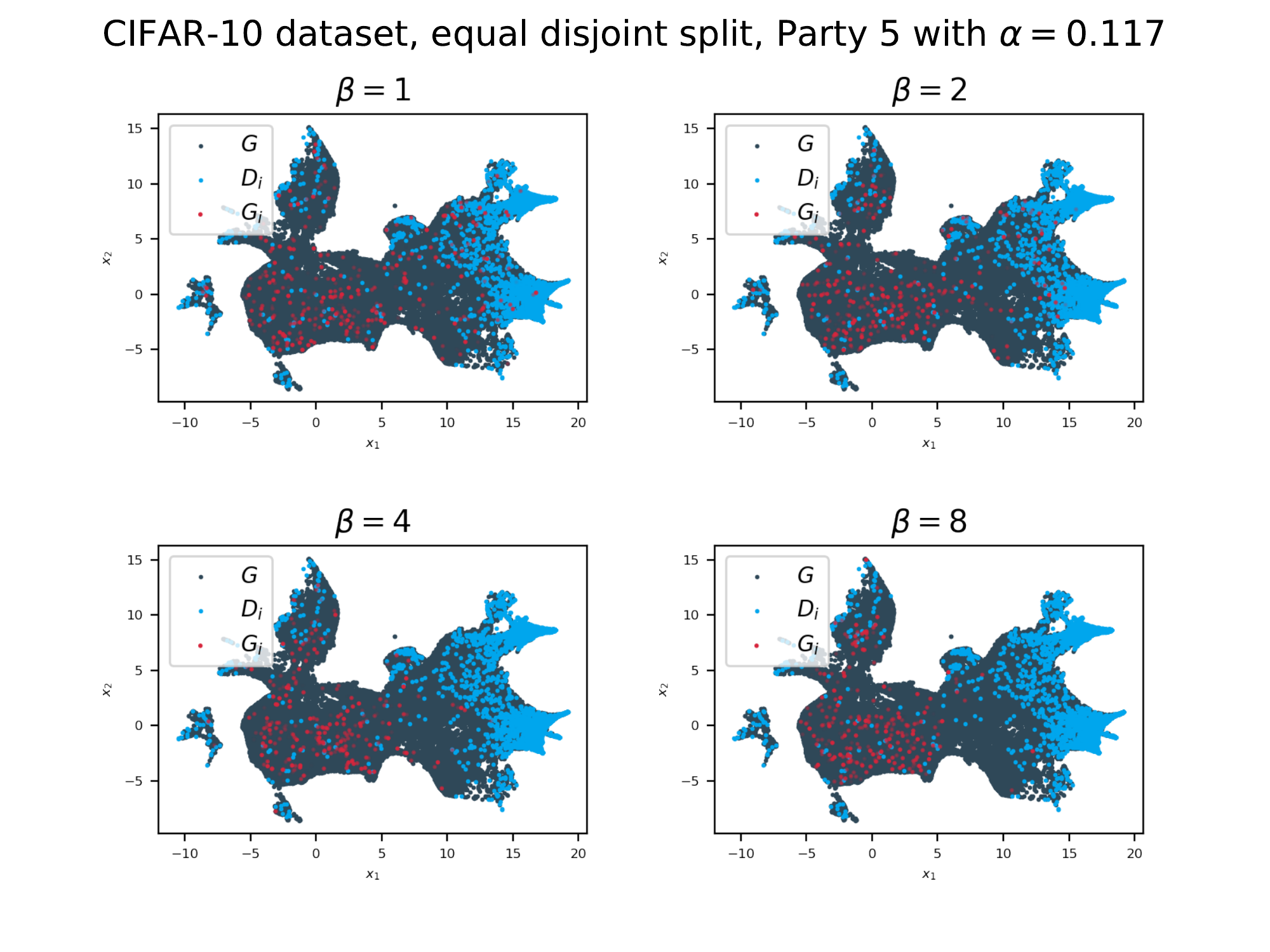}
\end{figure}
\begin{figure}[H]
    \centering
    \includegraphics[width=0.8\textwidth]{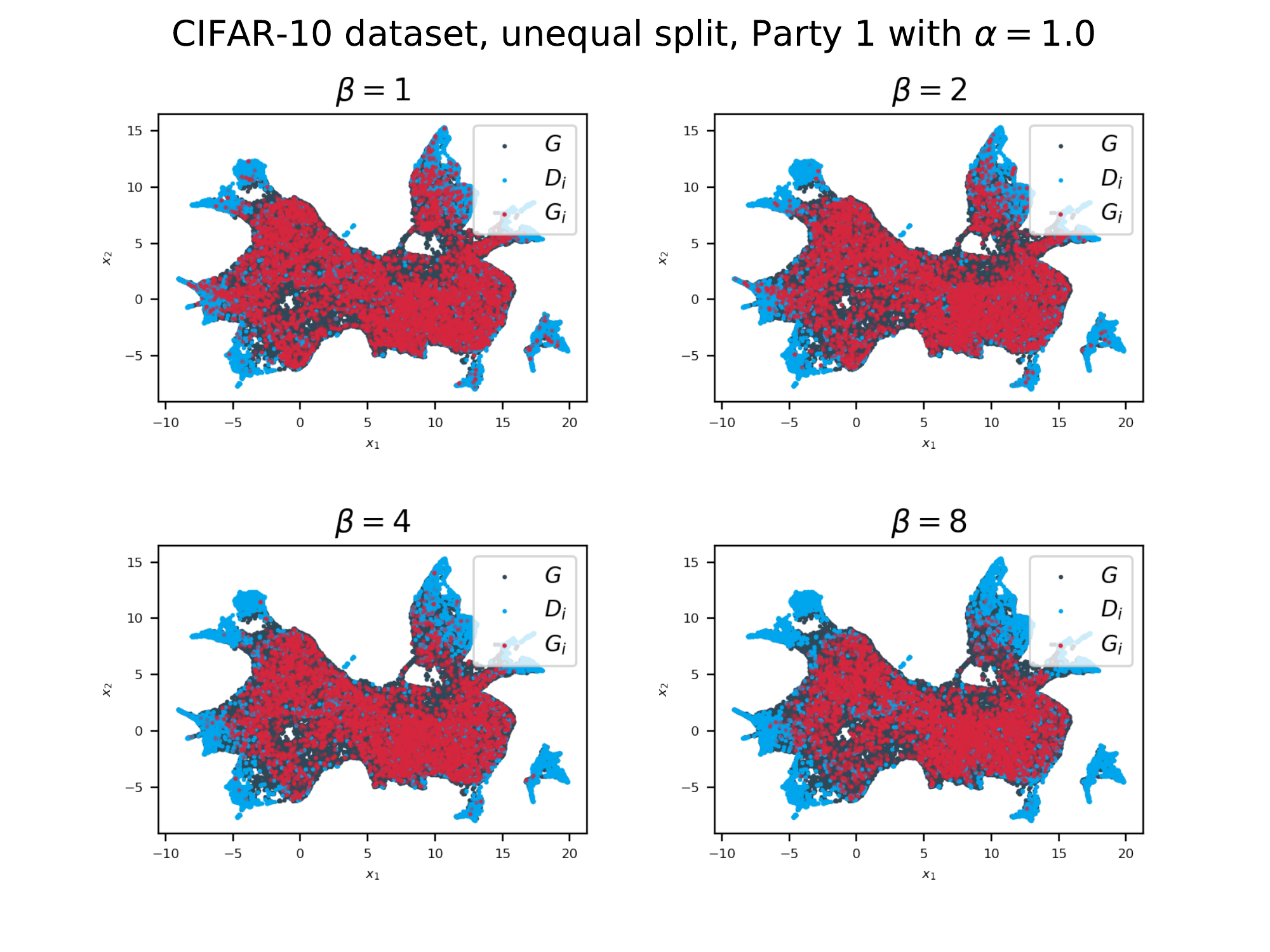}
\end{figure}
\begin{figure}[H]
    \centering
    \includegraphics[width=0.8\textwidth]{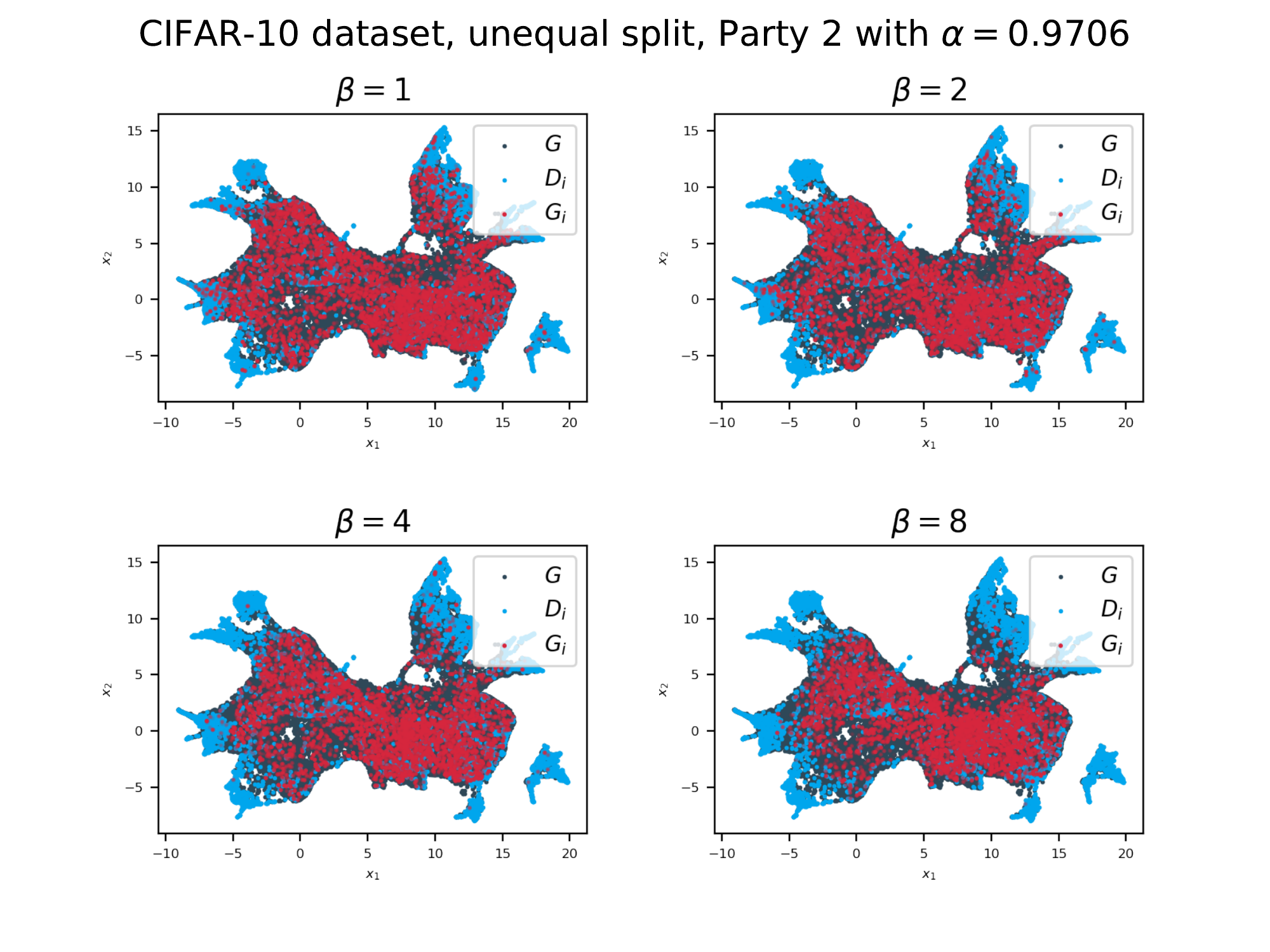}
\end{figure}
\begin{figure}[H]
    \centering
    \includegraphics[width=0.8\textwidth]{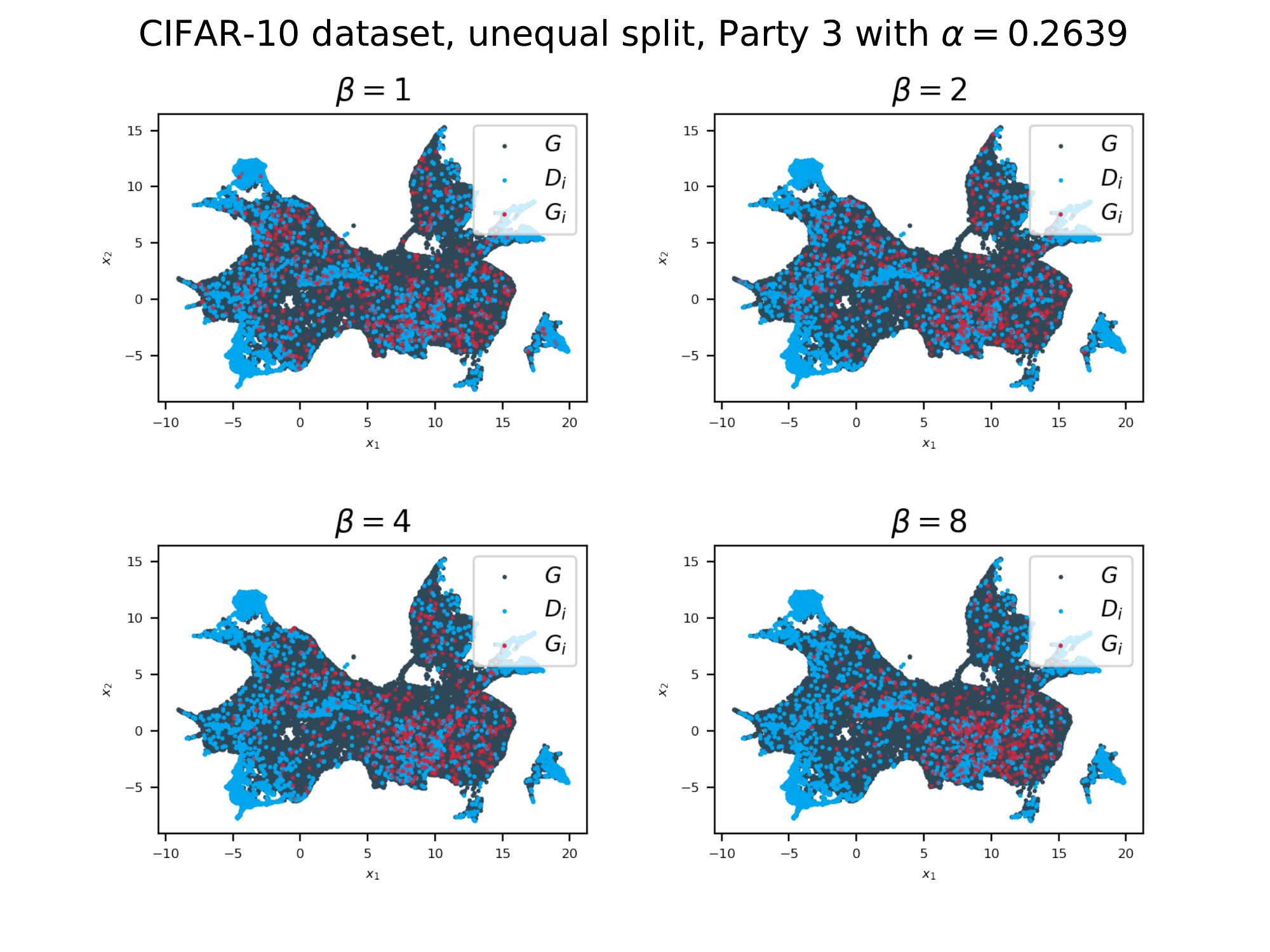}
\end{figure}
\begin{figure}[H]
    \centering
    \includegraphics[width=0.8\textwidth]{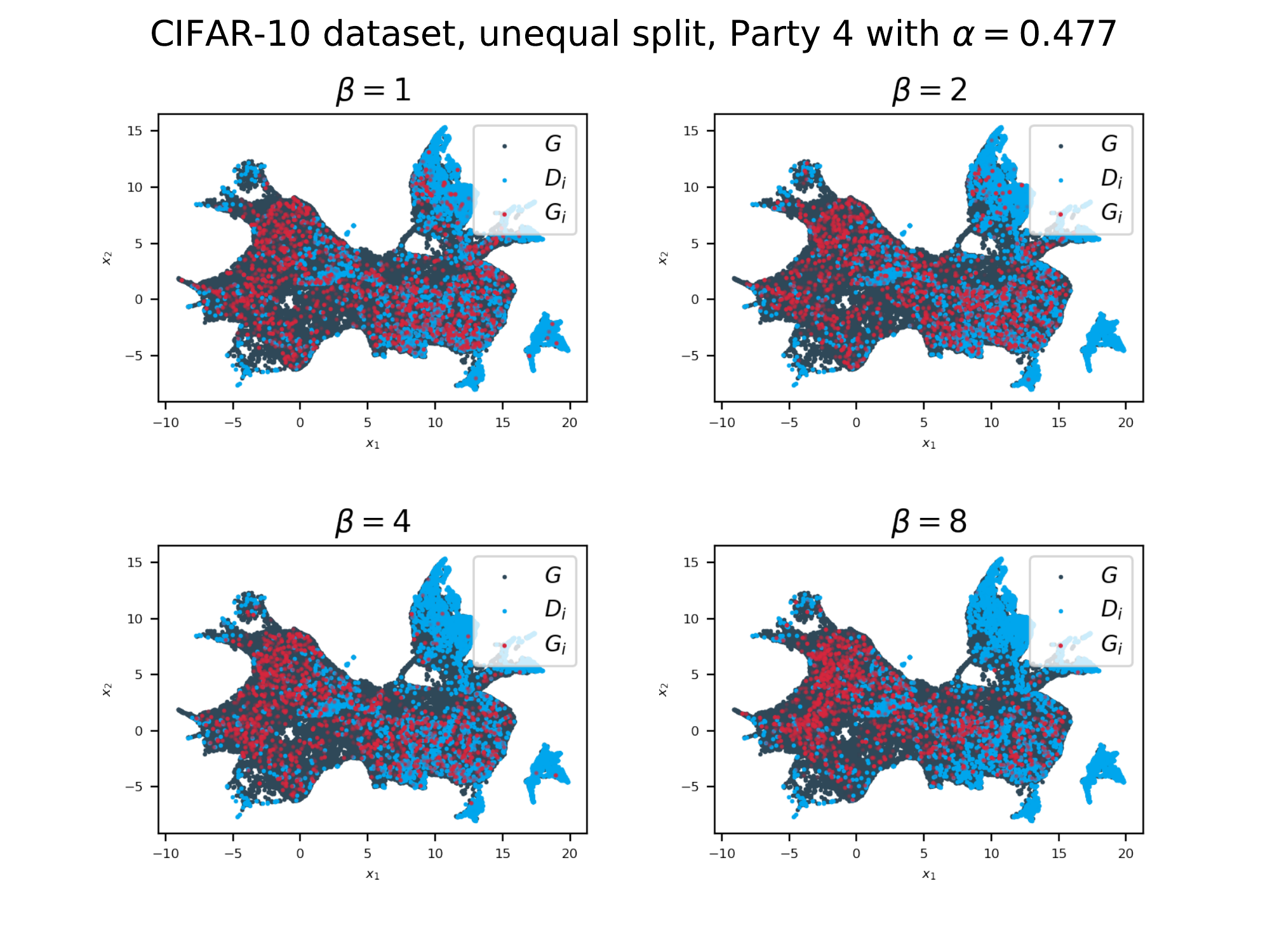}
\end{figure}
\begin{figure}[H]
    \centering
    \includegraphics[width=0.8\textwidth]{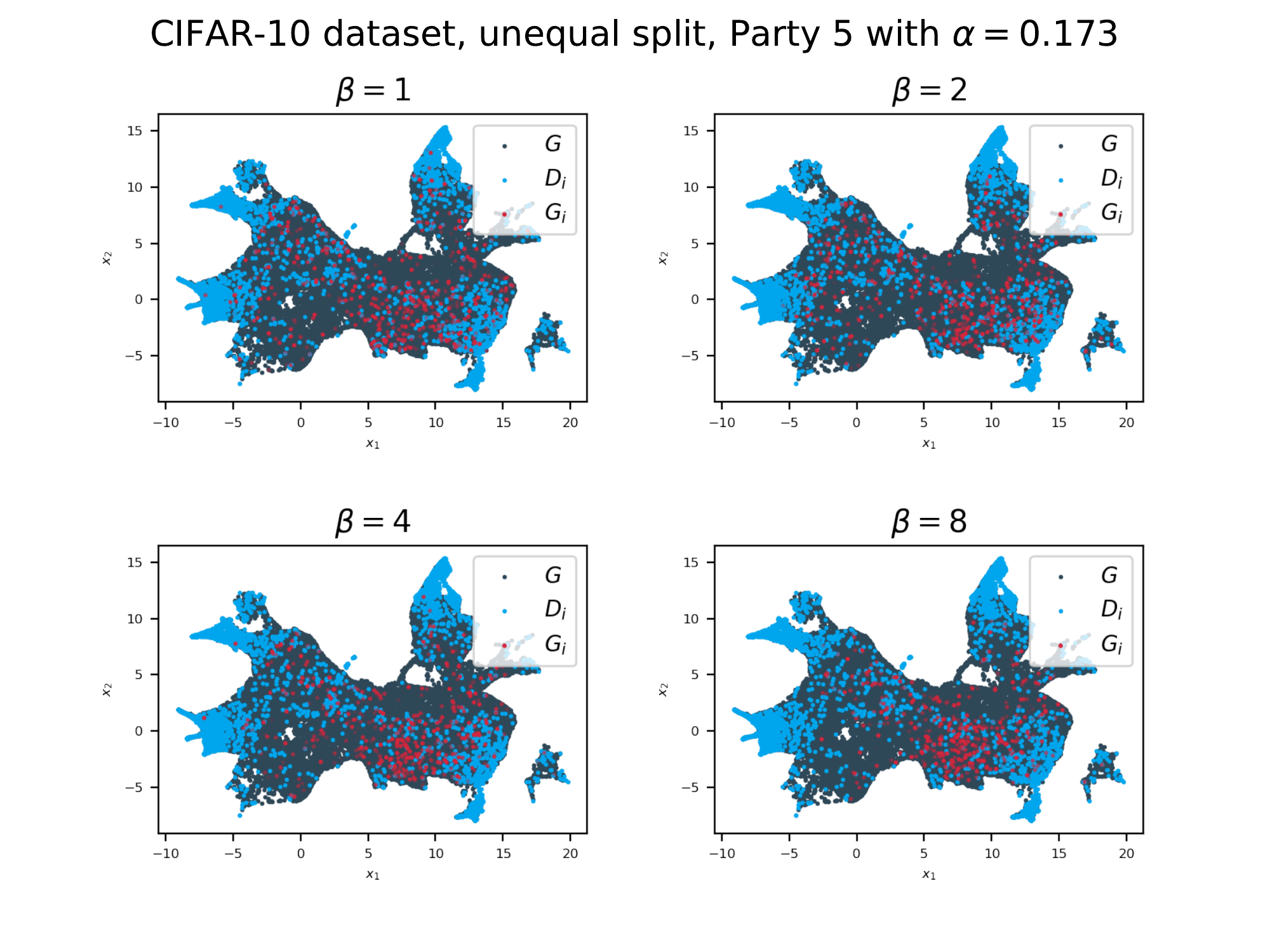}
\end{figure}

\end{document}